\documentclass[10pt,journal,compsoc]{IEEEtran}
\ifCLASSOPTIONcompsoc
  \usepackage[nocompress]{cite}
\else
  \usepackage{cite}
\fi
\ifCLASSINFOpdf
\else
\fi
\usepackage{cite}
\usepackage{amsmath,graphicx,amssymb,color,textcomp,amsfonts,subfigure,enumerate,bm}
\usepackage{extarrows}
\usepackage{makecell,multirow,diagbox,stfloats}
\newcommand{\reffig}[1]{Fig. \ref{#1}}
\newcommand{\refeq}[1]{(\ref{#1})}
\renewcommand\arraystretch{1.5}
\usepackage{amsthm}
\usepackage{algorithmicx,algorithm,balance}
\usepackage{algpseudocode}
\usepackage{colortbl}
\usepackage{array}
\usepackage{setspace}
\usepackage{booktabs}
\usepackage{cuted}
\usepackage{forloop}
\usepackage{bbm}
\newcounter{ct}

\newtheorem{proposition}{Proposition}

\begin{document}
%
\title{	Cluster Head Detection for Hierarchical UAV Swarm With Graph Self-supervised Learning }
%
%
%
%

\author{Zhiyu Mou, Jun Liu, Xiang Yun, Feifei Gao, and Qihui Wu 
\IEEEcompsocitemizethanks{\IEEEcompsocthanksitem Z. Mou, F. Gao are with the Institute for Artificial Intelligence, Tsinghua University (THUAI), State Key Lab of Intelligence Technologies and Systems, Tsinghua University, Beijing National Research Center for Information Science and Technology (BNRist), Department of Automation, School of Information Science and Technology, Tsinghua University, Beijing 100084, China. 
E-mail: mouzy20@mails.tsinghua.edu.cn, feifeigao@ieee.org.
\IEEEcompsocthanksitem J. Liu is with the Institute of Network Sciences and Cyberspace, Tsinghua University, Beijing 100084, China, and also with the Beijing National Research Center for Information Science and Technology (BNRist), Tsinghua University, Beijing 100084, China. E-mail: juneliu@tsinghua.edu.cn.

\IEEEcompsocthanksitem X. Yun, technical and innovation director/head of standard, is with Baicells Technologies Co., Ltd. E-mail: yunxiang@baicells.com.

\IEEEcompsocthanksitem Q. Wu is with the College of Electronic and Information Engineering, Nanjing University of Aeronautics and Astronautics, Nanjing 210016, China. E-mail: wuqihui2014@sina.com.

}
\thanks{Manuscript received Mar. 7, 2022.}}

%
%

\markboth{peer review papers}%
{peer review papers}
%



\IEEEtitleabstractindextext{%
\begin{abstract}
In this paper, we study the cluster head detection problem of a two-level unmanned aerial vehicle (UAV) swarm network (USNET) with multiple UAV clusters, where the inherent follow strategy (IFS) of low-level follower UAVs (FUAVs) with respect to high-level cluster head UAVs (HUAVs) is unknown. We first propose a graph attention self-supervised learning algorithm (GASSL) to detect the HUAVs of a single UAV cluster,
where the GASSL can fit the IFS at the same time. 
Then, to detect the HUAVs in the USNET with multiple UAV clusters, we develop a multi-cluster graph attention self-supervised learning algorithm (MC-GASSL) based on the GASSL. The MC-GASSL clusters the USNET with a gated recurrent unit (GRU)-based metric learning scheme and finds the HUAVs in each cluster with GASSL. Numerical results show that the GASSL can detect the HUAVs in single UAV clusters obeying various kinds of IFSs with over 98\% average accuracy. The simulation results also show that the clustering purity of the USNET with MC-GASSL exceeds that with traditional clustering algorithms by at least 10\% average. Furthermore, the MC-GASSL can efficiently detect all the HUAVs in USNETs with various IFSs and cluster numbers with low detection redundancies.
\end{abstract}

}

\maketitle

\IEEEdisplaynontitleabstractindextext

%
\IEEEpeerreviewmaketitle

\IEEEraisesectionheading{\section{Introduction}\label{sec:introduction}}

%
%
%
%
\IEEEPARstart{U}{nmanned} aerial vehicle (UAV) swarm network (USNET) composed of several hierarchical UAV clusters usually has structural advantages over flat UAV swarms in many aspects, including collective management, communication efficiency, and labor divisions \cite{uav_swarm}. As a result, USNET has become a critical technology to a broad of UAV-aided scenarios, such as data collections \cite{data_collection_1, data_collection_2}, area coverage \cite{area_coverage} and securities \cite{security}. 
Generating behaviors of UAVs is an important part in the USNET technique and has been investigated in many literatures. For example, the authors in \cite{data_collection_1, area_coverage, tra_1, tra_2, tra_3} have studied the trajectory planning algorithms of USNET exhaustively in various scenes, while \cite{data_1,data_2,data_3} and \cite{charging_1, charging_2} have proposed efficient methods on data transmissions and charging schedules for the USNET, respectively. Note that apart from the UAVs themselves, there usually exists another side of agents in practical applications of the USNET technique that respond to the behaviors of UAVs, such like ground users, clients or defenders \cite{defenders}. We name the other side of agents as \emph{observers} for convenience.
As high-level cluster head UAVs (HUAVs) often act as control and communication centers for low-level follower UAVs (FUAVs) in each UAV cluster,
one of the prerequisites for observers is to detect the HUAVs of the USNET \cite{hie}. 
In fact, detecting the HUAVs is urgently needed in many practical scenes, especially in those involving offensive and defensive confrontations \cite{confrontation}. 
However, to the best of our knowledge, few literatures have deeply investigated the cluster head detection of USNET from the perspective of observers.

The cluster head detection of USNET is a complex and even abstruse problem that faces many challenges \cite{survey_cd}. The first challenge lies in the indistinguishability of the appearance of UAVs from different levels \cite{appearance}. Specifically, the observers sometimes cannot tell the difference between HUAVs and FUAVs exclusively from exterior. They can only detect the HUAVs by mining the patterns of UAVs' flying behaviors, such as positions, speeds, accelerations, etc.
The second challenge lies in the node heterogeneities caused by the hierarchy \cite{hetero}. As HUAVs take the lead of FUAVs, they usually have different movement patterns from FUAVs. Hence, HUAVs and FUAVs should be viewed as two distinct types of UAV nodes. In this way, the USNET becomes a heterogeneous network that is more difficult to tackle compared with the homogeneous networks \cite{hete}. 
The third challenge is that the inherent behavior relationship between FUAVs and HUAVs is unknown to observers. Note that there usually exists an inherent follow strategy (IFS) that determines the flying behaviors of FUAVs with respect to HUAVs. The observers, however, cannot obtain the IFS in advance, especially in confrontation scenarios \cite{confrontation}. Hence, the specific features of IFS cannot be utilized in designing the cluster head detection methods, which makes the problem more difficult.
In addition, the UAVs are constantly moving in three-dimensional (3D) spaces, making the USNET a dynamic graph instead of a static one. Therefore, the features of each UAV are composed of sequential position and speed vectors with various lengths, which increases the difficulties in finding the HUAVs in the USNET. 
Moreover, the number of UAV clusters in the USNET is usually not available to observers, which formulates another challenge in the cluster head detection problem.

In this paper, we study the cluster head detection problem of a two-level USNET with multiple clusters, where the IFS of the FUAVs with respect to the HUAVs is unknown. Firstly, we propose a graph attention self-supervised learning algorithm (GASSL) to detect the HUAVs of a single cluster. 
The GASSL can detect the HUAVs through calculating the attention values of each UAV received from other UAVs, while approximately fit the IFS at the same time. We prove that the GASSL satisfies the necessary condition of finding the HUAVs.
Numerical results show that the GASSL can find the HUAVs under various kinds of IFSs with over 98\% average accuracy. 
Then, for the USNET with multiple number of UAV clusters, we develop a multi-cluster graph attention self-supervised learning algorithm (MC-GASSL) based on the GASSL, which clusters the USNET with a gated recurrent unit (GRU)-based metric learning scheme and find the HUAVs in each UAV cluster with GASSL. 
The simulation results show that the clustering purity of the USNET with MC-GASSL exceeds that with traditional clustering algorithms by at least 10\% average. Furthermore, the MC-GASSL can efficiently find all the HUAVs in USNETs with various IFSs and cluster numbers with low detection redundancies.

The rest parts of this paper are organized as follows. Section \ref{section:system_model} presents the system models of the cluster head detection problem of the USNET. Section \ref{section:GASSL} describes the proposed GASSL algorithm for USNETs with single UAV clusters, while Section \ref{section:MC-GASSL} focuses on the MC-GASSL algorithm for the USNET composed of multiple UAV clusters. Simulation results and analysis are provided in Section \ref{section:simulations}, and conclusions are made in Section \ref{section:conclusions}. The abbreviations are summarized in Table \ref{table:abbreviations}. 

\begin{table}[t]
	\centering
	\caption{{The Summarization of Abbreviations.}}
	\label{table:abbreviations}
	\begin{tabular}{c|c}
		\specialrule{0em}{1pt}{1pt}
		\hline
		\rowcolor[gray]{0.9}
		\small\textbf{{Abbreviations}}&\small \textbf{{Full Name}}\\
		\hline
		\small\makecell[c]{{UAV}}&\small\makecell[c]{{unmanned aerial vehicle}}\\
		\hline
		\small\makecell[c]{{USNET}}&\small\makecell[c]{{unmanned aerial vehicle swarm network} }\\
		\hline
		\small\makecell[c]{{HUAV}}&\small\makecell[c]{{high-level cluster head UAV} }\\
		\hline
		\small\makecell[c]{{FUAV}}&\small\makecell[c]{{low-level follower UAV}}\\
		\hline
		\small\makecell[c]{{GASSL}}&\small\makecell[c]{{graph attention self-supervised learning}}\\
		\hline
		\small\makecell[c]{{MC-GASSL}}&\small\makecell[c]{{multi-cluster graph attention} \\ { self-supervised learning}}\\
		\hline
		\small\makecell[c]{{AGAT}}&\small\makecell[c]{{adaptive graph attention network}}\\
		\hline
		\small\makecell[c]{{IFSN}}&\small\makecell[c]{{inherent follow strategy network}}\\
		\hline
		\small\makecell[c]{{IFS}}&\small\makecell[c]{{inherent follow strategy} }\\
		\hline
		\small\makecell[c]{{CL}}&\small\makecell[c]{{communication link}}\\
		\hline
	\end{tabular}
\end{table}

\emph{Notations:}  $x$, $\mathbf{x}$, $\mathbf{X}$ represent a scalar $x$, a vector $\mathbf x$ and a matrix $\mathbf X$, respectively; $\sum$, $\min$, $\max$ and $\nabla$ denote the sum, minimum, maximum and vector differential operator, respectively; $(x_{ij})$ represents a matrix with element $x_{ij}$ in the $i$-th row and the $j$-th column, and $(\mathbf{X})_{ij}$ represents the element of row $i$ and column $j$ in matrix $\mathbf{X}$; $\left\|\cdot\right\|_2$ denote the 2-norm of matrices; $\cup$, $\cap$ and $\backslash$ represent the union operator, the intersection operator and the difference operator between sets; $\subseteq$ represents the set on the left is the subset of the set on the right;
$\vDash$ is the concatenation operator that concatenates the vectors on both sides; 
$|\mathcal{S}|$ represents the number of elements in set $\mathcal{S}$; 
$\mathbb{R}^N$ and $\mathbb{R}^{N\times M}$ represent the $N$-dimensional vector space and $N$-by-$M$ real matrix space, respectively; $\mathbb{N}$ and $\mathbb{N}_{+}$ represents the set of non-negative and positive integers, respectively; $\mathbf{1}_n$ represents an $n$-dimensional vector where the components are all $1$'s; $\mathbbm{1}\{\cdot\}$ represents the indicative function with range $\{0,1\}$, $\leftarrow$ denotes the assignment from right to left, while $\rightarrow$ represents the approximation of the right term by the left term; $\triangleq$ defines the symbol on the left by equation on the right. $\in$ and $\notin$ represents the element on the left belongs to and does not belong to the set on the right, respectively; in addition, $[\cdot]\triangleq\max\{\cdot,0\}$.

\section{System Model}
\label{section:system_model}
We consider a two-level USNET with $M\in\mathbb{N}_+$ clusters initially, where the $j$-th cluster is composed of one HUAV and $m_j\in\mathbb{N}_+$ FUAVs with exactly the same appearance, 
$j\in\mathcal{M}\triangleq\{1,2,...,M\}$, as shown in \reffig{fig:scene}. The total number of UAVs in the initial USNET can be calculated as $N\triangleq\sum_{j=1}^M(1+m_j)=M+\sum_{j=1}^Mm_j$. Each UAV is endowed with a fixed index $i\in\mathcal{N}\triangleq\{1,2,...,N\}$, and the index set of all HUAVs is denoted as $\mathcal{H}\subseteq\mathcal{N}$. Hereafter, we use UAV, FUAV and HUAV $i$ to represent a certain UAV, FUAV and HUAV with index $i$, respectively.
Each FUAV has a \emph{communication link} (CL) to the HUAV in the same cluster, but has no CLs to other UAVs.
Establish an $X$-$Y$-$Z$ Cartesian coordinate for the USNET, and let the position of UAV $i$
at time step $t$ be $\mathbf{p}_{i,t}=[x_{i,t},y_{i,t},z_{i,t}]^T\in\mathbb{R}^3,\; t\in\mathbb{N}$,  where $x_{i,t}$, $y_{i,t}$ and $z_{i,t}$ represent the $X$, $Y$ and $Z$ axis components, respectively. The speed of UAV $i$ at time step $t$ can then be defined as $\mathbf{v}_{i,t}\triangleq\mathbf{p}_{i,t+1}-\mathbf{p}_{i,t}$. 
\begin{figure*}[t]
	\centering
	\includegraphics[width=160mm]{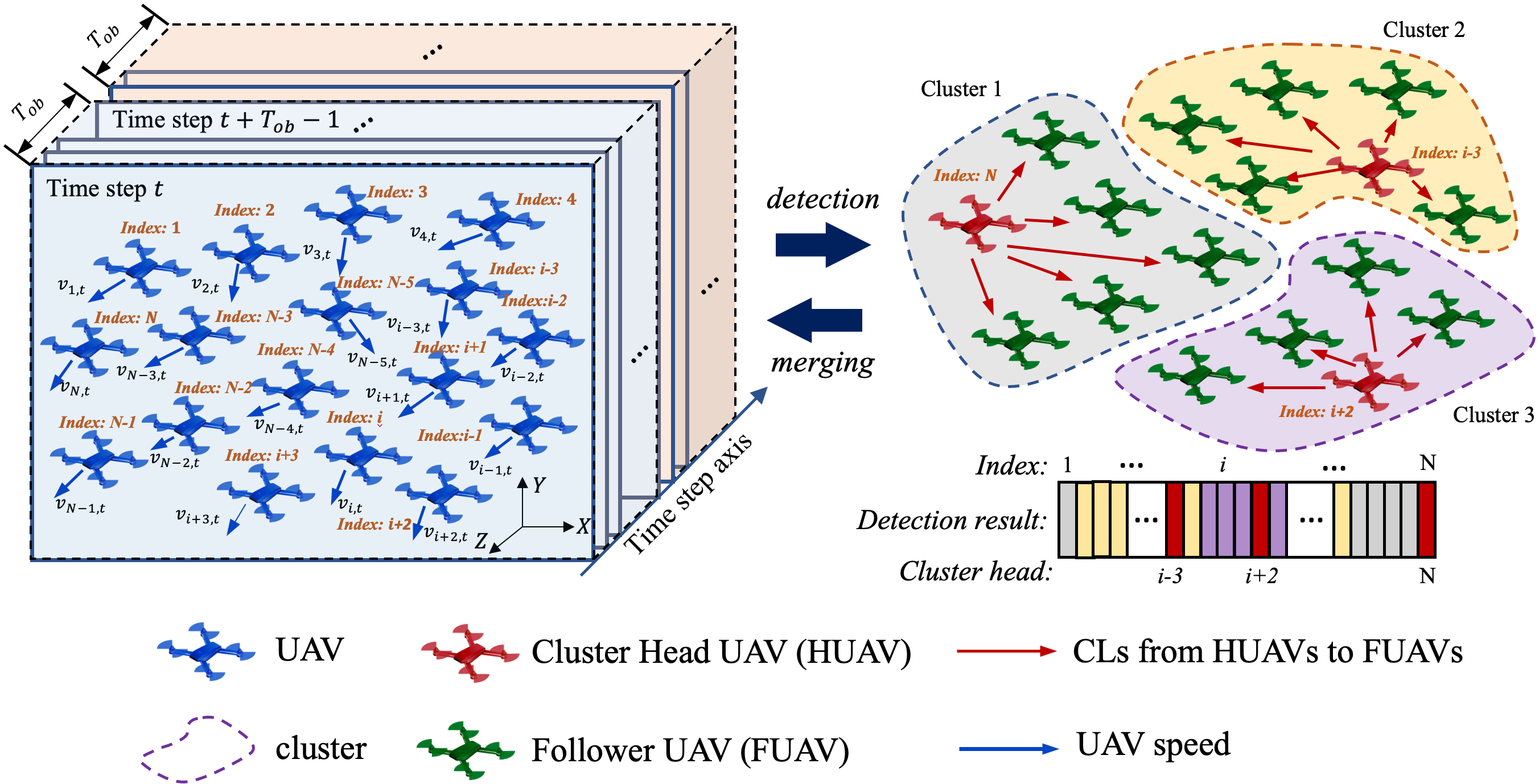}
	\caption{Cluster head detection through sequential observations of UAV behaviors.}
	\label{fig:scene}
\end{figure*}

The HUAVs determine their speeds independently, while the FUAVs follow the HUAVs within the same clusters. To be specific, each FUAV in the USNET obeys an identical IFS. Without loss of generality, we suppose the speeds of FUAVs only depend on their historical positions in the past $T_0$ time steps, as well as their HUAVs' historical speeds and positions in the past $T_0$ time steps, where $T_0\in\mathbb{N}_+$ is a constant. 
The IFS can then be represented as a function $f(\cdot)$. Specifically, let UAV $i$ and UAV $i_L$ be an FUAV and an HUAV within the same cluster, respectively. 
The speed of FUAV $i$ at time step $t+1$ can be expressed as 
\begin{align}
	\label{equ:vf}
	\mathbf{v}_{i,t+1}=f(\mathbf{P}_{i,t+1},\mathbf{V}_{i_L,t}, \mathbf{p}_{i_L,t-T_0+1}),
\end{align}
where $\mathbf{P}_{i,t+1}\triangleq [\mathbf{p}_{i,t+1},...,\mathbf{p}_{i,t-T_0+1}]^T\in\mathbb{R}^{(T_0+1)\times 3}$ and $\mathbf{V}_{i_L,t}\triangleq [\mathbf{v}_{i_L,t},...,\mathbf{v}_{i_L,t-T_0+1}]^T\in\mathbb{R}^{T_0\times 3}$. 
Note that \eqref{equ:vf} is a general form of the IFS that can encompass a wide range of functions. See Appendix \ref{appendix:A-1} for further illustrations on \eqref{equ:vf}.

Note that the specific form of the IFS $f(\cdot)$ is unknown in advance. Hence, we can only detect the HUAVs by mining the relationship between the speeds of UAVs. Specifically, we observe the positions and speeds of all UAVs in $T_{ob}\ge T_0$ consecutive time steps and destroy
the \emph{candidate HUAVs} produced by detection algorithms\footnote{The detection algorithms are proposed in Section \ref{section:GASSL} and \ref{section:MC-GASSL}.}.
Note that the candidate HUAVs may contain \emph{redundant detected HUAVs} that are actually FUAVs instead of HUAVs due to the limitations of the detection algorithms.
If there remain undetected HUAVs in the USNET, then the FUAVs in the clusters with no HUAVs will automatically merge to the nearest HUAVs. Afterwards, we will start another round of detection and the FUAVs will merge again until all the HUAVs are found out and destructed. 
We denote the total number of detection and merging rounds as $R\in\mathbb{N}_+$. Denote the index set of the candidate HUAVs in the $r$-th detection round as $\widehat{\mathcal{H}}_r$, and the index set of the remaining HUAVs after the $r$-th detection round as $\mathcal{H}_r$, where $r\in\{1,2,...,R\}$. Let $\mathcal{H}_0\triangleq\mathcal{H}$. Then we have
\begin{align}
	\label{equ:h_relation}
	\mathcal{H}_r=\mathcal{H}_{r-1}\backslash \widehat{\mathcal{H}}_r, \;\forall r,
\end{align}
and
\begin{align}
	\label{equ:h_subset}
	\mathcal{H}_0\supseteq\mathcal{H}_{1}\supseteq\cdots\supseteq\mathcal{H}_{R-1}\supset\mathcal{H}_R=\emptyset.
\end{align}
The set of successfully detected HUAVs in the $r$-th round can be expressed as $\mathcal{H}_{r-1}\backslash\mathcal{H}_r$, and the set of redundant detected HUAVs can be expressed as $\widehat{\mathcal{H}}_r\backslash\mathcal{H}_{r-1}$. The total number of successfully detected HUAVs in all rounds satisfies $\sum_{r=1}^R|\mathcal{H}_{r-1}\backslash\mathcal{H}_r|=|\mathcal{H}|$, while the total number of redundant detected HUAVs can be calculated as $\sum_{r=1}^R|\widehat{\mathcal{H}}_r\backslash\mathcal{H}_{r-1}|$. In addition, we can define the \emph{detection redundancy } as the ratio between $\sum_{r=1}^R|\widehat{\mathcal{H}}_r\backslash\mathcal{H}_{r-1}|$ and the total number of UAVs, i.e.,  $\frac{\sum_{r=1}^R|\widehat{\mathcal{H}}_r\backslash\mathcal{H}_{r-1}|}{N}$.
Denote the time step when the $r$-th round finishes as $t_r$, and let $t_0\triangleq0$. We define a function $\mathcal{R}:\mathbb{N}\rightarrow\{0,1,2,...,R\}$ as
\begin{align}
	\mathcal{R}(t)=\left\{
	\begin{aligned}
		\arg\max_{r} t\ge t_r, \quad \mbox{if} \;t<t_R;&\\
		R,\quad \quad\quad\quad \ \mbox{otherwise}.&
	\end{aligned}
	\right.
\end{align}
At time step $t$, the index set of remaining HUAVs can be expressed as $\mathcal{H}_{\mathcal{R}(t)}$, and the number of clusters is $|\mathcal{H}_{\mathcal{R}(t)}|$.
\subsection{USNET Graph}
The USNET can be viewed as a set of an undirected graphs $\mathcal{G}_t=\{\mathcal{G}_{1,t},\mathcal{G}_{2,t},...,\mathcal{G}_{|\mathcal{H}_{\mathcal{R}(t)}|,t}\}$ at each time step $t$, where $\mathcal{G}_{j,t}$ represents the $j$-th cluster. In addition, the $j$-th cluster is defined as $\mathcal{G}_{j,t}\triangleq\{\mathcal{N}_{j,t},\mathcal{E}_{j,t},\mathbf{X}_{j,t}\}$. The first term $\mathcal{N}_{j,t}$ is the index set of all UAVs in the $j$-th cluster, acting as the \emph{node set} of $\mathcal{G}_{j,t}$. The second term $\mathcal{E}_{j,t}\triangleq\{e_{ii'}|i,i'\in\mathcal{N}_{j,t},i\notin\mathcal{H}_{\mathcal{R}(t)}, i'\in\mathcal{H}_{\mathcal{R}(t)}\}$ is the \emph{edge set}, where $e_{ii'}$ represents the CL from HUAV $i'$ to FUAV $i$. The third term $\mathbf{X}_{j,t}\in\mathbb{R}^{3m_j\times (T_0+1)}$ is the \emph{time series topology matrix} that concatenates each UAV's position at time step $t-T_0+1$ as well as its speeds in the past $T_0$ time steps, i.e.,
\begin{align}
	\renewcommand{\arraystretch}{1}
	\mathbf{X}_{j,t}&\triangleq\begin{bmatrix} \mathbf{p}_{i_1,t-T_0+1} & \mathbf{v}_{i_1,t-T_0+1}  & \cdots & \mathbf{v}_{i_1,t}  \\ \mathbf{p}_{i_2,t-T_0+1} & \mathbf{v}_{i_2,t-T_0+1} & \cdots & \mathbf{v}_{i_2,t} \\
		\vdots & \vdots & \ddots & \vdots \\
		\mathbf{p}_{i_{m_j},t-T_0+1} & \mathbf{v}_{i_{m_j},t-T_0+1} & \cdots & \mathbf{v}_{i_{m_j},t}  \end{bmatrix}\\\notag
	&=\begin{bmatrix} \mathbf{p}_{i_1,t-T_0+1} & \mathbf{V}_{i_1,t}^T  \\ \mathbf{p}_{i_2,t-T_0+1} & \mathbf{V}_{i_2,t}^T  \\
		\vdots & \vdots \\
		\mathbf{p}_{i_{m_j},t-T_0+1} & \mathbf{V}_{i_{m_j},t}^T   \end{bmatrix},
\end{align} 
where $i_1,i_2,...,i_{m_j}\in\mathcal{N}_{j,t}$. Note that $\mathbf{X}_{j,t}$ also contains the position information of the $j$-th cluster in the period of time step $t-T_0+1$ to $t+1$ since the position of each UAV $i\in\mathcal{N}_{j,t}$ at time step $\tau\in\{t-T_0+2,...,t, t+1\}$ can be calculated as $\mathbf{p}_{i,\tau}=\mathbf{p}_{i,t-T_0+1}+\sum_{\tau'=t-T_0+1}^{\tau-1}\mathbf{v}_{i,\tau'}$.
\subsection{Problem Formulations}
We first study a basic cluster head detection problem, where the USNET is known to have $M=1$ cluster but with unknown HUAV index $i_L$. In this case, the USNET graph $\mathcal{G}_t$ only has one element $\mathcal{G}_{1,t}=\{\mathcal{N}_{1,t},\mathcal{E}_{1,t}, \mathbf{X}_{1,t}\}$, where $\mathcal{N}_{1,0}=\mathcal{N}$.
Note that we can surely destroy the HUAV $i_L$ within a single round if we simply let $\widehat{\mathcal{H}}_1\leftarrow\{1,2,...,N\}$. However, destroying all UAVs is source intensive, and we should try to detect the HUAV as accurately as possible. Hence, the goal is to find the HUAV $i_L$ at once while destroying as few number of UAVs as possible, which can be expressed as 
\begin{align}
	(\mathbf{P1}):\;	\max_{\widehat{\mathcal{H}}_1}\quad&J_s= \mathbbm{1}\{i_L\in\widehat{\mathcal{H}}_1\} -|\widehat{\mathcal{H}}_1| \label{optimization_single}\\
	\operatorname{s.t.}\quad&\eqref{equ:vf},\eqref{equ:h_relation}, \tag{{\ref{optimization_single}{a}}}\label{constraint_1_1_}\\
	&M=1, N=1+m_1,\tag{{\ref{optimization_single}{b}}}\label{constraint_1_2_}\\
	&f(\cdot) \text{ and } i_L \text{ are unknown}.\tag{{\ref{optimization_single}{c}}}\label{constraint_1_3_}	
\end{align}
Note that the optimal solution of $(\mathbf{P1})$ will be $\widehat{\mathcal{H}}_1=\{i_L\}$, which means we detect the precise HUAV $i_L$ at once. The corresponding maximum objective is $J_s^*=\mathbbm{1}\{i_L\in\{i_L\}\} -|\{i_L\}|=0$.

We then investigate the general cluster head detection problem of the USNET with multiple clusters, where $M$, $\mathcal{H}$ and $m_j,\forall j$ are all unknown in advance. On the one hand, we should decrease the number of detection rounds $R$ to increase the efficiency of detections. On the other hand, we should minimize the detection redundancies
to reduce the overhead of destroying. 
Hence, the goal can be expressed as 
\begin{align}
	(\mathbf{P2}):\;	\min_{\widehat{\mathcal{H}}_1,...,\widehat{\mathcal{H}}_R,R}\quad&J_m= R+\frac{1}{N}\sum_{r=1}^{R}|\widehat{\mathcal{H}}_r\backslash\mathcal{H}_{r-1}| \label{optimization_multiple}\\
	\operatorname{s.t.}\quad&\eqref{equ:vf},\eqref{equ:h_relation},\eqref{equ:h_subset} \tag{{\ref{optimization_multiple}{a}}}\label{constraint_2_1_}\\
	&N=\sum_{j=1}^M1+m_j,\tag{{\ref{optimization_multiple}{b}}}\label{constraint_2_2_}\\
	&f(\cdot), M, \mathcal{H}, m_j, \;\forall j \text{ are unknown}.\tag{{\ref{optimization_multiple}{c}}}\label{constraint_2_3_}	
\end{align}
Note that the optimal solution of $(\mathbf{P2})$ will be $R=1, \widehat{\mathcal{H}}_1=\mathcal{H}$, which means we detect all HUAVs without redundant UAVs in the first round. The corresponding minimum objective is $J_m^*=1+|{\mathcal{H}}\backslash\mathcal{H}_{0}|=1$.

It is worth noting that $(\mathbf{P1})$ and $(\mathbf{P2})$ are both unsupervised learning problems since we have no prior supervised or semi-supervised training data on the USNET.

\section{Graph Attention Self-Supervised Learning}
\label{section:GASSL}
Let us consider the cluster head detection problem $(\mathbf{P1})$. As FUAVs inherently obey the IFS $f(\cdot)$, each FUAV focuses on the HUAV $i_L$ and gives little attention to other FUAVs.
In other words, HUAV $i_L$ can be viewed as the attention center among all $1+m_1$ UAVs. 
Hence, inspired from graph attention networks (GATs) \cite{gat}, we develop a graph attention self-supervised learning (GASSL) algorithm for $(\mathbf{P1})$. The network structure and training scheme of the GASSL are stated as follows.
\begin{figure*}[t]
	\centering
	\includegraphics[width=175mm]{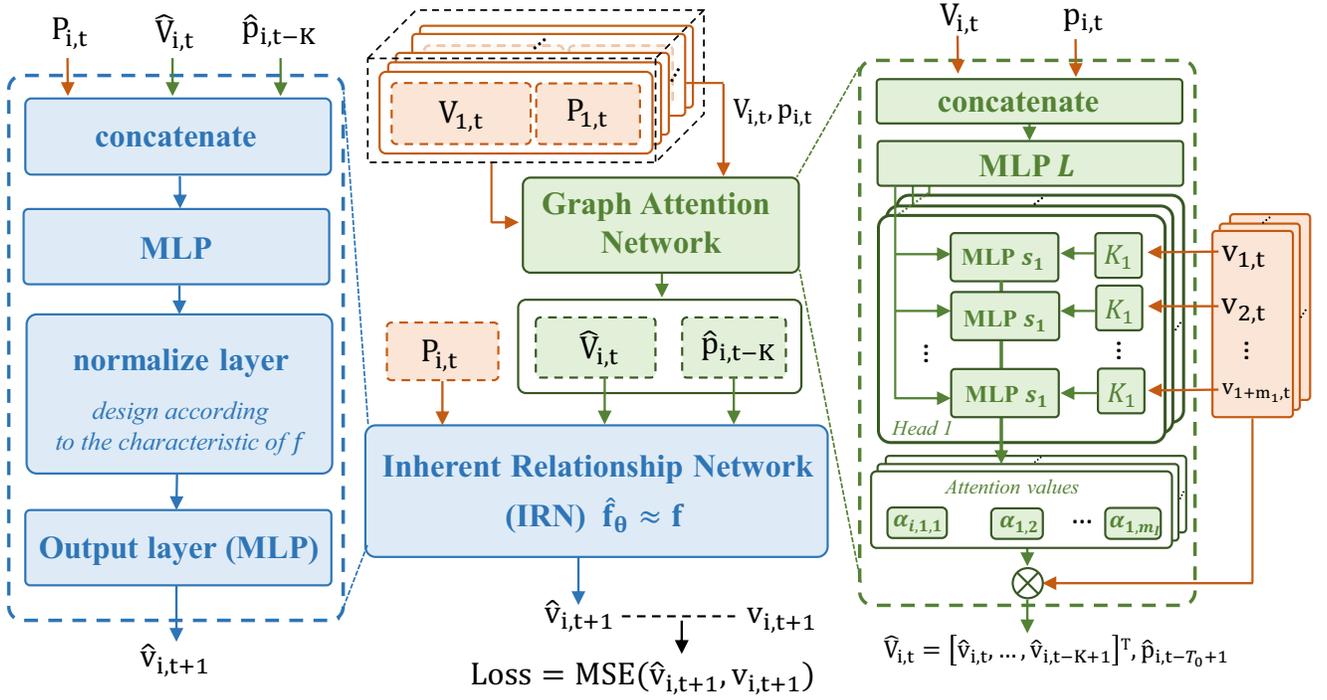}
	\caption{Network structure of graph attention self-supervised learning (GASSL). Green parts represent the modules of AGAT, blue parts are the modules of IFSN, and orange parts are the observed data on UAVs' speeds and positions.}
	\label{fig:gassl}
\end{figure*}
\subsection{Network Structure of GASSL}
The GASSL is mainly composed of two parts, including an adaptive GAT (AGAT) part and an inherent follow strategy network (IFSN) $\widehat{f}(\cdot)$ part. The goal of IFSN $\widehat{f}(\cdot)$ is to fit the IFS $f(\cdot)$ as much as possible, i.e.,
\begin{align}
	\widehat{f}(\cdot)\rightarrow f(\cdot).
\end{align}
Similar to the IFS $f(\cdot)$, the inputs of IFSN $\widehat{f}(\cdot)$ for each FUAV $i$ are $\mathbf{P}_{i,t}$, $\mathbf{V}_{i_L,t}$ and $\mathbf{p}_{i_L, t-T_0}$. However, the inputs $\mathbf{V}_{i_L,t}$ and $\mathbf{p}_{i_L, t-T_0+1}$ cannot be obtained directly since the index of HUAV $i_L$ is unknown. Hence, the AGAT is designed to produce the estimations $\widehat{\mathbf{V}}_{i,t}\triangleq[\widehat{\mathbf{v}}_{i,t-T_0+1},\widehat{\mathbf{v}}_{i,t-T_0+2},...,\widehat{\mathbf{v}}_{i,t}]^T$ and $\widehat{\mathbf{p}}_{i, t-T_0+1}$ for the history speeds $\mathbf{V}_{i_L,t}$ and positions $\mathbf{p}_{i_L, t-T_0+1}$ of HUAV $i_L$, respectively, i.e.,
\begin{align}
	\widehat{\mathbf{V}}_{i,t}\rightarrow \mathbf{V}_{i_L,t}, \quad \widehat{\mathbf{p}}_{i, t-T_0+1}\rightarrow \mathbf{p}_{i_L, t-T_0+1},
\end{align}
where $\widehat{\mathbf{v}}_{i,\tau}$ estimates the speed $\mathbf{v}_{i_L,\tau}$ of HUAV $i_L$ at time step $\tau$, $\forall \tau\in\{t,t+1,...,t-T_0+1\}$.
\reffig{fig:gassl} shows the detailed network structure of GASSL for any UAV $i$ in the cluster.
\subsubsection{Adaptive GAT}
The AGAT of UAV $i$ takes the time series topology matrix $\mathbf{X}_{1,t}$ of the cluster as the input, and calculates the \emph{attention weight} $\alpha_{i,i'}$ of UAV $i$ with respect to UAV $i'$ in the cluster, $\forall i'\in\mathcal{N}$. Specifically, the history speeds $\mathbf{V}_{i,t}$ and position $\mathbf{p}_{i,t-T_0+1}$ of UAV $i$ are concatenated and pre-processed with a trainable multi-layer perceptron (MLP) $L(\cdot)$ to obtain the \emph{query data}. Note that the MLP $L(\cdot)$ acts as the query matrix in the traditional GAT \cite{gat}.
As there are $T_0$ history speeds data and one history position data of each UAV $i'$ in $\mathbf{X}_{1,t}$, we construct $T_0+1$ \emph{attention head} modules in the AGAT. The $l$-th attention head is composed of a trainable MLP $s_l(\cdot)$ acting as the similarity function, and a trainable parameter $\mathbf{K}_l\in\mathbb{R}^{K\times 3}$, where $K\in\mathbb{N}_+$ is a hyperparameter and $l\in\{1,2,..,T_0+1\}$. 
The first $T_0$ attention heads deal with the history speeds of all UAVs, while the last attention head deals with the history positions of all UAVs. 
Specifically, the $l$-th ($l\neq T_0+1$) attention head transforms the speeds of all UAVs at time step $t-l+1$ linearly by $\mathbf{K}_l$, obtaining the \emph{key data} of the $l$-th attention head. Each 
\stripsep=0pt
\begin{strip}
	\hrule
	\begin{align}
		\label{equ:attention_i}
		\alpha_{i,i',l}=\frac{\exp\bigg (s_l\bigg(L\big(\mathbf{V}_{i,t}\updownarrow\mathbf{p}_{i,t-T_0+1}\big)\updownarrow\mathbf{K}_l\mathbf{v}_{i',t-l+1}\bigg)\bigg)}{\sum_{i''=1}^{m_1+1}\exp\bigg (s_l\bigg(L\big(\mathbf{V}_{i,t}\updownarrow\mathbf{p}_{i,t-T_0+1}\big)\updownarrow\mathbf{K}_l\mathbf{v}_{i'',t-l+1}\bigg)\bigg)}.
	\end{align}
	\begin{align}
		\label{equ:attention_i'}
		\alpha_{i,i',T_0+1}=\frac{\exp\bigg (s_{T_0+1}\bigg(L\big(\mathbf{V}_{i,t}\updownarrow\mathbf{p}_{i,t-T_0+1}\big)\updownarrow\mathbf{K}_{T_0+1}\mathbf{p}_{i',t-T_0+1}\bigg)\bigg)}{\sum_{i''=1}^{m_1+1}\exp\bigg (s_{T_0+1}\bigg(L\big(\mathbf{V}_{i,t}\updownarrow\mathbf{p}_{i,t-T_0+1}\big)\updownarrow\mathbf{K}_{T_0+1}\mathbf{v}_{i'',t-{T_0}+1}\bigg)\bigg)}.
	\end{align}
\end{strip}
key date is concatenated with query data and processed by similarity function $s_l(\cdot)$ and softmax operation, as shown in \eqref{equ:attention_i}, where $\updownarrow$ represents the concatenation operation, and $\alpha_{i,i',l}$ is the \emph{attention value} of UAV $i$ with respect to the speed $\mathbf{v}_{i',t-l+1}$ of UAV $i'$. The last attention head processes the positions of all UAVs at time step $t-T_0+1$ in the same way and outputs the attention values $\alpha_{i,i',T_0+1},\forall i'$ with respect to $\mathbf{p}_{i',t-T_0+1}$ as \eqref{equ:attention_i'}.
The attention weight $\alpha_{i,i'}$ of UAV $i$ with respect to UAV $i'$ is calculated as the average attention values in all $T_0+1$ attention heads, i.e.,
\begin{align}
	\alpha_{i,i'}=\frac{1}{T_0+1}\sum_{l=1}^{T_0+1}\alpha_{i,i',l}.
\end{align}
Denote the parameters of the AGAT for UAV $i$ as $\mathbf{\Phi}_i$.
Note that the attention weight $\alpha_{i',i}$ may not equal to $\alpha_{i,i'}$ since $\alpha_{i',i}$ is derived from the AGAT for UAV $i'$ with different parameters $\mathbf{\Phi}_{i'}$.
The \emph{value data} are directly made as the history speeds and positions themselves with no trainable transformation parameters. Then the element $\widehat{\mathbf{v}}_{i,t-l+1}$ in the output $\widehat{\mathbf{V}}_{i,t}$ is calculated as the weighted combination of the value data $\mathbf{v}_{i',t-l+1}$ with the attention values of the $l$-th attention head $\alpha_{i,i',l}$, i.e., 
\begin{align}
	\widehat{\mathbf{v}}_{i,t-l+1}=\sum_{i'=1}^{1+m_1}\alpha_{i,i',l}\mathbf{v}_{i',t-l+1}, \;\;l\neq T_0+1.
\end{align}
Similarly, the output $\widehat{\mathbf{p}}_{i, t-T_0+1}$ is derived from the weighted combination of the value data $\mathbf{p}_{i',t-T_0+1}$ with the attention values of the last attention head, i.e.,
\begin{align}
	\widehat{\mathbf{p}}_{i,t-T_0+1}=\sum_{i'=1}^{1+m_1}\alpha_{i,i',T_0+1}\mathbf{p}_{i',t-T_0+1}.
\end{align}

\subsubsection{Inherent Follow Strategy Network} The IFSN receives $\widehat{\mathbf{V}}_{i,t}$ and $\widehat{\mathbf{p}}_{i,t-T_0+1}$ from the AGAT,  and concatenates them together with the positions $\mathbf{P}_{i,t}$ of UAV $i$. The concatenations are standardized as zero-mean vectors with unit variance to make the speed and position data in the same order of magnitude.
The standardized vectors are processed with a trainable MLP to obtain the estimation $\widehat{\mathbf{v}}_{i,t+1}$ of the speed of UAV $i$ at time step $t+1$. 
Note that the parameters of IFSN $\widehat{f}(\cdot)$ are shared among all UAVs. We denote the parameters of IFSN $\widehat{f}(\cdot)$ as $\mathbf{\Gamma}$. Then all the parameters of GASSL for UAV $i$ can be denoted as $\mathbf{\Theta}_i\triangleq\{\mathbf{\Phi}_i,\mathbf{\Gamma}\}$.
The loss function $\mathcal{L}(\mathbf{\Theta}_i)$ for the GASSL for UAV $i$ is designed as the square error between the real speed $\mathbf{v}_{i,t+1}$ and the estimated speed $\widehat{\mathbf{v}}_{i,t+1}$, i.e.,
\begin{align}
	\mathcal{L}(\mathbf{\Theta}_i)&=\big(\mathbf{v}_{i,t+1}-\widehat{\mathbf{v}}_{i,t+1}\big)^2\notag\\
	&=\bigg(\mathbf{v}_{i,t+1}-\widehat{f}\big(\mathbf{P}_{i,t}, \widehat{\mathbf{V}}_{i,t},\widehat{\mathbf{p}}_{i,t-T_0+1}\big)\bigg)^2.
\end{align}
Note that $\mathcal{L}(\mathbf{\Theta}_i)\ge 0$.
As proved in Proposition \ref{proposition:1},  a sufficient condition for $\mathcal{L}(\mathbf{\Theta}_i)=0$ is that the IFSN $\widehat{f}(\cdot)$ completely fits the IFS $f(\cdot)$ and UAV $i$ only focuses on the speeds and positions of HUAV $i_L$.
\begin{proposition}
	\label{proposition:1}
	A sufficient condition for $\mathcal{L}(\mathbf{\Theta}_i)=0$ is that
	\begin{itemize}
		\item the IFSN $\widehat{f}(\cdot)$ completely fits the IFS $f(\cdot)$, i.e., $\widehat{f}(\cdot)=f(\cdot)$;
		\item UAV $i$ only focuses on the speeds and positions of HUAV $i_L$, i.e., $\alpha_{i,i_L,l}=1,\forall l$ and  $\alpha_{i,i',l}=0,\forall i'\neq i_L,\forall l$,
	\end{itemize}
	hold at the same time.
\end{proposition}
\begin{proof}
	When $\alpha_{i,i_L,l}=1,\forall l$ and  $\alpha_{i,i',l}=0,\forall i'\neq i_L,\forall l$, we have
	\begin{align}
		\widehat{\mathbf{v}}_{i,t-l+1}&=\sum_{i'=1}^{1+m_1}\alpha_{i,i',l}\mathbf{v}_{i',t-l+1}\notag\\
		&=\sum_{i'=1}^{1+m_1}\mathbbm{1}\{i'=i_L\}\mathbf{v}_{i',t-l+1}\notag\\
		&=\mathbf{v}_{i_l,t-l+1},
	\end{align}
	and 
	\begin{align}	\widehat{\mathbf{p}}_{i,t-T_0+1}&=\sum_{i'=1}^{1+m_1}\alpha_{i,i',T_0+1}\mathbf{p}_{i',t-T_0+1}\notag\\
		&=\sum_{i'=1}^{1+m_1}\mathbbm{1}\{i'=i_L\}\mathbf{p}_{i',t-T_0+1}\notag\\
		&=\mathbf{p}_{i_L,t-T_0+1}.
	\end{align}
	Hence, we have $\widehat{\mathbf{V}}_{i,t}=\mathbf{V}_{i_L,t}$ and $\widehat{\mathbf{p}}_{i,t-T_0+1}=\mathbf{p}_{i_L,t-T_0+1}$. Since $\widehat{f}(\cdot)=f(\cdot)$, we can derive
	\begin{align}
		\mathcal{L}(\mathbf{\Theta}_i)&=\bigg(\mathbf{v}_{i,t+1}-\widehat{f}\big(\mathbf{P}_{i,t}, \widehat{\mathbf{V}}_{i,t},\widehat{\mathbf{p}}_{i,t-T_0+1}\big)\bigg)^2\notag\\
		&=\bigg(\mathbf{v}_{i,t+1}-f\big(\mathbf{P}_{i,t}, {\mathbf{V}}_{i_L,t},{\mathbf{p}}_{i_L,t-T_0+1}\big)\bigg)^2\notag\\
		&=\big(\mathbf{v}_{i,t+1}-{\mathbf{v}}_{i,t+1}\big)^2\notag\\
		&=0.
	\end{align}
So far, we have completed the proof. 
\end{proof}
Therefore, we can reduce the value of the loss function $\mathcal{L}(\mathbf{\Theta}_i)$ towards zero by training the networks in GASSL, and obtain the index of HUAV $i_L$ based on the attention weights. It is worth noting that as the IFS $f(\cdot)$ may only relates to part of the HUAV's history speeds or positions, attention heads dealing with irrelevant speed or position data of HUAVs can produce attention weights with random values after training (see Appendix \ref{appendix:A-2}). Hence, we choose the index of the maximum attention weights as the index of HUAV for UAV $i$, i.e.,
\begin{align}
	\label{equ:max_i_L}
	i_{L}^{(i)}\leftarrow\arg\max_{i'} \alpha_{i,i'}.
\end{align}
Note that $i_L^{(i)}$ only represents the choice of HUAV from the perspective of UAV $i$.
Therefore, we apply the GASSL to all $N$ UAVs and obtain the indexes of HUAV from the view of all UAVs. 
Then the probability of UAV $i$ being the HUAV $i_L$ can be calculated as 
\begin{align}
	\label{equ:prob_1}
	c_{i}=\frac{1}{N}\sum_{i'\neq i}\mathbbm{1}\{i_L^{(i')}=i\}.
\end{align}
The UAVs with top-$K$ probabilities are selected as the candidate HUAVs, i.e.,
\begin{align}
	\label{equ:prob_gen}
	\widehat{\mathcal{H}}_1\leftarrow\{i_1,i_2,...,i_K\mid c_{i_1}\ge \cdots\ge c_{i_N},i_1,...,i_N\in\mathcal{N}\},
\end{align}
where $K\in\mathbb{N}_+$ is a hyper-parameter.
Particularly, when $K=1$, the candidate HUAV is the UAV with the maximum probability, i.e.,
\begin{align}
	\label{equ:prob_2}
	\widehat{\mathcal{H}}_1\leftarrow\{ \arg\max_{i\in\mathcal{N}}c_i\}.
\end{align}
We next present the detailed training scheme of the GASSL.

\subsection{Training Scheme of GASSL}
\label{section:training_scheme}
\begin{figure*}[t]
	\centering
	\includegraphics[width=125mm]{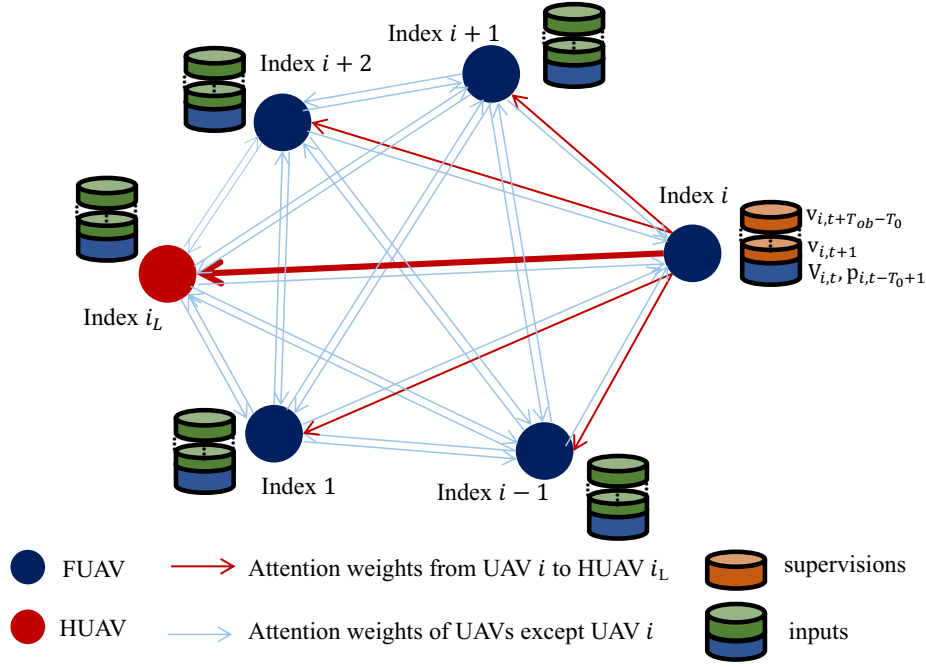}
	\caption{Learning scheme of GASSL.}
	\label{fig:ls}
\end{figure*}
As we observe the behaviors of the USNET for $T_{ob}$ time steps, we can construct a batch of $T_{ob}-T_0+1$ time series topology matrices $\mathbf{X}_{1,t}$, $\mathbf{X}_{1,t+1}$, ..., $\mathbf{X}_{1,t+T_{ob}-T_0}$. As shown in \reffig{fig:ls}, we feed the GASSL for UAV $i$ with the whole batch, where $\mathbf{v}_{i,t+1}$, $\mathbf{v}_{i,t+2}$,...,$\mathbf{v}_{i,t+T_{ob}-T_0}$ act as the supervised data. The loss function for the batch data $\mathcal{L}_\mathcal{B}(\mathbf{\Theta}_i)$ can be expressed as 
\begin{align}
	\label{equ:update_1}
	\mathcal{L}_\mathcal{B}(\mathbf{\Theta}_i)&=\sum_{b=1}^{T_{ob}-T_0}\big(\mathbf{v}_{i,t+b}-\widehat{\mathbf{v}}_{i,t+b}\big)^2\notag\\
	&=\sum_{b=1}^{T_{ob}-T_0}\bigg(\mathbf{v}_{i,t+b}-\widehat{f}\big(\mathbf{P}_{i,t+b-1}, \widehat{\mathbf{V}}_{i,t+b-1},\widehat{\mathbf{p}}_{i,t-T_0+b}\big)\bigg)^2.
\end{align} 
We update the parameters of the networks in GASSL for UAV $i$ with gradient descent method until convergence, i.e.,
\begin{align}
	\label{equ:update_2}
	\mathbf{\Phi}_i&\leftarrow	\mathbf{\Phi}_i-\beta\nabla_{	\mathbf{\Phi}_i}\mathcal{L}_\mathcal{B}(\mathbf{\Theta}_i),\\
	\label{equ:update_3}
	\mathbf{\Gamma}&\leftarrow	\mathbf{\Gamma}-\beta\nabla_{	\mathbf{\Gamma}}\mathcal{L}_\mathcal{B}(\mathbf{\Theta}_i),
\end{align}
where $\beta>0$ is the learning rate. The overall GASSL algorithm is summarized in Algorithm \ref{algorithm:1}. 

\begin{algorithm}[t]
	\normalsize\caption{GASSL Algorithm for $(\mathbf{P1})$}
	\label{algorithm:1}
	\setstretch{1} 
	{\bf Inputs:} The observations on the chronological positions and speeds of each UAV in the USNET with $M=1$ cluster and $N=1+m_1$ UAVs.\\
	{\bf Outputs:} The solutions $\widehat{\mathcal{H}}_1$ to $(\mathbf{P1})$.\\
	{\bf Initializations:} The parameter $\mathbf{\Phi}_i$ of the AGAT for UAV $i$, the parameter $\mathbf{\Gamma}$ of the IFSN $\widehat{f}(\cdot)$, the maximum number of training episodes $\Omega\in\mathbb{N}_+$, the hyper-parameter $K\in\mathbb{N}_+$.
	\begin{algorithmic}[1]
		\normalsize
		\For {$i=1$ to $N$}
		\State Construct a batch of $T_{ob}-T_0+1$ time series topology matrices $\mathbf{X}_{1,t}$, ... , $\mathbf{X}_{1,t+T_{ob}-T_0}$.
		\For {$\omega=1$ to $\Omega$} 
		\State Update the parameters of the AGAT and IFSN $\widehat{f}(\cdot)$ with \eqref{equ:update_1}, \eqref{equ:update_2} and \eqref{equ:update_3}.
		\EndFor
		\State Obtain the index of HUAV $i_L^{(i)}$ for UAV $i$  through \eqref{equ:max_i_L}.
		\EndFor
		\State Calculate the probability of each UAV being the HUAV with \eqref{equ:prob_1}, and derive $\widehat{\mathcal{H}}_1$ with \eqref{equ:prob_gen}.
	\end{algorithmic}
\end{algorithm}

\section{Multi-Cluster Graph Attention Self-Supervised Learning}
\label{section:MC-GASSL}
Let us consider the cluster head detection problem $(\mathbf{P2})$. As FUAVs in different clusters follows distinct HUAVs, directly applying GASSL to the USNET with multiple clusters may not perform well. 
Hence, we propose a multi-cluster graph attention self-supervised learning (MC-GASSL) algorithm that first clusters the USNET with metric learning method and then detects the cluster heads with GASSL. The structure of MC-GASSL is shown in \reffig{fig:mc-gassl}.
\begin{figure*}[t]
	\centering
	\includegraphics[width=150mm]{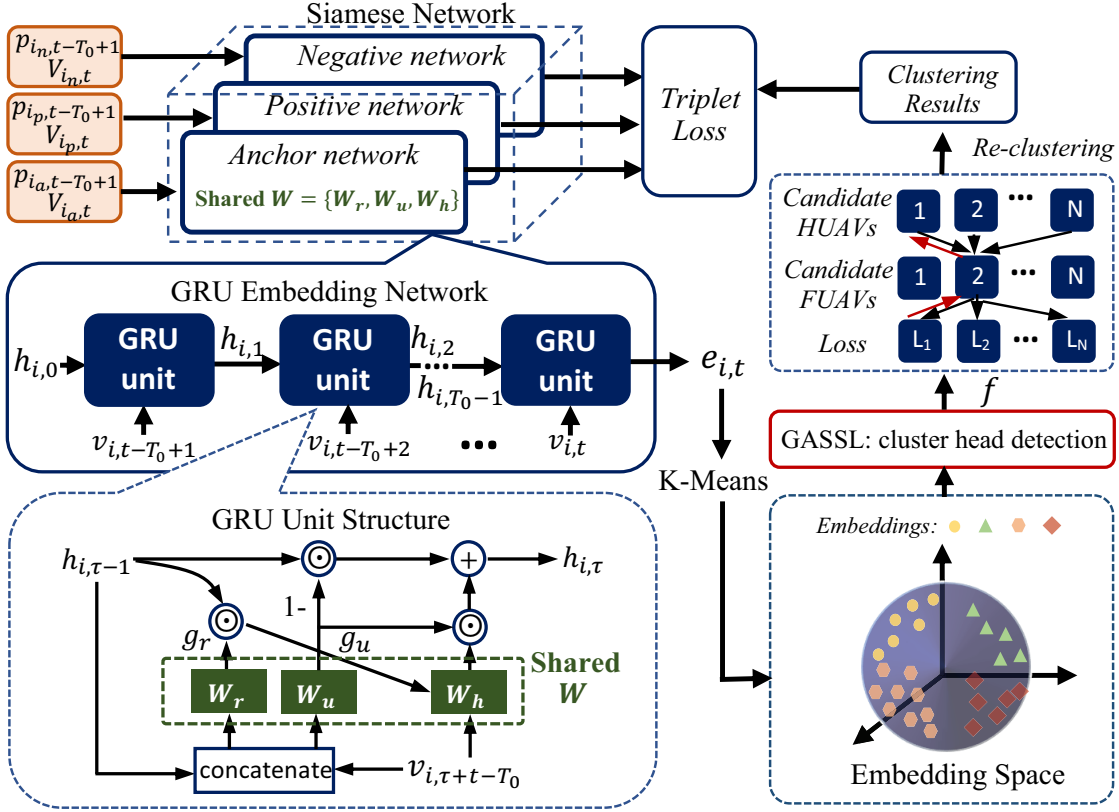}
	\caption{The structure of MC-GASSL.}
	\label{fig:mc-gassl}
\end{figure*}
\subsection{MC-GASSL Algorithm}
As the speeds and positions of UAVs are sequential data, we propose a GRU-based metric learning method to cluster the USNET before detecting HUAVs. To be specific, we design a \emph{GRU network} composed of $T_0$ identical GRU units \cite{gru} to extract the features of each UAV's positions and speeds. For UAV $i$, the GRU network receives the position $\mathbf{p}_{i,t-T_0+1}$ and the speed vectors in $\mathbf{V}_{i,t}$ sequentially, and outputs the extracted feature $\mathbf{e}_{i,t}\in\mathbb{R}^3$ of UAV $i$. Note that the extracted features of UAVs within the same cluster are expected to be close together in the feature space, while the extracted features of UAVs in distinct clusters are expected to be
separated from each other as far as possible. 
Since in this way, we can cluster the UAVs through the geometric distance between the extracted features of UAVs with traditional clustering methods, such as K-Means \cite{k-means}.

\subsubsection{Structure of the GRU Network}
The structure of the GRU network is shown in \reffig{fig:mc-gassl}.
Specifically, the $\tau$-th GRU unit takes both the output $\mathbf{h}_{i,\tau-1}\in\mathbb{R}^{3}$ of the previous GRU unit and $\mathbf{v}_{i,t-T_0+\tau}$ as input, and calculates $\mathbf{h}_{i,\tau}\in\mathbb{R}^{3}$ for the next GRU unit, where $\mathbf{h}_{i,0}\triangleq \mathbf{p}_{i,t-T_0+1}$ and $\mathbf{h}_{i,T_0}\triangleq\mathbf{e}_{i,t}$, $\forall \tau\in\{1,2,...,T_0\}$. Inside the $\tau$-th GRU, $\mathbf{h}_{i,\tau-1}$ and $\mathbf{v}_{i,t-T_0+\tau}$ are concatenated and transformed to the \emph{reset gate} $\mathbf{g}_r\in\mathbb{R}^{3}$ and \emph{update gate} $\mathbf{g}_u\in\mathbb{R}^{3}$, i.e.,
\begin{align}
	\mathbf{g}_r&=\sigma\big(\mathbf{W}_r(\mathbf{h}_{i,\tau-1}\updownarrow\mathbf{v}_{i,t-T_0+\tau})\big),\\
	\mathbf{g}_u&=\sigma\big(\mathbf{W}_u(\mathbf{h}_{i,\tau-1}\updownarrow\mathbf{v}_{i,t-T_0+\tau})\big),
\end{align}
where $\mathbf{W}_r\in\mathbb{R}^{3\times 6}$ and $\mathbf{W}_u\in\mathbb{R}^{3\times 6}$ are two trainable parameters, and $\sigma(\cdot)$ is the sigmoid function. The reset gate $\mathbf{g}_r$ resets $\mathbf{h}_{i,\tau-1}$ through Hadamard product, and the obtained reset vector is concatenated with $\mathbf{v}_{i,t-T_0+\tau}$ and transformed to vector $\mathbf{h}'_{i,\tau}\in\mathbb{R}^{3}$, i.e.,
\begin{align}
	\mathbf{h}'_{i,\tau}=\tanh \bigg(\mathbf{W}_h\big(\mathbf{v}_{i,t-T_0+\tau}\updownarrow(\mathbf{g}_r\odot\mathbf{h}_{i,\tau-1})\big)\bigg),
\end{align}
where $\mathbf{W}_h\in\mathbb{R}^{3\times 6}$ is a trainable parameter, and $\tanh(\cdot)$ is the hyperbolic tangent function. Then the output $\mathbf{h}_{i,\tau}$ of the $\tau$-th GRU unit is calculated as
\begin{align}
	\label{equ:gru}
	\mathbf{h}_{i,\tau}=(1-\mathbf{g}_u)\odot\mathbf{h}_{i,\tau-1}+\mathbf{g}_u\odot\mathbf{h}'_{i,\tau}.
\end{align}
Note that each GRU unit shares the same trainable parameters, and we denote the set of all the trainable parameters in the GRU network as $\mathbf{W}\triangleq\{\mathbf{W}_r,\mathbf{W}_u,\mathbf{W}_h\}$. 

\subsubsection{Metric Learning of the GRU Network}
To train the GRU network, we construct a \emph{Siamese network} consisting of three identical GRU networks with shared trainable parameters $\mathbf{W}$, and train the Siamese network with deep metric learning method \cite{metric_learning}. Specifically, suppose we have a dataset $\mathcal{D}_t=\{\mathcal{N}_{1,t},\mathbf{X}_{1,t},...,\mathcal{N}_{M_0,t},\mathbf{X}_{M_0,t}\}$ of the USNET with known clustering information $\mathcal{N}_{j,t}$ and known times series topology matrix $\mathbf{X}_{j,t}$, $\forall j\in\{1,2,...,M_0\}$, where $M_0\in\mathbb{N}_+$ is a constant. We arbitrarily choose an \emph{anchor} UAV $i_a\in\mathcal{N}_{j,t}$ from cluster $j$, and sample a \emph{positive} UAV $i_p$ from the same cluster with hard sampling method \cite{hard-sample}, i.e.,
\begin{align}
	\renewcommand{\arraystretch}{1}
	i_p=\arg\max_{i\in\mathcal{N}_{j,t}}\left\|\begin{bmatrix}
		\mathbf{p}_{i,t-T_0+1}\\
		\mathbf{V}_{i,t}
	\end{bmatrix}-\begin{bmatrix}
		\mathbf{p}_{i_a,t-T_0+1}\\
		\mathbf{V}_{i_a,t}
	\end{bmatrix}\right\|_2,
\end{align}
as well as a \emph{negative} UAV $i_n$ from clusters other than cluster $j$ with hard sampling method, i.e.,
\begin{align}
	\renewcommand{\arraystretch}{1}
	i_n=\arg\min_{i\notin\mathcal{N}_{j,t}}\left\|\begin{bmatrix}
		\mathbf{p}_{i,t-T_0+1}\\
		\mathbf{V}_{i,t}
	\end{bmatrix}-\begin{bmatrix}
		\mathbf{p}_{i_a,t-T_0+1}\\
		\mathbf{V}_{i_a,t}
	\end{bmatrix}\right\|_2.
\end{align}
We then input the position and speeds of anchor, positive, and negative UAVs into the three GRU networks in the Siamese network, respectively, and obtain the corresponding extracted features $\mathbf{e}_{i_a,t}$, $\mathbf{e}_{i_p,t}$ and $\mathbf{e}_{i_n,t}$.
We construct the loss function $\mathcal{L}^s(\mathbf{W})$ of the Siamese network in the form of the triplet loss function \cite{triplet-loss}, i.e.,
\begin{align}
	\label{equ:loss}
	\mathcal{L}^s(\mathbf{W})=\bigg[\left\|\mathbf{e}_{i_p,t}-\mathbf{e}_{i_a,t}\right\|_2-\left\|\mathbf{e}_{i_n,t}-\mathbf{e}_{i_a,t}\right\|_2+\gamma\bigg]_+.
\end{align}
where $\gamma>0$ is a constant. Note that $\mathcal{L}^s(\mathbf{W})\ge 0$, and  $\mathcal{L}^s(\mathbf{W})= 0$ means that the distance $\left\|\mathbf{e}_{i_a,t}-\mathbf{e}_{i_n,t}\right\|_2$ between the extracted features of anchor UAV and negative UAV is at least $\gamma$ larger than the distance $\left\|\mathbf{e}_{i_a,t}-\mathbf{e}_{i_p,t}\right\|_2$ between the extracted features of anchor UAV and positive UAV. Further explanations on the design of $\mathcal{L}^s(\mathbf{W})$ are expressed in Appendix \ref{appendix:loss_design}.
In practice, we usually sample several  anchor UAVs in various clusters at once and find the corresponding positive and negative UAVs to form a batch of training data. The loss function for the batch can be expressed as 
\begin{align}
	\mathcal{L}^s_\mathcal{B}(\mathbf{W})=\sum_{b=1}^B\bigg[\left\|\mathbf{e}_{i_p,t,b}-\mathbf{e}_{i_a,t,b}\right\|_2-\left\|\mathbf{e}_{i_n,t,b}-\mathbf{e}_{i_a,t,b}\right\|_2+\gamma\bigg]_+,
\end{align}
where $B\in\mathbb{N}_+$ is the size of the batch, and $\mathbf{e}_{i_a,t,b}$, $\mathbf{e}_{i_p,t,b}$ and $\mathbf{e}_{i_n,t,b}$ are the $b$-th anchor, positive, and negative UAV in the batch, respectively, $\forall b\in\{1,2,...,B\}$. Then the parameter $\mathbf{W}$ can be updated by gradient descent, i.e.,
\begin{align}
	\label{equ:update_w}
	\mathbf{W}\leftarrow\mathbf{W}-\beta'\nabla_{\mathbf{W}}\mathcal{L}^s_\mathcal{B}(\mathbf{W}),
\end{align}
where $\beta'>0$ is the learning rate. We iteratively sample batches from $\mathcal{D}$ and apply \eqref{equ:update_w} to the parameter $\mathbf{W}$ until convergence. Then we can obtain the nearly optimal parameters of the GRU network.
\subsubsection{Off-line Training of the GRU Network with Meta Learning}
However, we cannot directly obtain the training dataset $\mathcal{D}_t$ in practice, since the the cluster information $\mathcal{N}_{j,t},\forall j$ is unknown. Hence, the GRU network is not able to be trained on-line in the first round, and the clustering may be imprecise. To address this issue, we propose an off-line meta learning scheme for the GRU network. The meta learning scheme can find promising initial parameters for the GRU network in an off-line manner \cite{maml}. Specifically, we construct a set of $F\in\mathbb{N}_+$ IFSs $\mathcal{F}=\{\widetilde{f}_1,\widetilde{f}_2,...,\widetilde{f}_F\}$, where $\widetilde{f}_m$ is in the form of \eqref{equ:vf}, $\forall m\in\{1,2,...,F\}$.  Based on each IFS $\widetilde{f}_m$, we build a \emph{support dataset} $\mathcal{D}_{S}^m=\{\mathcal{N}_{S,1}^{m},\mathbf{X}_{S,1}^m,...,\mathcal{N}_{S,M_0}^m,\mathbf{X}_{S,M_0}^m\}$ of the USNET with $M_0$ clusters, where $\mathcal{N}_{S,j}^m$ is the index set of UAVs in cluster $j$, $\mathbf{X}_{S,j}^m$ represents the time series topology matrix of cluster $j$, and all FUAVs obey the same IFS $f_m$. 
Meanwhile, we construct a \emph{query dataset} $\mathcal{D}_{Q}^m=\{\mathcal{N}_{Q,1}^{m},\mathbf{X}_{Q,1}^m,...,\mathcal{N}_{Q,M_0}^m,\mathbf{X}_{Q,M_0}^m\}$ of the USNET with $M_0$ clusters, where $\mathcal{N}_{Q,j}^m$ is the index set of UAVs in cluster $j$, $\mathbf{X}_{Q,j}^m$ represents the time series topology matrix of cluster $j$, and all FUAVs obey the same IFS $f_m$. Hence the total dataset is composed of the support set $\mathcal{D}_S=\{\mathcal{D}_S^1,\mathcal{D}_S^2,...,\mathcal{D}_S^{F}\}$ and the query set $\mathcal{D}_Q=\{\mathcal{D}_Q^1,\mathcal{D}_Q^2,...,\mathcal{D}_Q^{F}\}$.
We conduct $M_0$ episodes of the meta learning for the Siamese network. To begin with, we randomly initialize the parameter of the Siamese network as $\mathbf{W}_0$.
In the $m$-th episode, we take $\mathcal{D}_{S}^m$ and $\mathcal{D}_Q^m$ to update the parameter of the Siamese network at the $m$-th episode $\mathbf{W}_m$. Specifically, a temporary Siamese network with parameter ${\mathbf{\Pi}}_{m}$ is endowed with $\mathbf{W}_{m-1}$, i.e., ${\mathbf{\Pi}}_{m}\leftarrow\mathbf{W}_{m-1}$. The parameter ${\mathbf{\Pi}}_{m}$ is updated in the direction of $\nabla_{{\mathbf{\Pi}}_{m}}\mathcal{L}^s_\mathcal{B}({\mathbf{\Pi}}_{m}; \mathcal{D}^m_S)$ by $\beta_{{meta}}>0$ step size, i.e.,
\begin{align}
	\label{support_set_update}
	{\mathbf{\Pi}}_{m}^\ddagger&={\mathbf{\Pi}}_{m}-\alpha_{{meta}}\nabla_{\mathbf{\Pi}_{m}}\mathcal{L}^s_\mathcal{B}({\mathbf{\Pi}}_{m}; \mathcal{D}^m_S)\notag\\
	&=\mathbf{W}_{m-1}-\alpha_{{meta}}\nabla_{\mathbf{W}_{m-1}}\bigg[\sum_{b=1}^B\bigg[\left\|\mathbf{e}_{i_p,b}^{S,m}-\mathbf{e}_{i_a,b}^{S,m}\right\|_2\notag\\
	&-\left\|\mathbf{e}_{i_n,b}^{S,m}-\mathbf{e}_{i_a,b}^{S,m}\right\|_2+\gamma\bigg]_+\bigg],
\end{align}
where $\mathbf{\Pi}_{m}^\ddagger$ is the updated parameter of the temporary Siamese network, $\mathbf{e}^{S,m}_{i_a,b}$, $\mathbf{e}^{S,m}_{i_p,b}$ and $\mathbf{e}^{S,m}_{i_n,b}$ are the extracted features of the $b$-th anchor, positive, and negative UAV sampled in $\mathcal{D}^m_S$ using hard sampling method.
The parameter of the Siamese network is updated in the direction of $\nabla_{\mathbf{\Pi}_{m}^\ddagger}\mathcal{L}^s_\mathcal{B}({\mathbf{\Pi}}_{m}^\ddagger; \mathcal{D}^m_Q)$ by $\alpha_{{meta}}$ step size, 
i.e.,
\begin{align}
	\label{query_set_update}
	\mathbf{W}_m&=\mathbf{W}_{m-1}-\alpha_{{meta}}\nabla_{\mathbf{\Pi}_{m}^\ddagger}\mathcal{L}^s_\mathcal{B}({\mathbf{\Pi}}_{m}^\ddagger; \mathcal{D}^m_Q)\notag\\
	&=\mathbf{W}_{m-1}-\alpha_{{meta}}\nabla_{\mathbf{\Pi}_{m}^\ddagger}\bigg[\sum_{b=1}^B\bigg[\left\|\mathbf{e}_{i_p,b}^{Q,m}-\mathbf{e}_{i_a,b}^{Q,m}\right\|_2\notag\\
	&-\left\|\mathbf{e}_{i_n,b}^{Q,m}-\mathbf{e}_{i_a,b}^{Q,m}\right\|_2+\gamma\bigg]_+\bigg],
\end{align}  
where $\mathbf{e}^{Q,m}_{i_a,b}$, $\mathbf{e}^{Q,m}_{i_p,b}$ and $\mathbf{e}^{Q,m}_{i_n,b}$ are the extracted features of the $b$-th anchor, positive, and negative UAV sampled in $\mathcal{D}^m_Q$ using the hard sampling method.
After $M_0$ episodes, we obtain the \emph{meta parameters} of the Siamese network $\mathbf{W}^\star\triangleq\mathbf{W}_{M_0}$ that act as the initial parameters of the GRU network.
\begin{algorithm}[t]
	\normalsize\caption{MC-GASSL Algorithm for $(\mathbf{P2})$}
	\label{algorithm:2}
	\setstretch{1} 
	{\bf Inputs:} The observations on the chronological positions and speeds of each UAV in the USNET.\\
	{\bf Outputs:} The solutions $\widehat{\mathcal{H}}_1,...,\widehat{\mathcal{H}}_R$ and $R$ to $(\mathbf{P2})$.\\
	{\bf Initializations:} The parameters of the GRU network $\mathbf{W}_0$, the support dataset $\mathcal{D}_S$ and the query dataset $\mathcal{D}_Q$, the maximum number of detection rounds $R_m\in\mathbb{N}_+$.\\
	{\bf Off-line Meta Learning:}
	\begin{algorithmic}[1]
		\normalsize
		\For {$m=1$ to $M_0$} 
		\State Sample a batch of anchor, positive, and negative UAVs by hard sampling method. Train one step on parameter $\mathbf{W}_{m-1}$ with \refeq{support_set_update}, and update $\mathbf{W}_{m-1}$ to $\mathbf{W}_m$ with \refeq{query_set_update}. 
		\EndFor
		\State Obtain the meta parameters $\mathbf{W}^\star$ of the GRU network.
		
	\end{algorithmic}
	{\bf On-line Executions:}
	\begin{algorithmic}[1]
		\normalsize 
		\For{$r=1$ to $R_m$}
		\State Observe the positions and speeds of all UAVs for $T_{ob}$ time steps, and calculate the extracted features of the positions and speeds of all UAVs with the GRU network.
		\State Estimate the number of clusters $\widehat{M}$ with gap-statistic method based on the extracted features, and cluster the UAVs with K-Means.
		\State Detect the cluster heads in each cluster with the GASSL algorithm in Algorithm \ref{algorithm:1}, and destruct all the detected HUAVs $\widehat{\mathcal{H}}_r$.
		\State Update the parameter of the GRU network to $\mathbf{W}_r^\star$ with on-line training method.
		\If{$\mathcal{H}_r=\emptyset$}
		\State Let $R\leftarrow r$. Break the loop.
		\EndIf 
		\EndFor
	\end{algorithmic}
\end{algorithm}

\subsubsection{Clustering in the Feature Space}
In the first detection round, the extracted features of UAVs are calculated through the GRU network with meta parameter $\mathbf{W}^\star$. In the following rounds, the GRU network is trained on-line with method proposed in Section \ref{section:on-line}, and the extracted features of UAVs are calculated through the GRU network with updated parameters. 
Recall that as the number of clusters $M$ is unknown, we need to estimate the number of clusters before clustering the USNET. We here utilize the gap-statistic method \cite{gap-statistic} on the extracted features in the feature space to produce the estimator $\widehat{M}$. Then the K-Means method is applied to the extracted features for clustering with number of clusters $\widehat{M}$.
\subsubsection{Cluster Head Detection with GASSL}
After the USNET is clustered, we use the GASSL algorithm to detect the HUAVs in each cluster. The candidate HUAVs set $\widehat{\mathcal{H}}_r$ is composed of all the candidate HUAVs detected in $\widehat{M}$ clusters. Destruct the UAVs in  $\widehat{\mathcal{H}}_r$ and the remaining UAVs will merge to new clusters automatically. Note that the parameters of the IFSN $\widehat{f}(\cdot)$ are shared among all the UAVs. 
\subsubsection{On-line Training of the GRU Network}
\label{section:on-line}
With the trained IFSN $\widehat{f}(\cdot)$, we can re-cluster the USNET to provide the dataset for the further on-line training of the Siamese network, as shown in \reffig{fig:mc-gassl}. Specifically, we list all UAVs as {candidate HUAVs} and endow each UAV $i$ a \emph{score} $c_i$. Each UAV finds the one in all other UAVs best suited to be its HUAV. The most suitable HUAV $i_H$ for UAV $i$ should minimize the mean square error loss between the predicted velocities $\widehat{f}(\mathbf{P}_{i,t+b-1},\mathbf{V}_{i_H,t_b-1},\mathbf{p}_{i_H,t-T_0+b})$ and the true velocities $\mathbf{v}_{i,t+b}$, i.e.,
\begin{align}
	i_H&=\arg\min_{i'\neq i}\notag\\
	& \sum_{b=1}^{T_{ob}-T_0}\bigg(\widehat{f}(\mathbf{P}_{i,t+b-1},\mathbf{V}_{i_H,t_b-1},\mathbf{p}_{i_H,t-T_0+b})-\mathbf{v}_{i,t+b}\bigg)^2.
\end{align}
The score of UAV $i'$ of being an HUAV is calculated as 
\begin{align}
	c_{i'}=\sum_{i\neq i'}\mathbbm{1}\{i_H=i'\}.
\end{align}
According to the scores, we choose top $\widehat{M}$ UAVs to be the HUAVs in the dataset for on-line training. Therefore, we can construct the on-line training dataset $\mathcal{D}_r=\{\mathcal{N}_{1,t_r},\mathbf{X}_{1,t_r},...,\mathcal{N}_{\widehat{M},t_r},\mathbf{X}_{\widehat{M},t_r}\}$ in the $r$-th detection round, where $\mathcal{N}_{j,t_r}$ contains an HUAV chosen in the $r$-th detection round and its FUAVs. The update rule for the parameter of the GRU network in the $r$-th detection round can be expressed as 
\begin{align}
	\mathbf{W}^\star_r=\mathbf{W}^\star_{r-1}-\beta'\nabla_{\mathbf{W}^\star_{r-1}}\mathcal{L}^s_{\mathcal{B}}(\mathbf{W}^\star_{r-1};\mathcal{D}_r),
\end{align}
where $\mathbf{W}^\star_{0}\triangleq\mathbf{W}^\star$.
The overall algorithm of MC-GASSL is summarized in Algorithm \ref{algorithm:2}.

\section{Simulation Results}
\label{section:simulations}
\begin{table*}[t]
	\centering
	\caption{{Type and descriptions of the IFS $f(\cdot)$}}
	\label{table:vf_type}
	\begin{tabular}{cccc}
		\hline
		\rowcolor[gray]{0.9}
		\small\textbf{{Type}}&\small\textbf{{Notations}}&\small \textbf{{Follow-up Function $f(\cdot)$}}&	\small\textbf{{Descriptions}}\\
		\hline
		\specialrule{0em}{2.5pt}{2.5pt}
		\small\makecell[c]{{$\mathbf{1}$}}&\small$f_1$&\small\makecell[c]{{$\mathbf{v}_{i,t+1}=\text{norm}(\kappa_0\mathbf{v}_{i_L,t}+\kappa_n\mathbf{n})$}}&\small\makecell[c]{{proportional to the previous}\\ {speed $\mathbf{v}_{i_L,t}$ }}\\
		\specialrule{0em}{2.5pt}{2.5pt}
		\cline{1-4}
		\specialrule{0em}{2.5pt}{2.5pt}
		\small\makecell[c]{{$\mathbf{2}$}}&\small$f_2$&\small\makecell[c]{{$
				\mathbf{v}_{i,t+1}=\text{norm}(\sum_{\tau=0}^1\kappa_\tau\mathbf{v}_{i_L,t-\tau}+\kappa_n\mathbf{n})$}}&\small\makecell[c]{{linear combination of }\\{$\mathbf{v}_{i_L,t},\mathbf{v}_{i_L,t-1}$}}\\
		\specialrule{0em}{2.5pt}{2.5pt}
		\cline{1-4}
		\specialrule{0em}{2.5pt}{2.5pt}
		\small\makecell[c]{{$\mathbf{3}$}}&\small$f_3$&\small\makecell[c]{{$
				\mathbf{v}_{i,t+1}=\text{norm}(\sum_{\tau=0}^2\kappa_\tau\mathbf{v}_{i_L,t-\tau}+\kappa_n\mathbf{n})$}}&\small\makecell[c]{{linear combination of}\\{$\mathbf{v}_{i_L,t},\mathbf{v}_{i_L,t-1},\mathbf{v}_{i_L,t-2}$}}\\
		\specialrule{0em}{2.5pt}{2.5pt}
		\cline{1-4}
		\specialrule{0em}{2.5pt}{2.5pt}
		\small\makecell[c]{{$\mathbf{4}$}}&\small$f_4$&\small\makecell[c]{{$
				\mathbf{v}_{i,t+1}=\text{norm}(\sum_{\tau=0}^1\kappa_\tau\mathbf{v}_{i_L,t-\tau}\odot\mathbf{v}_{i_L,t-\tau}+\kappa_n\mathbf{n})$}}&\small\makecell[c]{{quadratic combination of}\\{$\mathbf{v}_{i_L,t}, \mathbf{v}_{i_L,t-1}$}}\\
		\specialrule{0em}{2.5pt}{2.5pt}
		\cline{1-4}
		\specialrule{0em}{2.5pt}{2.5pt}
		\small\makecell[c]{{$\mathbf{5}$}}&\small$f_5$&\small\makecell[c]{{$
				\mathbf{v}_{i,t+1}=\text{norm}(\sum_{\tau=0}^2\kappa_\tau\mathbf{v}_{i_L,t-\tau}\odot\mathbf{v}_{i_L,t-\tau}+\kappa_n\mathbf{n})$}}&\small\makecell[c]{{quadratic combination of}\\{$\mathbf{v}_{i_L,t},\mathbf{v}_{i_L,t-1},\mathbf{v}_{i_L,t-2}$}}\\
		\specialrule{0em}{2.5pt}{2.5pt}
		\cline{1-4}
		\specialrule{0em}{2.5pt}{2.5pt}
		\small\makecell[c]{{$\mathbf{6}$}}&\small$f_6$&\small\makecell[c]{{$
				\mathbf{v}_{i,t+1}=\text{norm}\big(\kappa_0\mathbf{v}_{i_L,t}+\kappa_n\mathbf{n}+$}\\ {$\mathbbm{1}\{\left\|\mathbf{p}_{i_L,t}-\mathbf{p}_{i,t}\right\|_2>\kappa_r\}\kappa_p\frac{\mathbf{p}_{i_L,t}-\mathbf{p}_{i,t}}{\left\|\mathbf{p}_{i_L,t}-\mathbf{p}_{i,t}\right\|_2}\big)$}}&\small\makecell[c]{{following  speed $\mathbf{v}_{i_L,t}$, keep }\\{in range $\kappa_r$ with HUAV}}\\
		\specialrule{0em}{2.5pt}{2.5pt}
		\cline{1-4}
		\specialrule{0em}{2.5pt}{2.5pt}
		\small\makecell[c]{{$\mathbf{7}$}}&\small$f_7$&\small\makecell[c]{{$
				\mathbf{v}_{i,t+1}=\text{MLP}(\mathbf{P}_{i,t},\mathbf{V}_{i_L,t},\mathbf{p}_{i_L,t-3})$}}&\small\makecell[c]{{fully connected neural network}\\{with inputs $\mathbf{P}_{i,t}, \mathbf{V}_{i_L,t},\mathbf{p}_{i_L,t-3}$}}\\
		\specialrule{0em}{2.5pt}{2.5pt}
		\hline
		\hline
		\specialrule{0em}{2.5pt}{2.5pt}
		\multicolumn{4}{c}{\small\makecell[c]{{$\kappa_0=1,\kappa_1=1,\kappa_2=1,\kappa_3=1,\kappa_p=1,\kappa_n=0.05,\kappa_r=60,\text{norm}(\cdot)=\frac{\cdot}{\left\|\cdot\right\|_2}$,} \\ { $\mathbf{n}\in\mathbb{R}^3$ represents noise, and $\mathbf{n}\sim\mathcal{N}(\mathbf{0},\mathbf{I}),$}
		}}\\
		\specialrule{0em}{2.5pt}{2.5pt}
		\hline
	\end{tabular}
\end{table*}
\begin{table*}[t]
	\centering
	\caption{{The detection rate ($\%$) of $(\mathbf{P1})$ with GASSL.}}
	\label{table:gassl}
	\begin{tabular}{c|c|c|c|c|c|c|c|c|c|c|c}
		\specialrule{0em}{1pt}{1pt}
		\toprule[2pt]
		\hline
		\multirow{2}{*}{\small\textbf{{Type}}}&\multirow{2}{*}{\small\makecell[c]{{\textbf{Follow-up}}\\{ \textbf{Function $f(\cdot)$}}}}&\multicolumn{10}{c}{\small\textbf{Number of FUAVs $m_1$}}\\
		\cline{3-12}
		&&\small\makecell[c]{{$\mathbf{2}$}}&\small\makecell[c]{{$\mathbf{3}$}}&\small\makecell[c]{{$\mathbf{5}$}}&\small\makecell[c]{{$\mathbf{10}$}}&\small\makecell[c]{{$\mathbf{15}$}}&\small\makecell[c]{{$\mathbf{20}$}}&\small\makecell[c]{{$\mathbf{25}$}}&\small\makecell[c]{{$\mathbf{30}$}}&\small\makecell[c]{{$\mathbf{40}$}}&\small\makecell[c]{{$\mathbf{50}$}}\\
		\hline
		\small$\mathbf{1}$&$f_1$&\small\makecell[c]{{$98.1\%$}}&\small$100.0\%$&\small$100.0\%$&\small$100.0\%$&\small$99.8\%$&\small$99.7\%$&\small$100.0\%$&\small$100.0\%$&\small$98.9\%$&\small$99.9\%$\\
		\hline
		\small$\mathbf{2}$&$f_2$&\small$99.8\%$&\small$100.0\%$&\small$100.0\%$&\small$100.0\%$&\small$100.0\%$&\small$98.3\%$&\small$99.2\%$&\small$99.1\%$&\small$100.0\%$&\small$100.0\%$\\
		\hline
		\small$\mathbf{3}$&$f_3$&\small$99.9\%$&\small$100.0\%$&\small$100.0\%$&\small$100.0\%$&\small$100.0\%$&\small$100.0\%$&\small$100.0\%$&\small$99.9\%$&\small$98.9\%$&\small$100.0\%$\\
		\hline
		\small$\mathbf{4}$&$f_4$&\small$92.0\%$&\small$93.4\%$&\small$99.2\%$&\small$91.4\%$&\small$98.3\%$&\small$98.5\%$&\small$99.1\%$&\small$99.5\%$&\small$94.2\%$&\small$97.5\%$\\
		\hline
		\small$\mathbf{5}$&$f_5$&\small$98.4\%$&\small$91.4\%$&\small$98.5\%$&\small$99.1\%$&\small$99.3\%$&\small$95.4\%$&\small$98.2\%$&\small$100.0\%$&\small$98.1\%$&\small$99.3\%$\\
		\hline
		\small$\mathbf{6}$&$f_6$&\small$98.2\%$&\small$100.0\%$&\small$99.9\%$&\small$98.5\%$&\small$99.1\%$&\small$99.2\%$&\small$99.9\%$&\small$100.0\%$&\small$98.9\%$&\small$99.5\%$\\
		\hline
		\small$\mathbf{7}$&$f_7$&\small$99.9\%$&\small$98.7\%$&\small$100.0\%$&\small$99.8\%$&\small$100.0\%$&\small$99.2\%$&\small$99.4\%$&\small$98.3\%$&\small$100.0\%$&\small$99.2\%$\\
		\hline
		\midrule[1pt]
		\specialrule{0em}{2pt}{2.2pt}
		\multicolumn{12}{c}{\small\makecell[c]{{The detection rate is calculated as $1+\overline{J_s}$, since we let $K=1$,}\\{The IFS $f(\cdot)$ is unknown when detecting.}}}\\
		\specialrule{0em}{1.5pt}{1pt}
		\hline
		\bottomrule[2pt]
	\end{tabular}
\end{table*}
\begin{figure*}
	\setlength{\abovecaptionskip}{0.5cm}
	\centering
	\subfigure[Attention weights]{
		\label{fig:simulations:1-1} 
		\includegraphics[width=80mm]{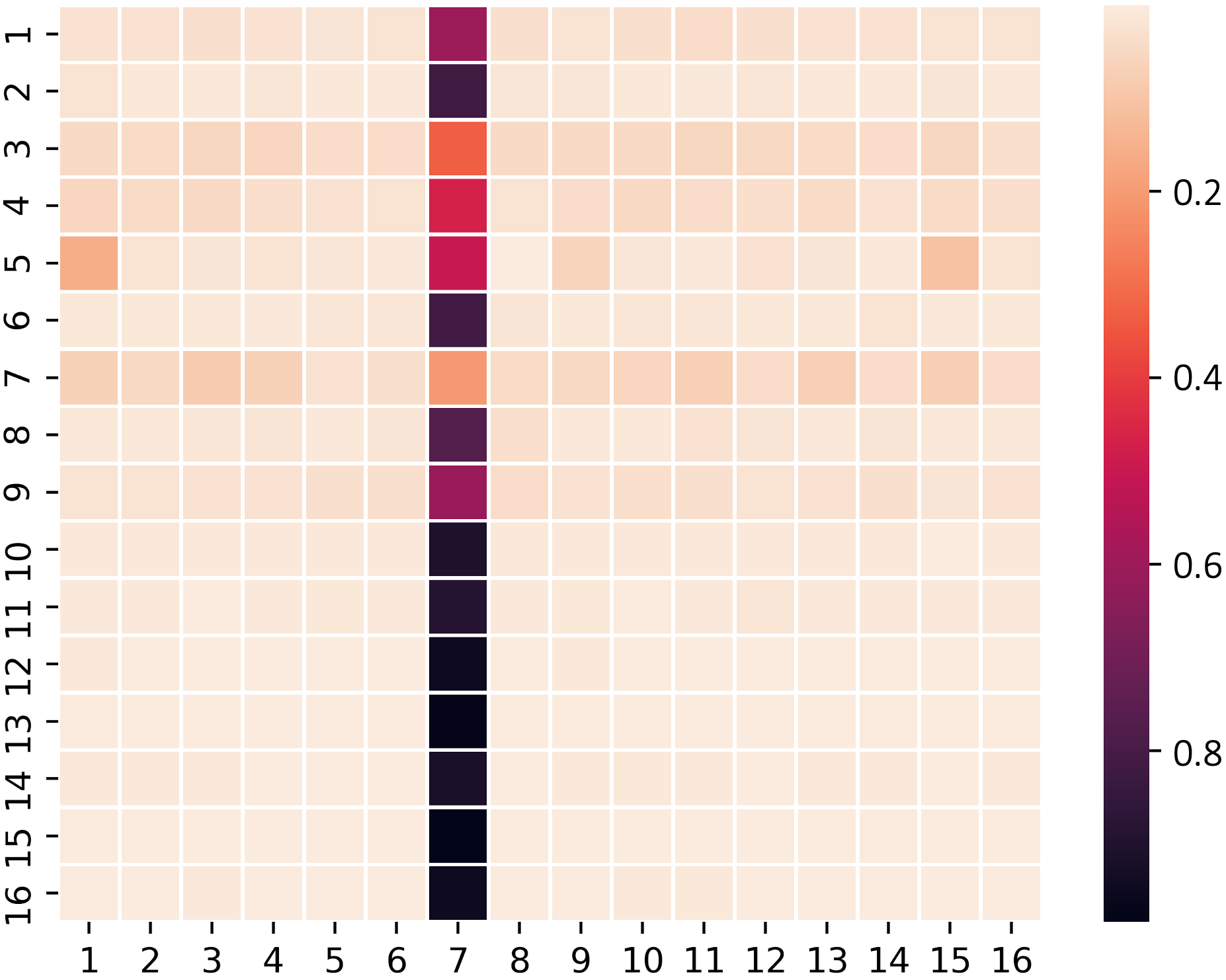}}
	\subfigure[Attention values of head 1.]{
		\label{fig:simulations:1-2} 
		\includegraphics[width=80mm]{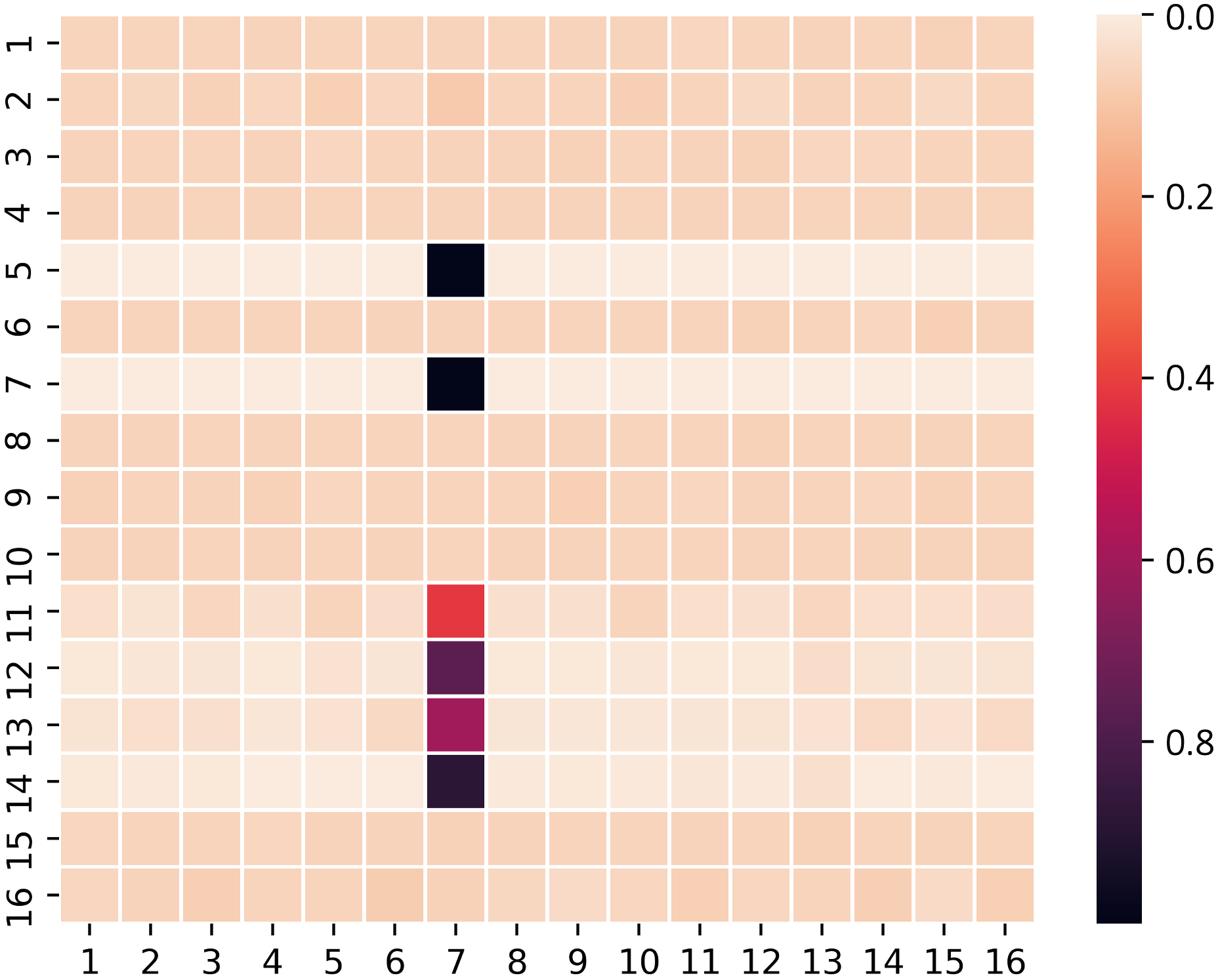}}
	
	\subfigure[Attention values of head 2.]{
		\label{fig:simulations:1-3} 
		\includegraphics[width=80mm]{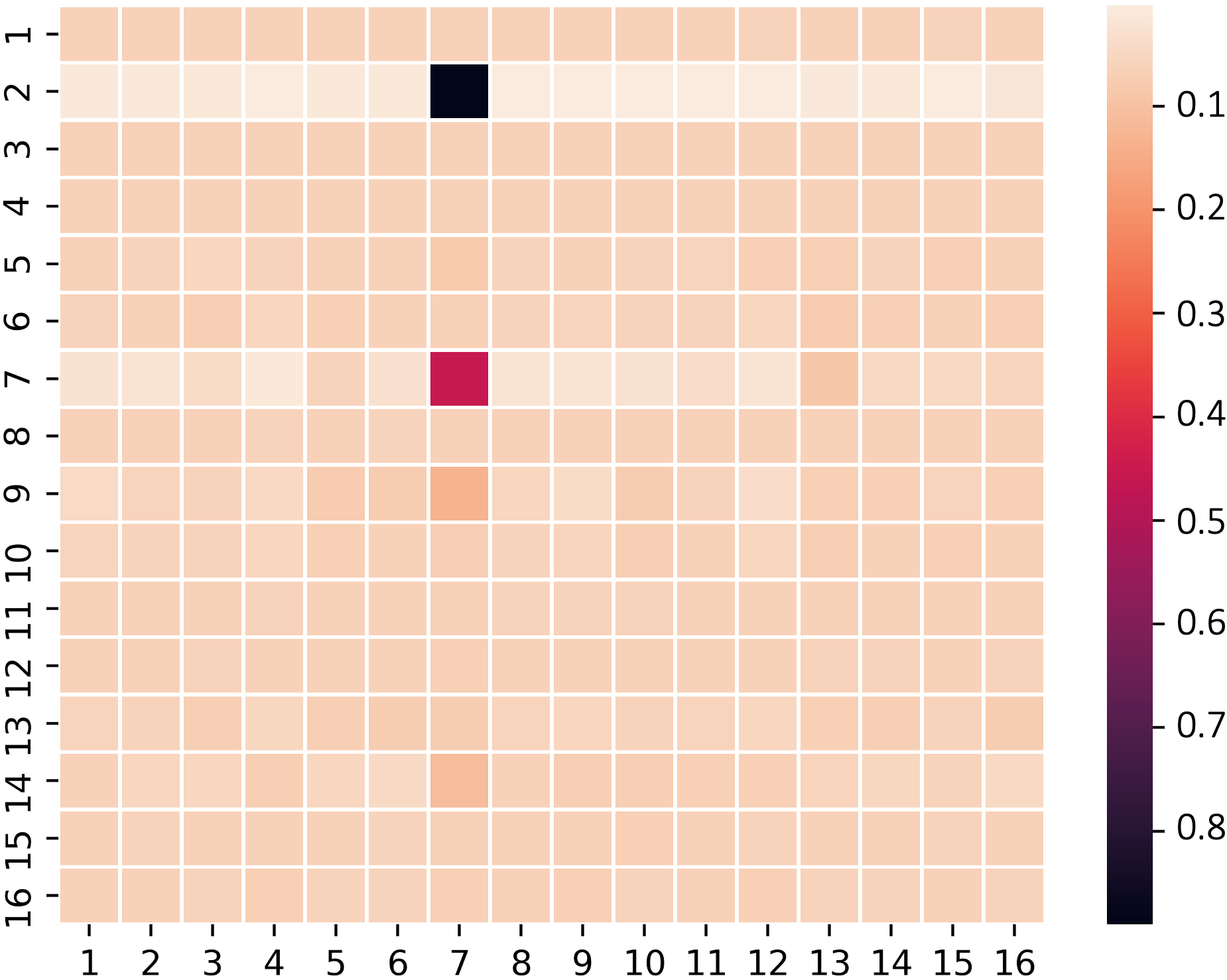}}
	\subfigure[Attention values of head 3.]{
		\label{fig:simulations:1-4} 
		\includegraphics[width=80mm]{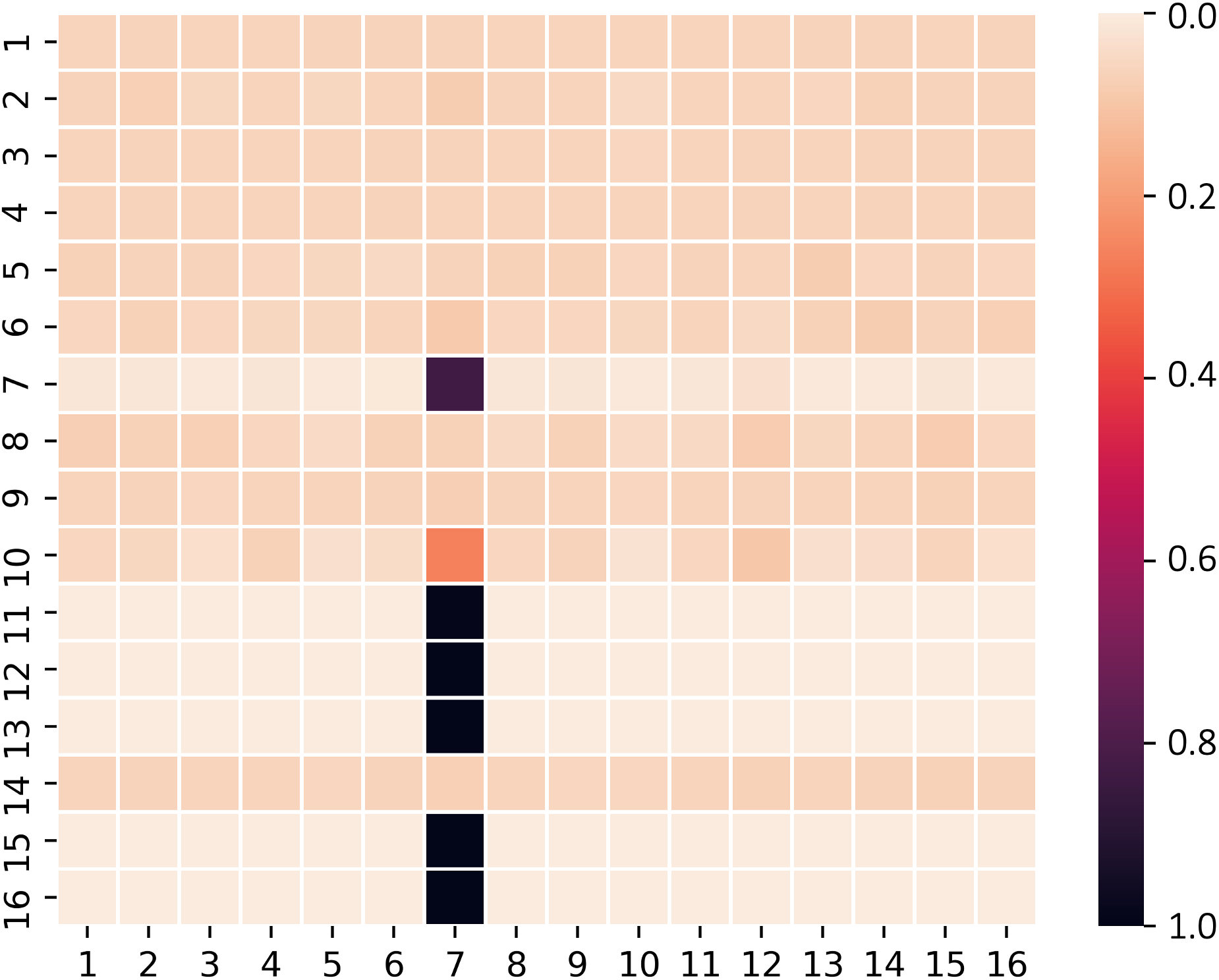}}
	
	\subfigure[Attention values of head 4.]{
		\label{fig:simulations:1-5} 
		\includegraphics[width=80mm]{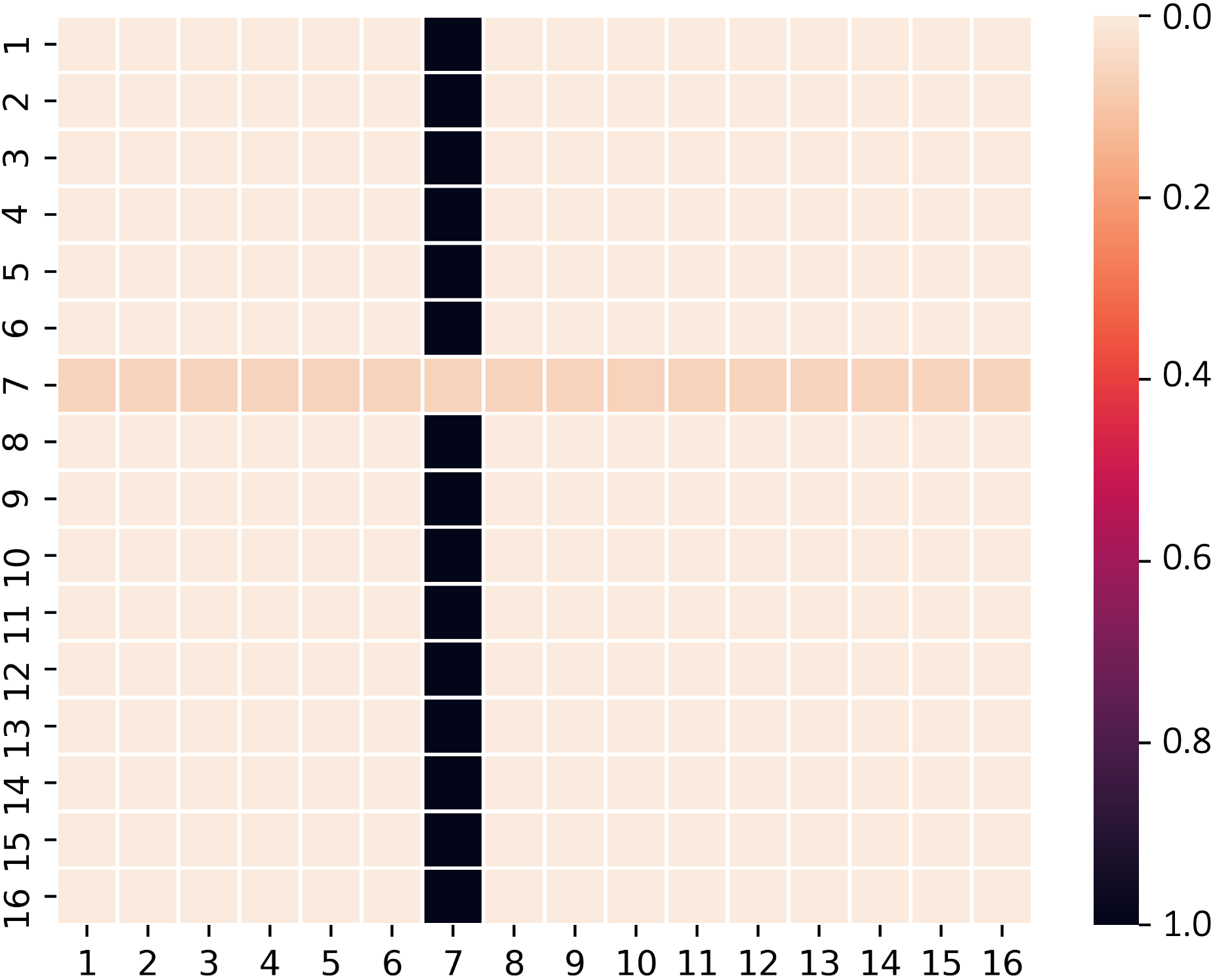}}
	\subfigure[Attention values of head 5.]{
		\label{fig:simulations:1-6} 
		\includegraphics[width=80mm]{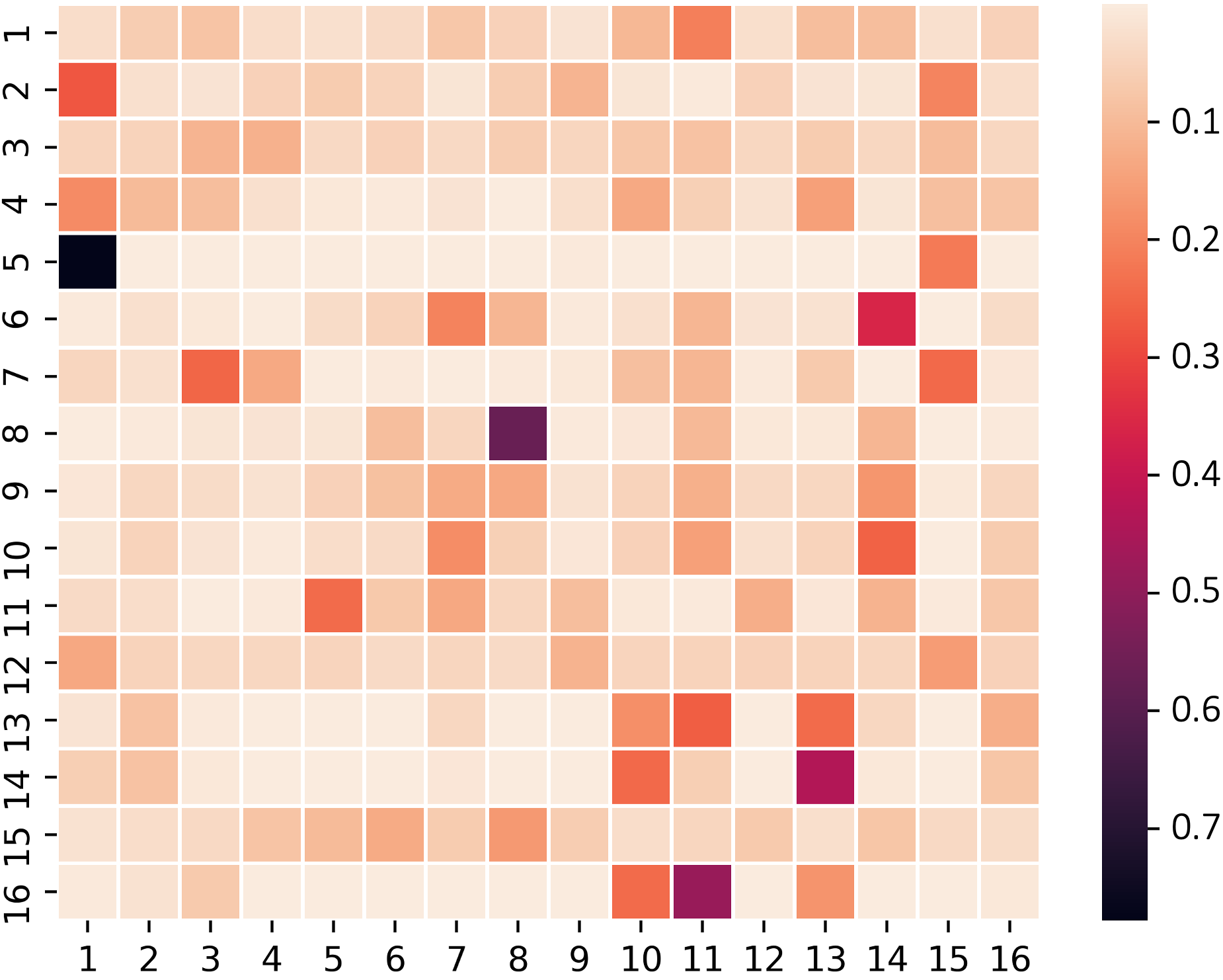}}
	
	\caption{Heat map of attentions weights and attention values of each attention head. The real index of HUAV is $i_L=7$, while the detected HUAV $\widehat{\mathcal{H}}_1=\{7\}$.}
	\label{fig:simulation:attention-1}
\end{figure*}
\begin{figure*}[t]
	\centering
	\subfigure[Attention weights between UAVs]{
		\label{fig:simulations:2-1} 
		\includegraphics[width=85mm]{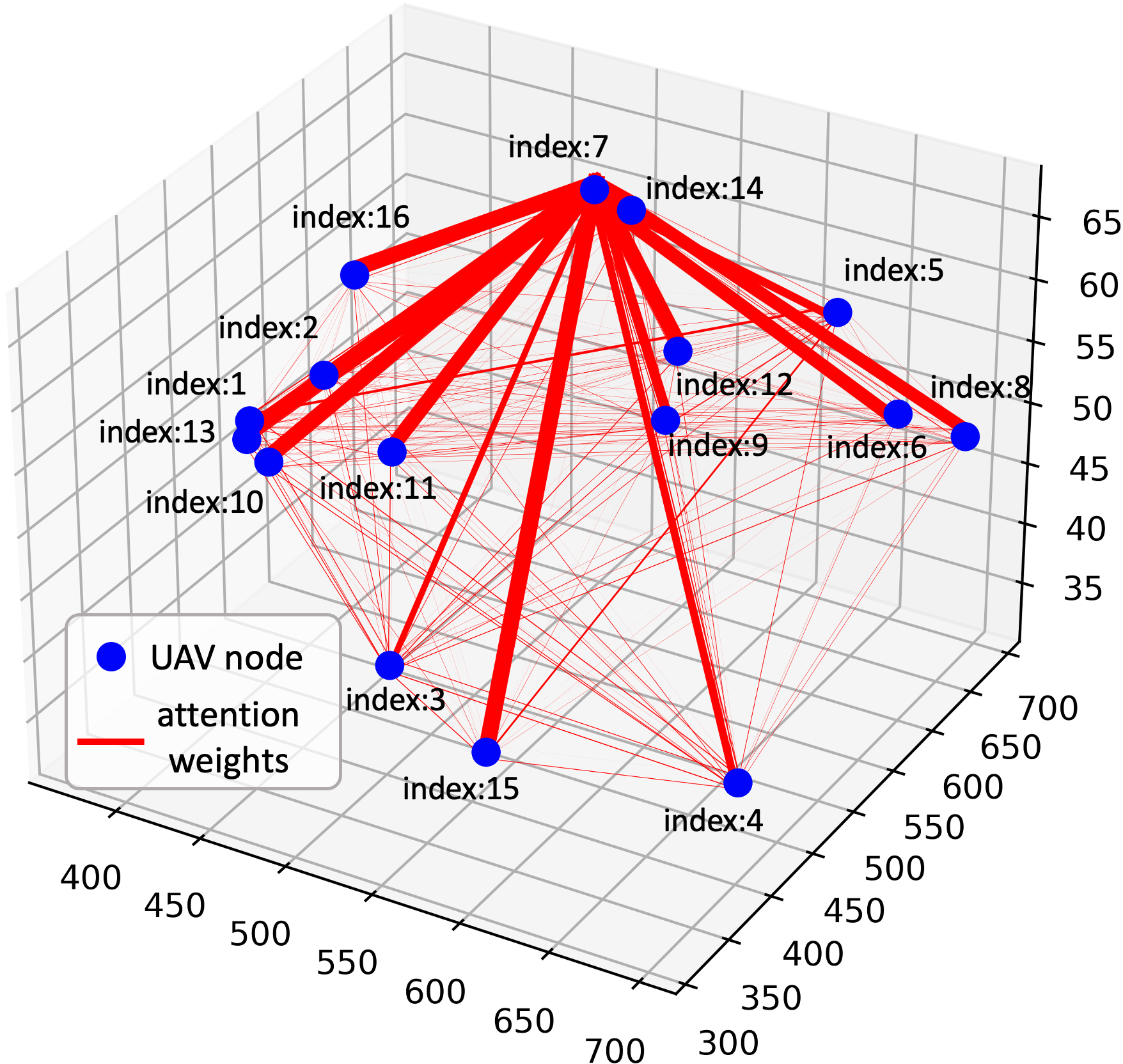}}
	\subfigure[Hierarchical graph structure of the USNET.]{
		\label{fig:simulations:2-2} 
		\includegraphics[width=85mm]{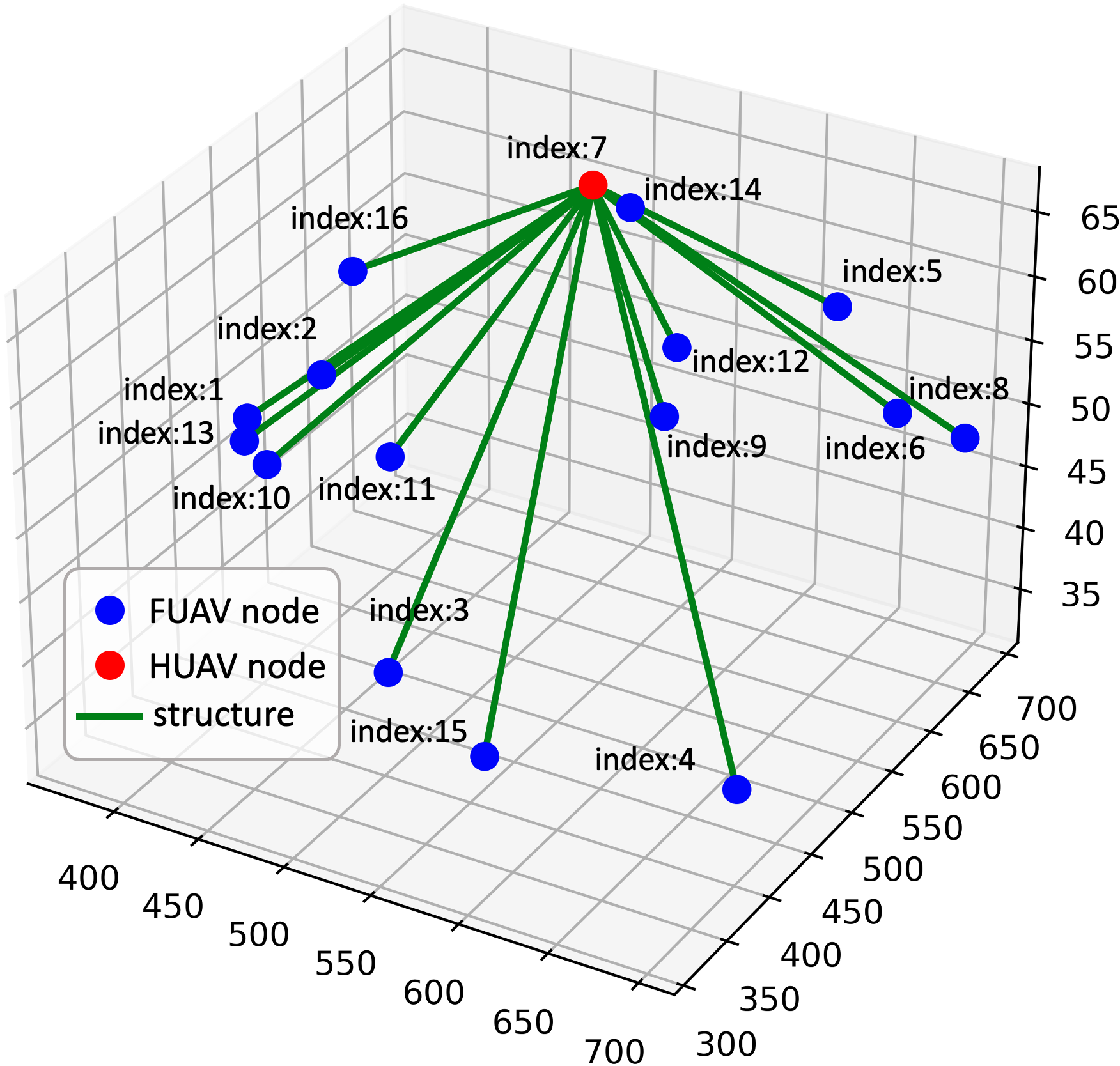}}
	\caption{Attention weights between UAVs and the derived hierarchy graph structure of the USNET.}
	\label{fig:simulation:2}
\end{figure*}
\begin{figure*}[h]
	\centering
	\subfigure[Attention weights between UAVs]{
		\label{fig:meta-loss} 
		\includegraphics[width=85mm]{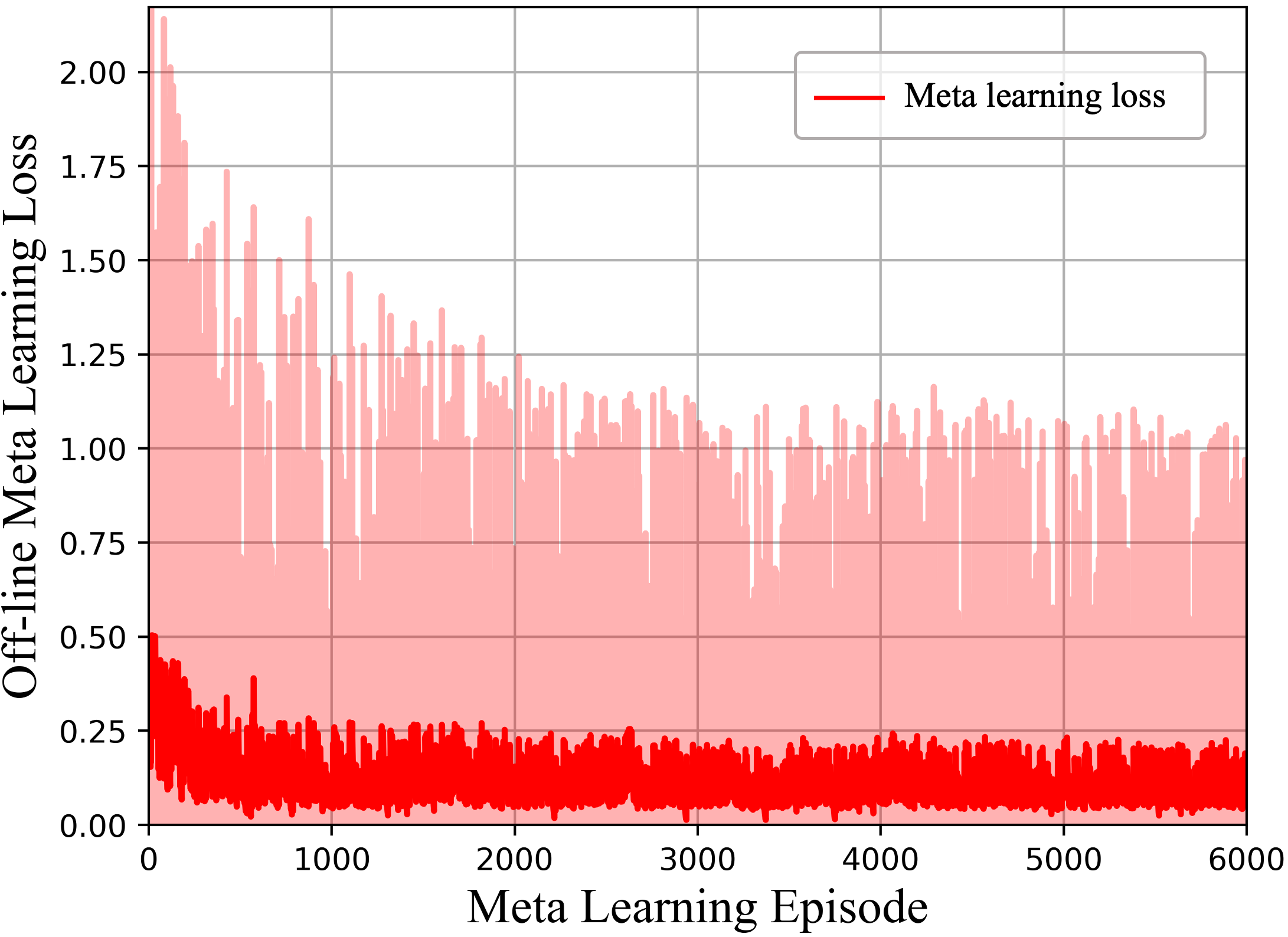}}
	\subfigure[Hierarchical graph structure of the USNET.]{
		\label{fig:meta-compare-loss} 
		\includegraphics[width=85mm]{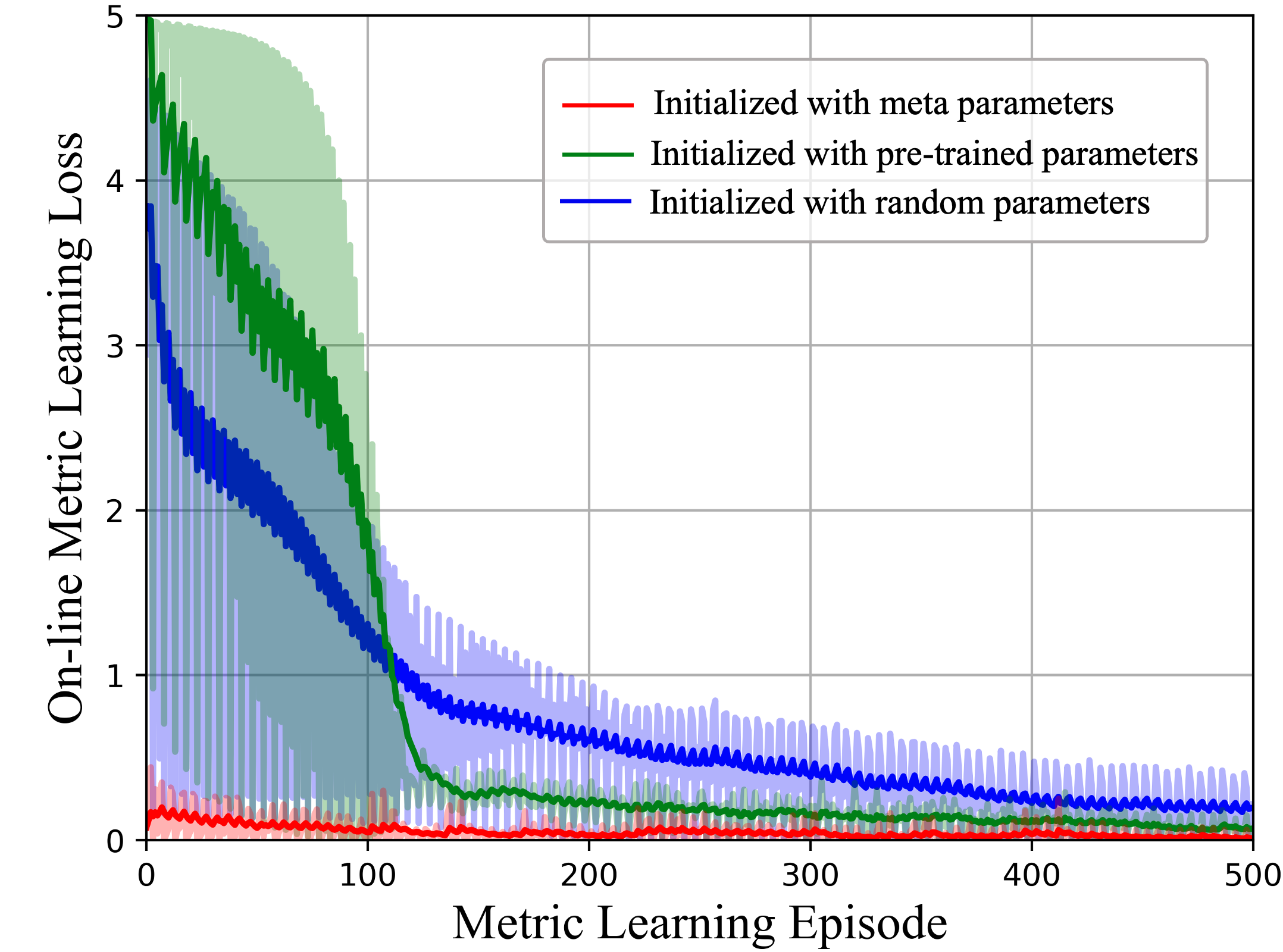}}
	\caption{Attention weights between UAVs and the derived hierarchy graph structure of the USNET.}
	\label{fig:meta}
\end{figure*}
\begin{figure}
	\centering
	\includegraphics[width=90mm]{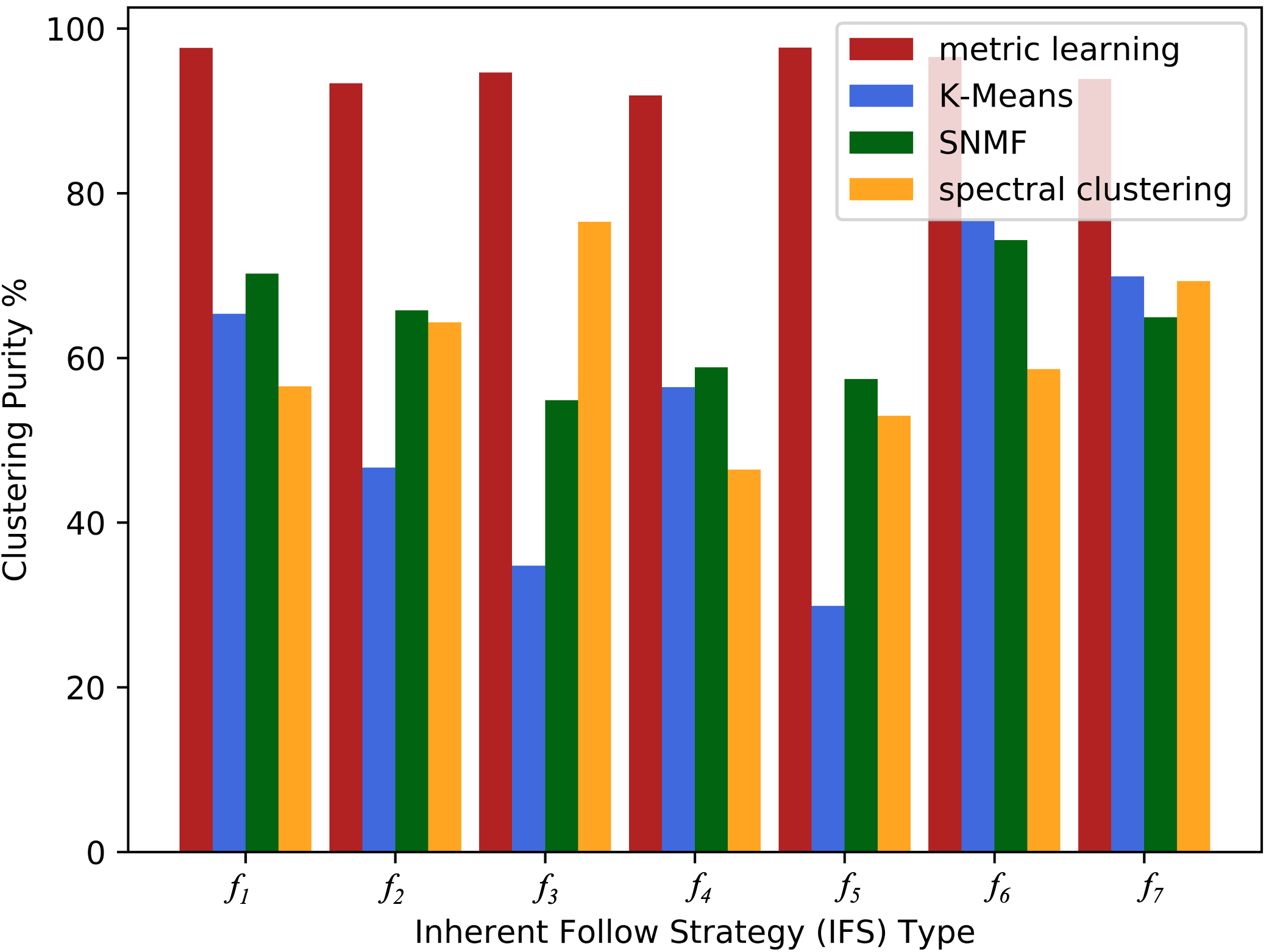}
	\caption{The clustering purity with different algorithms.}
	\label{fig:purity-compare}
\end{figure}

In the simulation, the UAVs in the USNET are initially distributed in a 1,000m$\times$1,000m$\times$100m three-dimensional space. The number of FUAVs in distinct clusters may be different, but not less than $2$, i.e., $m_j\ge 2,\forall j\in\mathcal{M}$. The speeds of distinct HUAVs are generated independently and randomly, while FUAVs in different clusters obey the same IFS $f(\cdot)$. Let $T_0= 4$, and let the observation time period be $T_{ob}=100$. In addition, the query matrix $L(\cdot)$ and similarity function $s_l(\cdot),\forall l\in\{1,2,...,T_0+1\}$ in the AGAT are MLPs of one and two hidden layers, respectively. The MLP in the IFSN $\widehat{f}(\cdot)$ is composed of three hidden layers, and there is no constraint layer in the IFSN $\widehat{f}(\cdot)$.
As shown in Table \ref{table:vf_type}, we implement $7$ different types of the IFS $f(\cdot)$, including functions of the linear and quadratic combinations of the HUAV's history speeds $f_1$ to $f_5$, functions of positions $f_6$, and even neural networks $f_7$.

\subsection{Simulation Results of $(\mathbf{P1})$ with GASSL}
Endow the hyper-parameter $K$ in the GASSL with 1, i.e., $K\leftarrow 1$. We simulate $(\mathbf{P1})$ with all seven types of the IFS $f(\cdot)$ under different number $m_1$ of FUAVs 1,000 times each.
The detection rate of the GASSL are shown in Table \ref{table:gassl}. 
We can see that the average object function $\overline{J_s}$ to $(\mathbf{P1})$ achieves the optimal value $J_s^*=0$ in many cases, indicating a detection rate\footnote{We define the \emph{detection rate} as the ratio between the times detecting the HUAV $i_L$ successfully and the total simulation times $10$. Then the detection rate can be calculated as $\overline{J_s}+1$. For example, $\overline{J_s}=J_s^*=0$ means a $100\%$ detection rate. } of $100\%$. The average object function $\overline{J_s}$ has value of $-0.018$, indicating an average detection rate of $98.2\%$. Hence, $(\mathbf{P1})$ can be efficiently solved with the GASSL.

Take the case where $f=f_1$ and $m_1=15$ as an example\footnote{More examples of the attention weights in different cases are represented in Appendix \ref{appendix:heat-map}.}. Note that the real index of HUAV is $i_L=7$. \reffig{fig:simulation:attention-1} shows the heat map of the attentions between UAVs when using the GASSL, where the $i$-th row in \reffig{fig:simulations:1-1} displays the attention weights of UAV $i$ with respect to all UAVs, and the $i$-th row in \reffig{fig:simulations:1-2} to \reffig{fig:simulations:1-6} displays the attention values of UAV $i$ to all UAVs in the attention head $1$ to $5$, respectively. We can see that UAVs are paying the most attention to UAV $7$ in  all the attention heads except head $5$. From the attention weights, we can determine the detected HUAV as $\widehat{\mathcal{H}}_1=\{7\}$, which is consistent with the real index of HUAV $i_L=7$. Fig.~6 indicates the attention relationship between UAVs and the derived hierarchy structure of the USNET. Specifically, in \reffig{fig:simulations:2-1}, the thickness of the line reflects the magnitude of the attention weight between two UAVs, and UAV $7$ has the most great attentions from all UAVs in the USNET.
\reffig{fig:simulations:2-2} shows the hierarchy structure of the USNET, where UAV $7$ acts as the HUAV. 
\begin{figure*}[t]
	\centering
	\subfigure[Initial feature space]{
		\label{fig:feature-space:1} 
		\includegraphics[width=85mm]{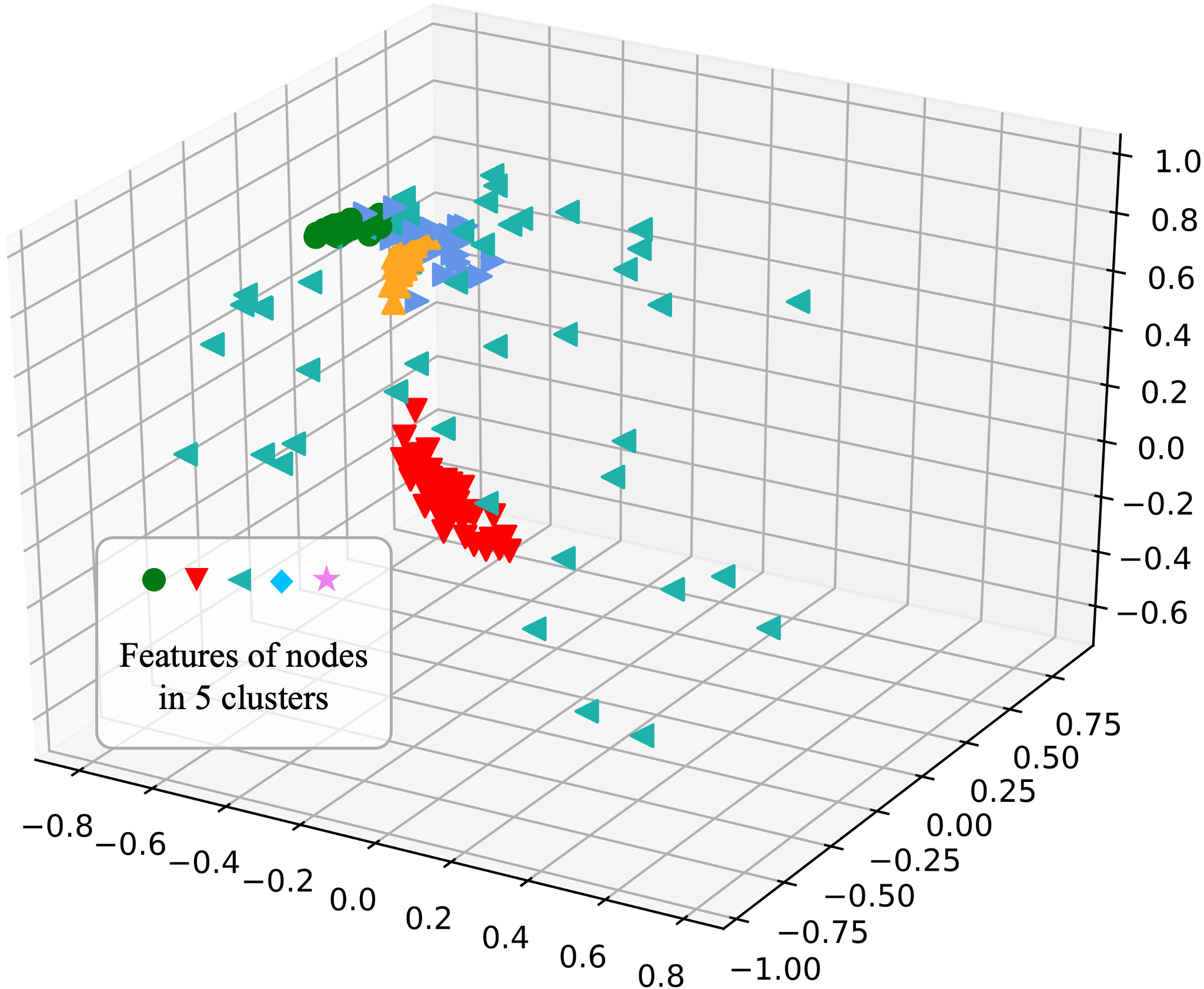}}
	\subfigure[Feature space after $500$ metric learning episodes.]{
		\label{fig:feature-space:2} 
		\includegraphics[width=85mm]{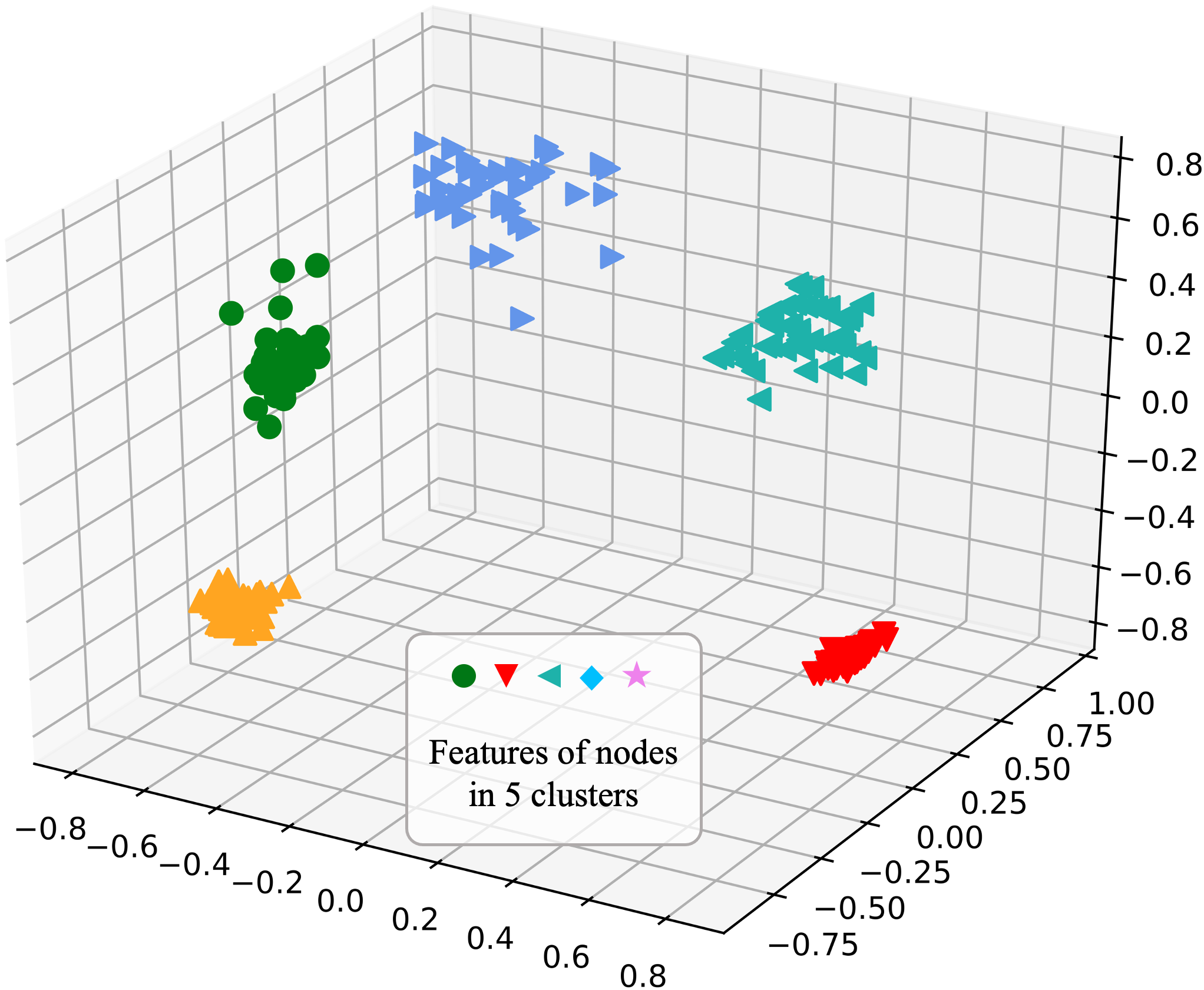}}
	\caption{The UAVs' feature changing in the metric learning process.}
	\label{fig:feature-space}
\end{figure*}
\begin{figure*}[t]
	\centering
	\subfigure[The USNET has $M=5$ UAV clusters, and $N=82$ UAVs.]{
		\label{fig:1} 
		\includegraphics[width=85mm]{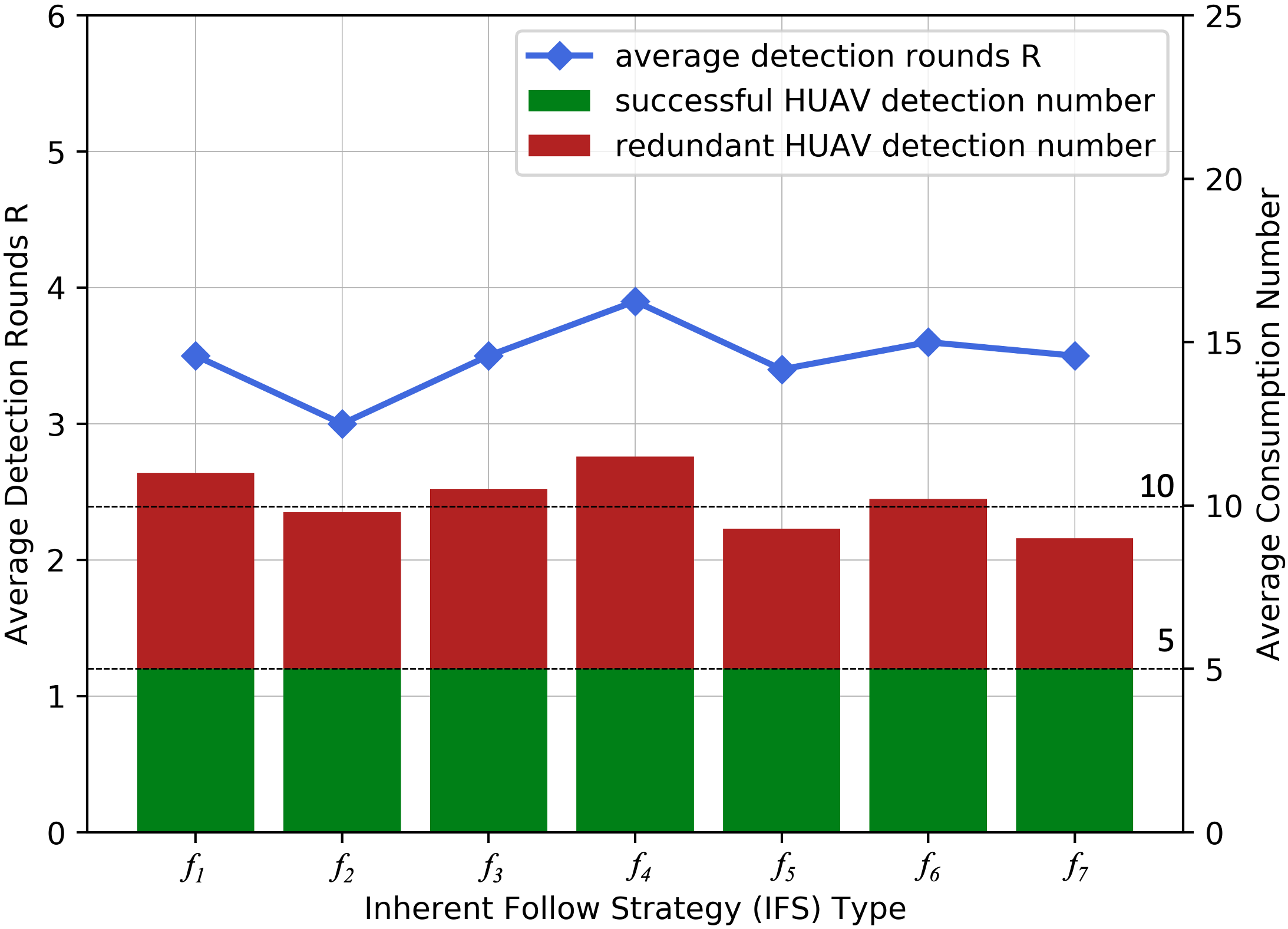}}
	\subfigure[The USNET obeys the IFS $f_1$ and has  $N=82$ UAVs.]{
		\label{fig:2} 
		\includegraphics[width=85mm]{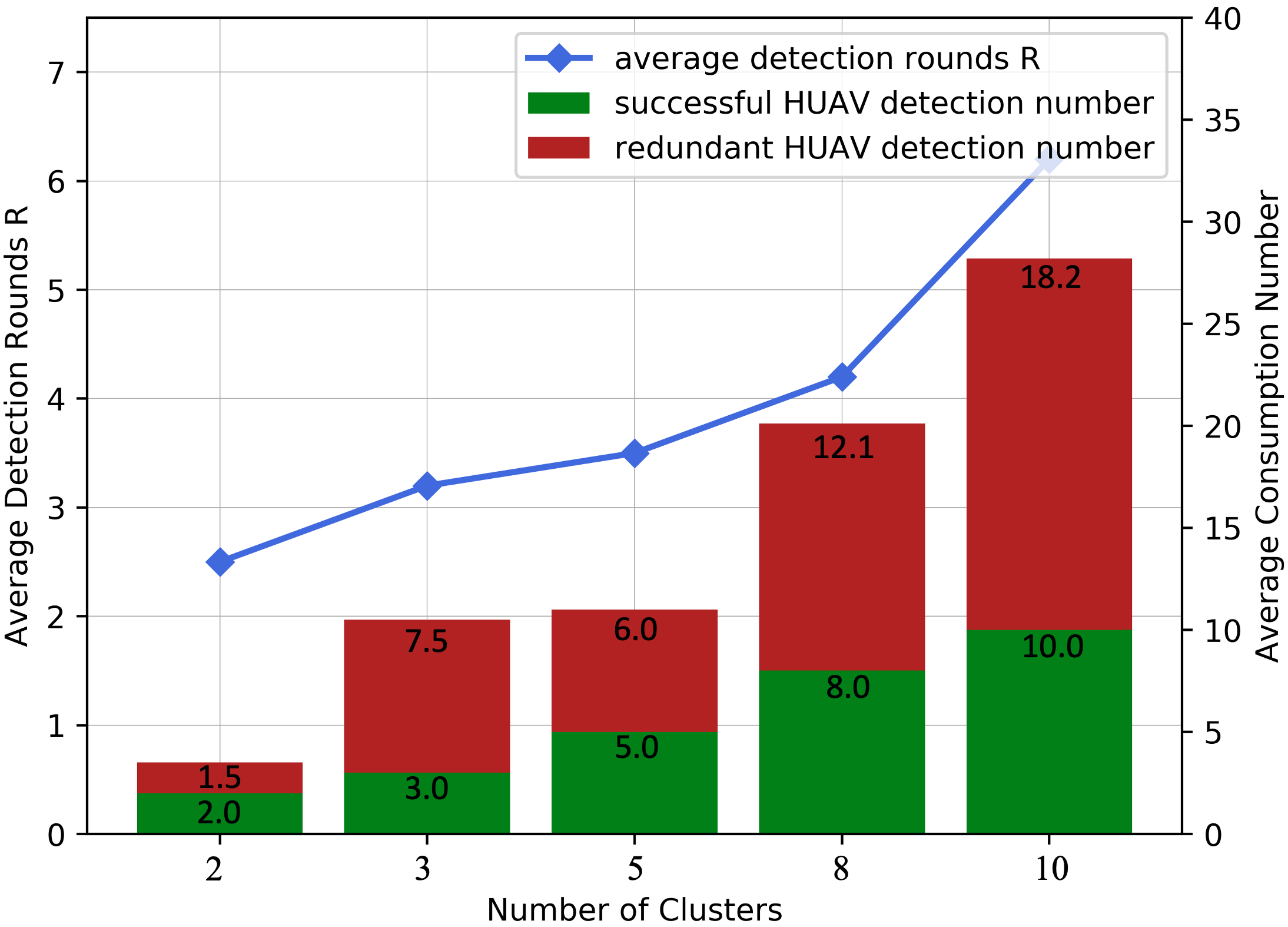}}
	
	\subfigure[The USNET obeys the IFS $f_3$ and has $M=5$ clusters.]{
		\label{fig:3} 
		\includegraphics[width=85mm]{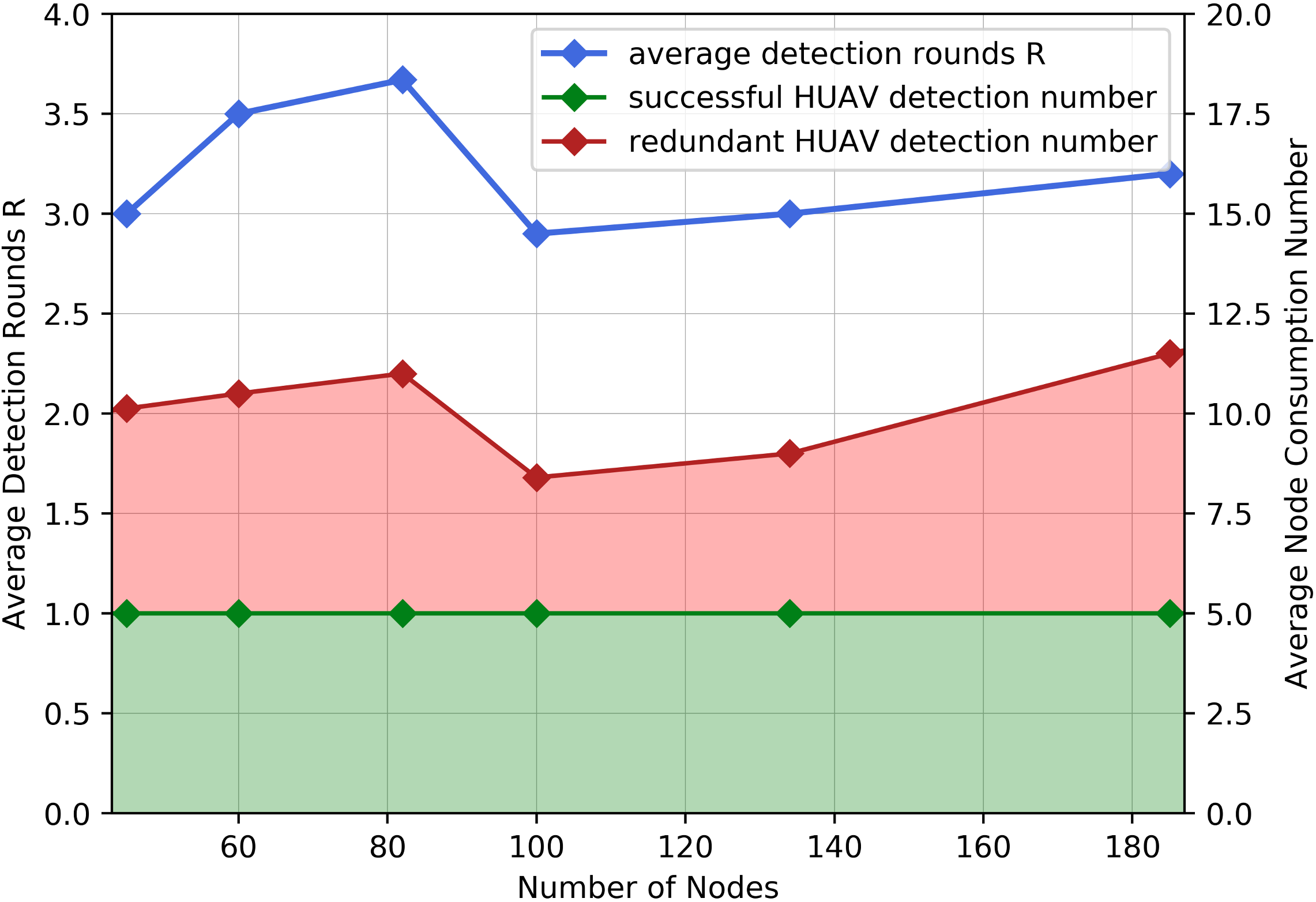}}
	\caption{The average detection results of $(\mathbf{P2})$ using MC-GASSL.}
	\label{fig:multi-cluster-results}
\end{figure*}

\subsection{Simulation Results of $(\mathbf{P2})$ with MC-GASSL}
\subsubsection{Off-line Meta Learning of the GRU Network}
\reffig{fig:meta-loss} shows the loss function during the off-line meta learning process. We can see that the loss function decreases with the learning episodes and converge to the value about 0.125. This indicates that the meta learning makes the GRU network cluster the USNET with various type of IFS within small metric learning loss. \reffig{fig:meta-compare-loss} shows the loss function of the on-line metric learning process using the GRU network initialized with meta parameters, pre-trained parameters \cite{pre-trained} and random parameters, respectively. We can see that the starting point of the loss function with meta parameters is much smaller than that with pre-trained parameters and random parameters. The loss functions of all three kinds of  parameters drops as the metric learning processes. Nonetheless, the loss function of meta parameters is always smaller than the other two loss functions and converges to a lower value. Hence, the off-line meta learning can help the GRU network find better initialized parameters and improve the performance of on-line metric learning.

\subsubsection{On-line Metric Learning of the GRU Network}


We construct USNETs obeying IFSs $f_1$ to $f_7$, respectively, and compare the clustering performance between the metric learning\footnote{The GRU network is initialized with meta parameters.} and other traditional clustering algorithms, including K-Means, symmetric non-negative matrix factorization (SNMF) \cite{nmf} and spectral clustering \cite{spectral_clustering}. 
\reffig{fig:purity-compare} shows the clustering purities of these four algorithms. We can see that the average clustering purities of metric learning  exceed those of other algorithms in all seven cases, which indicates the effectiveness of the metric learning in clustering the USNET.

\reffig{fig:feature-space} shows an example of the UAVs' feature changing during the metric learning process, where the USNET has $M=5$ clusters with a total of $N=237$ UAVs, obeying the IFS $f=f_1$. As shown in \reffig{fig:feature-space:1}, the initial features of UAVs in different clusters are mixed together and can hardly be clustered with traditional cluster methods, such as K-Means. Nonetheless, as the metric learning processes, features of UAVs in distinct clusters move away from each other, while the features of UAVs in the same clusters gradually converge together, as displayed in \reffig{fig:feature-space:2}. Hence, the metric learning can map the history positions and speeds of UAVs to a proper feature space, where traditional K-Means method can cluster the USNET easily.

\subsubsection{Cluster Head Detection $(\mathbf{P2})$ with MC-GASSL}
We construct a USNET consisting of $M=5$ UAV clusters with a total of $N=82$ UAVs.
We use the MC-GASSL to separately detect the cluster heads when the constructed USNET obeys IFS $f_1$ to $f_7$, and the detection results are shown in \reffig{fig:1}. We can see that the MC-GASSL can detect all $5$ HUAVs with an average of $R=3$ to $R=4$ detection rounds regardless of the types of the obeyed IFS. The number of successful detected HUAVs is $5$ in all seven cases, where the detection redundancy is about $\frac{5}{82}=6.1\%\ll 1$.
Hence, the MC-GASSL is effective in detecting cluster heads when the USNET obeys different IFSs. 

We construct USNETs with $M=2,3,5,8,10$ clusters that obey the same IFS $f=f_1$ and has a total of $N=82$ UAVs. We detect the cluster heads of the USNET with MC-GASSL, and the detection results are shown in \reffig{fig:2}. We can see that the number of average detection rounds as well as the detection redundancy increase with the number of UAV clusters. Nonetheless, the MC-GASSL can detect all the HUAVs successfully within an average of $R=7$ detection rounds, and the average detection redundancy is smaller than $\frac{18.2}{82}=22.2\%$. Hence, the MC-GASSL is effective in detecting cluster heads of USNETs with various UAV clusters.

We construct USNETs with identical number of clusters $M=5$, obeying the same type of IFS $f=f_3$, but with different number of UAVs. We utilize the MC-GASSL to detect the cluster heads of USNETs, and the detection results are shown in \reffig{fig:3}. We can see that the MC-GASSL can detect all $5$ HUAVs after an average of $3$ to $4$ detection rounds regardless of the total number of UAVs. The number of successful detected HUAV is $5$, while the number of redundant HUAV is about $5$ in all cases. Hence, the MC-GASSL is effective in detecting cluster heads of USNETs with various number of total UAVs.

\begin{figure*}[t]
	\centering
	\subfigure[Ground truth of the initial USNET.]{
		\label{fig:mul} 
		\includegraphics[width=85mm]{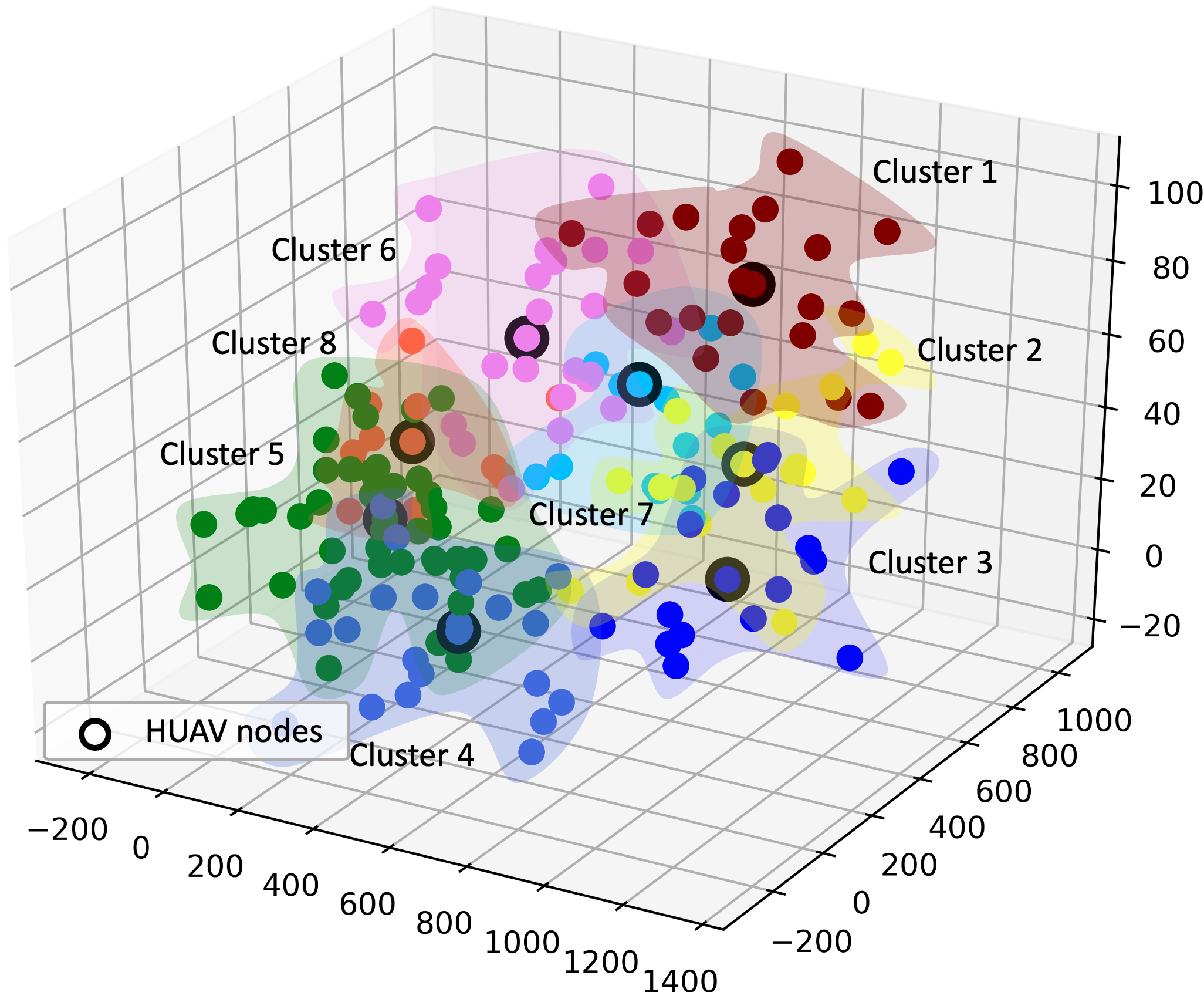}}
	\subfigure[Detection round $r=1$, successfully find $6$ HUAVs with $2$ redundant UAVs.]{
		\label{fig:mul-1} 
		\includegraphics[width=85mm]{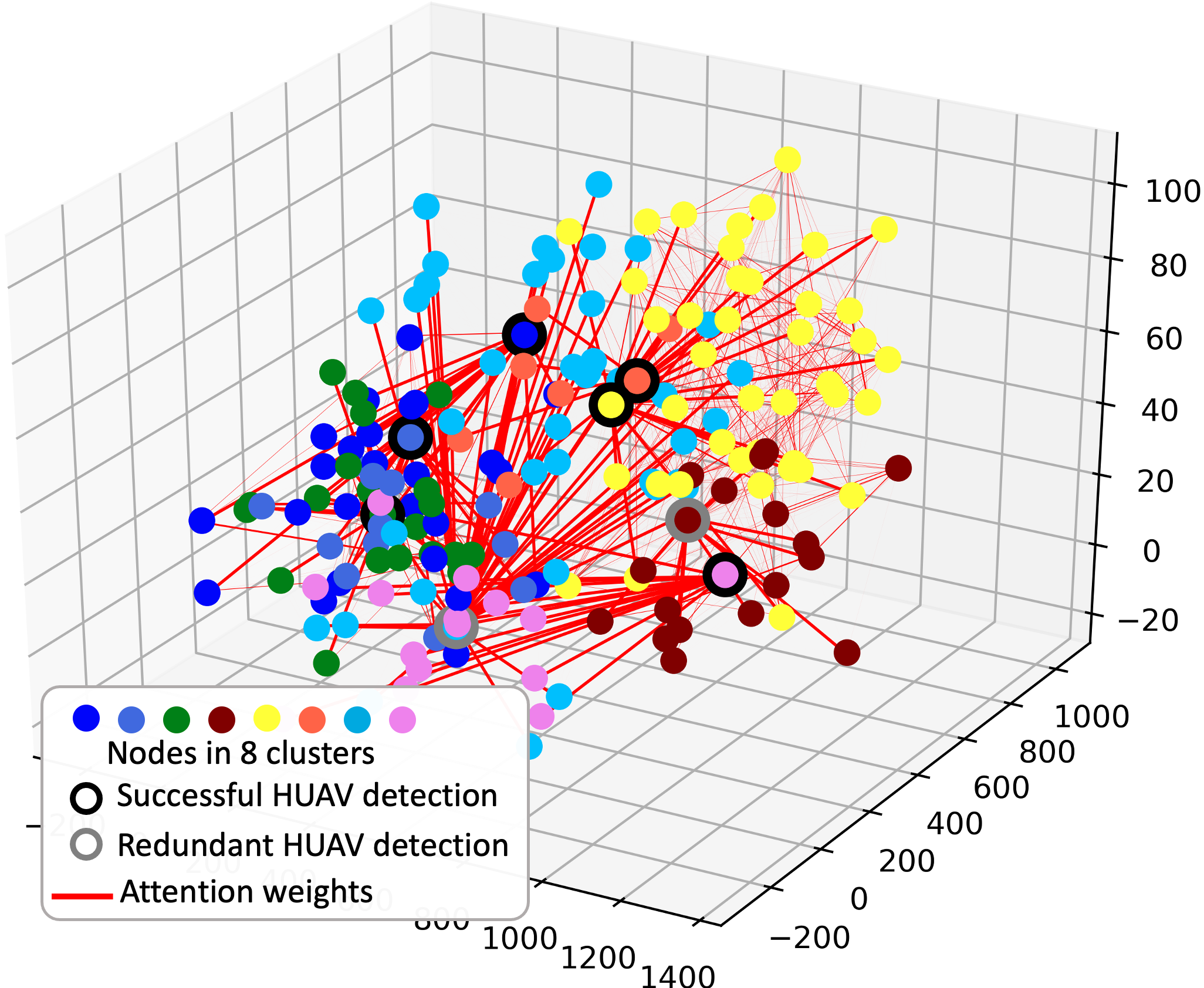}}
	
	\subfigure[Detection round $r=2$, successfully find the remaining two HUAVs.]{
		\label{fig:mul-2} 
		\includegraphics[width=85mm]{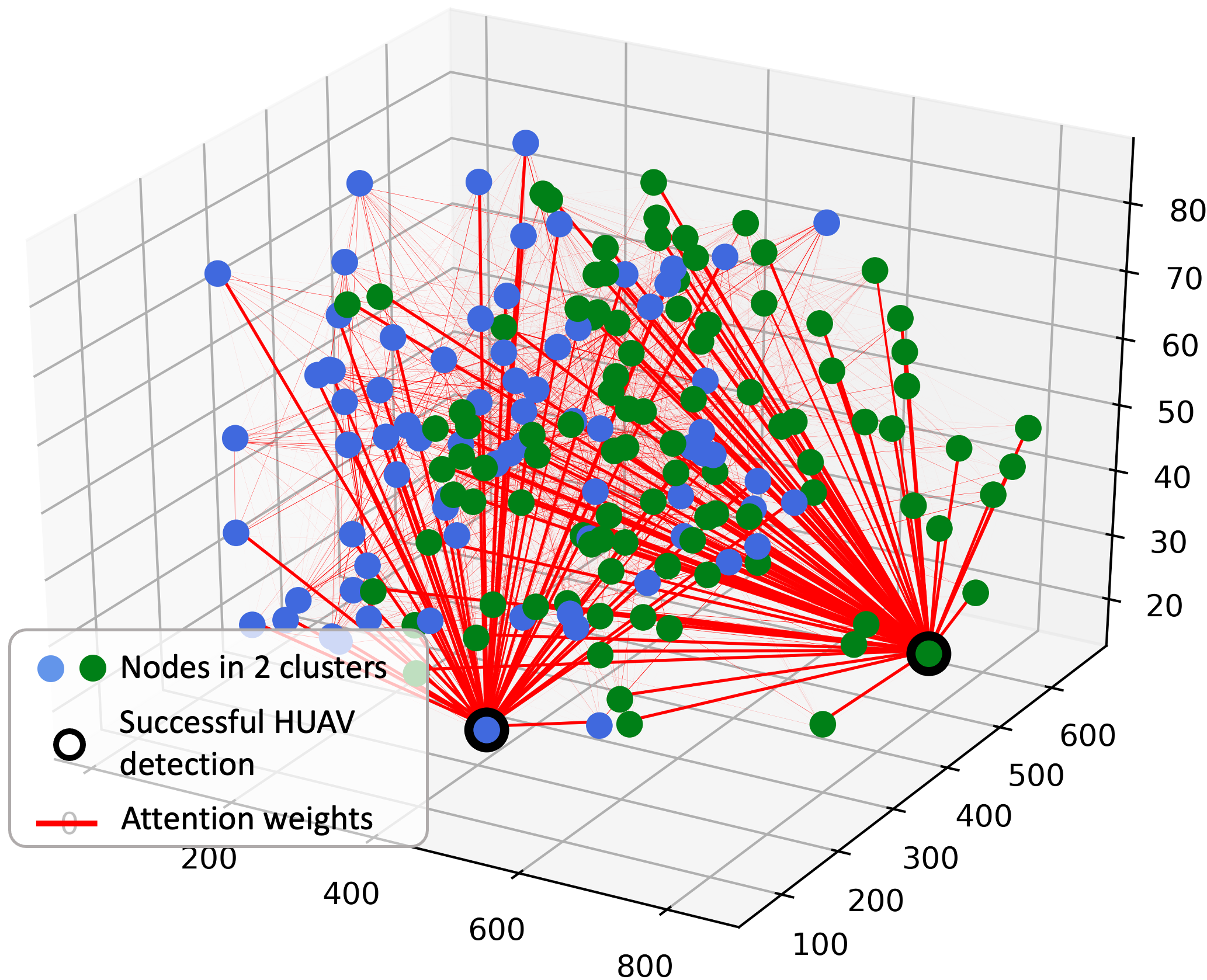}}
	\subfigure[HUAV detection rates during the detection rounds, where the total detection rate reaches 100\% after $R=2$ rounds.]{
		\label{fig:mul-3} 
		\includegraphics[width=85mm]{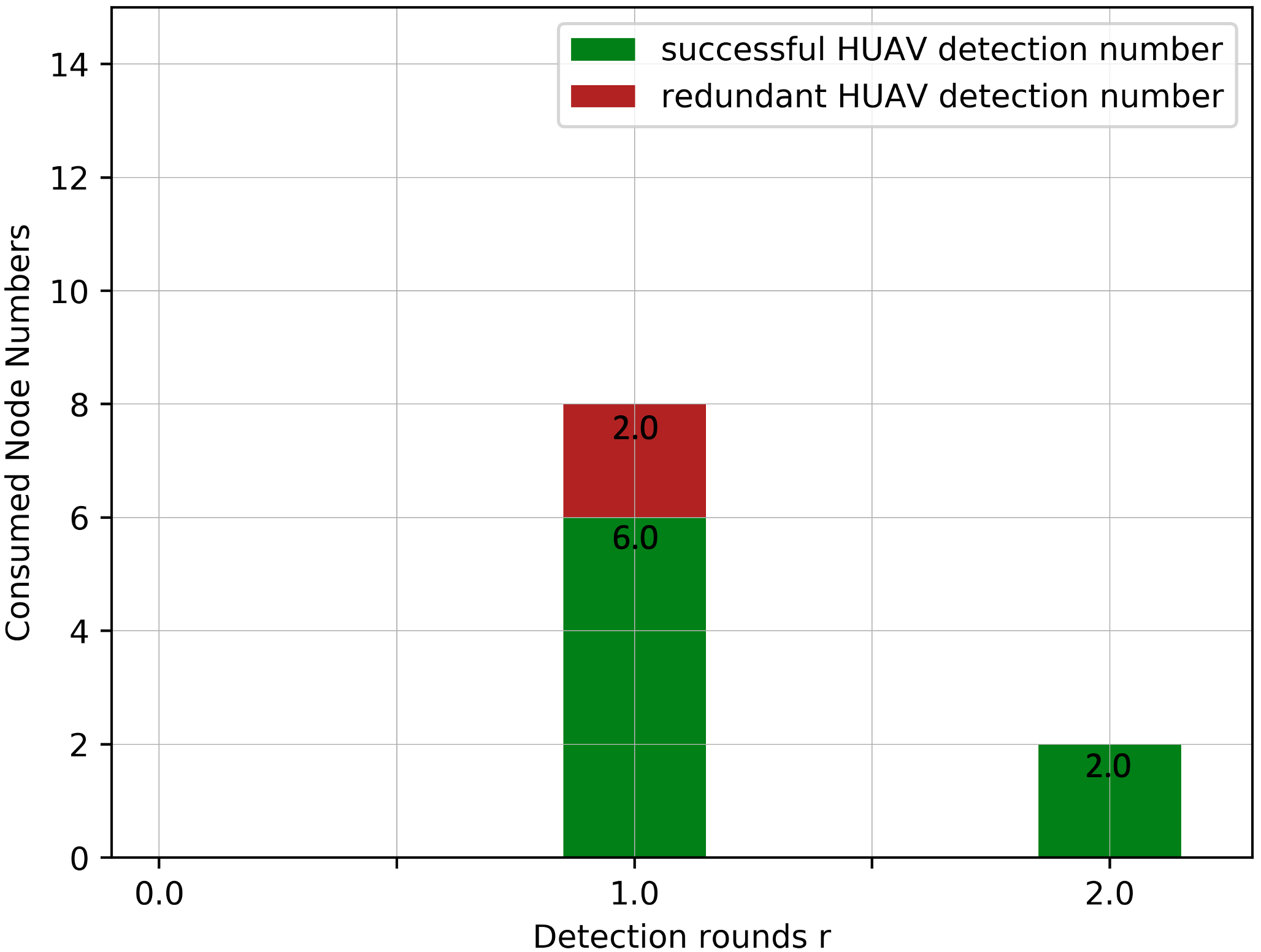}}
	
	\caption{The detection process of the USNET with MC-GASSL.}
	\label{fig:mul-cluster}
\end{figure*}

Take the case where $f=f_1$ and $M=8$ as an example. The total number of UAVs in the USNET is $185$. \reffig{fig:mul-cluster} shows the cluster head detection process with MC-GASSL, where \reffig{fig:mul} represents the ground truth of the USNET, \reffig{fig:mul-1} and \reffig{fig:mul-2} shows the attention weights and detected HUAVs during the first and second detection round, respectively, and \reffig{fig:mul-3} shows the detection rates during the whole detection process. We can see that the MC-GASSL in the first round detects $8$ UAVs, of which $6$ UAVs are correct HUAVs and $2$ UAVs are false detections. The MC-GASSL finds all the remaining HUAVs in the second round. The total consumption of UAVs is about $5.40\%$, where the successful detection rate is $4.32\%$ and redundant detection rate is $1.08\%$.

\section{Conclusions}
\label{section:conclusions}
In this paper, we study the cluster head detection problem of a two-level USNET with multiple clusters, where the IFS is unknown. We propose a GASSL to detect the hierarchical structures of the USNET composed of a single cluster. Specifically, the GASSL can find the HUAVs through calculating the attention values of each UAVs received from other UAVs and can approximately fit the IFS at the same time.
Then, we extend the GASSL to MC-GASSL to detect the hierarchical structure of the USNET with multiple number of UAV clusters.
The numerical results show that the GASSL can find the HUAVs under various kinds of IFSs with over 98\% average accuracy. The simulation results also show that the clustering purity of the USNET with MC-GASSL exceeds that with traditional clustering algorithms, and the MC-GASSL can find all the HUAVs efficiently with low detection consumptions.


%

\appendices
\section{Further Illustrations on $f(\cdot)$ and $T_0$}
\label{appendix:A}
\subsection{Function Form of $f(\cdot)$}
\label{appendix:A-1}
The form of $f(\mathbf{P}_{i,t+1},\mathbf{V}_{i_L,t}, \mathbf{p}_{i_L,t-T_0+1})$ can represent a large range of functions with independent variables $\mathbf{P}_{i,t+1}$, $\mathbf{V}_{i_L,t}$ and $\mathbf{p}_{i_L,t-T_0+1}$. Note that the reason of using $\mathbf{V}_{i_L,t}$ instead of $\mathbf{V}_{i_L,t+1}$ is that the behaviors of FUAVs are usually delayed relative to the behaviors of HUAVs. Nonetheless, one can define an IFS $f(\cdot)$ with input $\mathbf{V}_{i_L,t+1}$ if the delay can be ignored, and the proposed algorithms will still work.  Moreover, the reason of using $\mathbf{p}_{i_L,t-T_0+1}$ instead of $\mathbf{P}_{i_L,t+1}$ is that the position of HUAV $i'$ from time step $t-T_0+2$ to $t+1$ can be derived with $\mathbf{V}_{i_L,t}$ and $\mathbf{p}_{i_L,t-T_0+1}$. Specifically, the position of HUAV $i'$ at time step $\tau\in\{t-T_0+2,...,t,t+1\}$ can be calculated as $\mathbf{p}_{i_L,\tau}=\mathbf{p}_{i_L,t-T_0+1}+\sum_{\tau'=t-T_0+1}^{\tau-1}\mathbf{v}_{i_L,\tau'}$.
\subsection{Value of $T_0$ and Redundant Variables of $f(\cdot)$ }
\label{appendix:A-2}
In practice, we may not know the exact value of $T_0$. To make sure that the form of $f(\cdot)$ in \eqref{equ:vf} can represent the truth inherent follow strategy of the USNET, we can endow $T_0$ with a large value based on prior experience. Therefore, there may be superfluous variable in the inputs of $f(\cdot)$. For example, we endow $T_0$ with $2$, while FUAV $i$ only follows the speed of HUAV $i_L$ in the previous time step, i.e.,
\begin{align}
	\mathbf{v}_{i,t+1}&=f(\mathbf{P}_{i,t+1},\mathbf{V}_{i_L,t}, \mathbf{p}_{i_L,t-T_0+1})\notag\\
	&=f(\mathbf{p}_{i,t+1},\mathbf{p}_{i,t}, \mathbf{v}_{i_L,t},\mathbf{v}_{i_L,t-1},\mathbf{p}_{i_L,t-1})\notag\\
	&=\mathbf{v}_{i_L,t}.
\end{align} 
Under such circumstances,  $\mathbf{p}_{i,t+1},\mathbf{p}_{i,t}, \mathbf{v}_{i_L,t-1}$ and $\mathbf{p}_{i_L,t-1}$ are all redundant variables for $f(\cdot)$. The attention heads dealing with $\mathbf{v}_{i_L,t}$, $\mathbf{v}_{i_L,t-1}$ and $\mathbf{p}_{i_L,t-1}$ produces the attention values to HUAV $i'$ as $\alpha_{i,i_L,1}$, $\alpha_{i,i_L,2}$ and $\alpha_{i,i_L,3}$, respectively.
When $\widehat{f}(\cdot)$ approximates $f(\cdot)$, we can reduce the loss function $\mathcal{L}(\mathbf{\Theta})$ towards zero even with random $\alpha_{i,i_L,2}$ and $\alpha_{i,i_L,3}$.

\begin{figure*}
	\setlength{\abovecaptionskip}{0.5cm}
	\centering
	\subfigure[$m_1=2,N=3$.]{
		\label{fig:simulations:3-1} 
		\includegraphics[width=80mm]{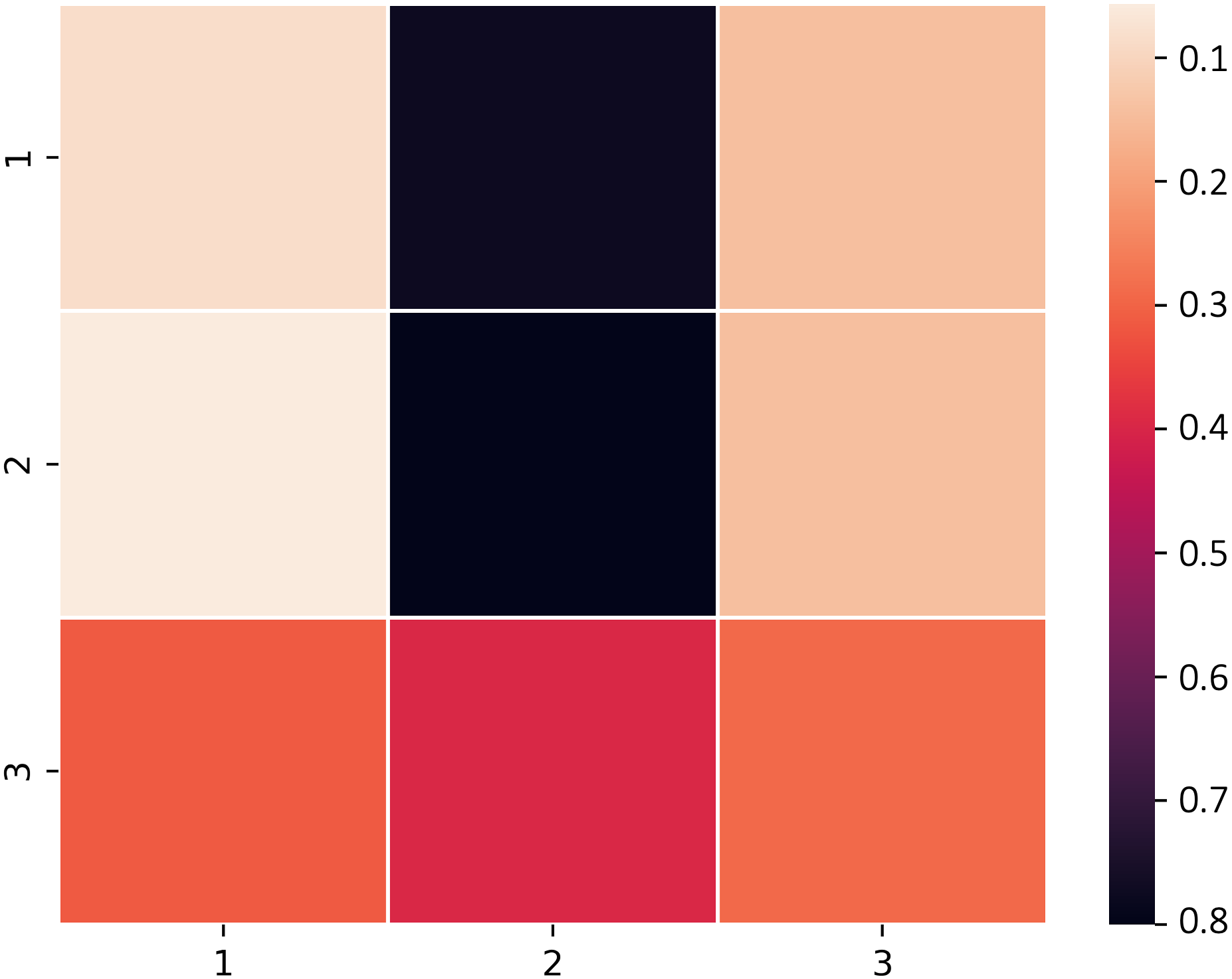}}
	\subfigure[$m_1=5,N=6$.]{
		\label{fig:simulations:3-2} 
		\includegraphics[width=80mm]{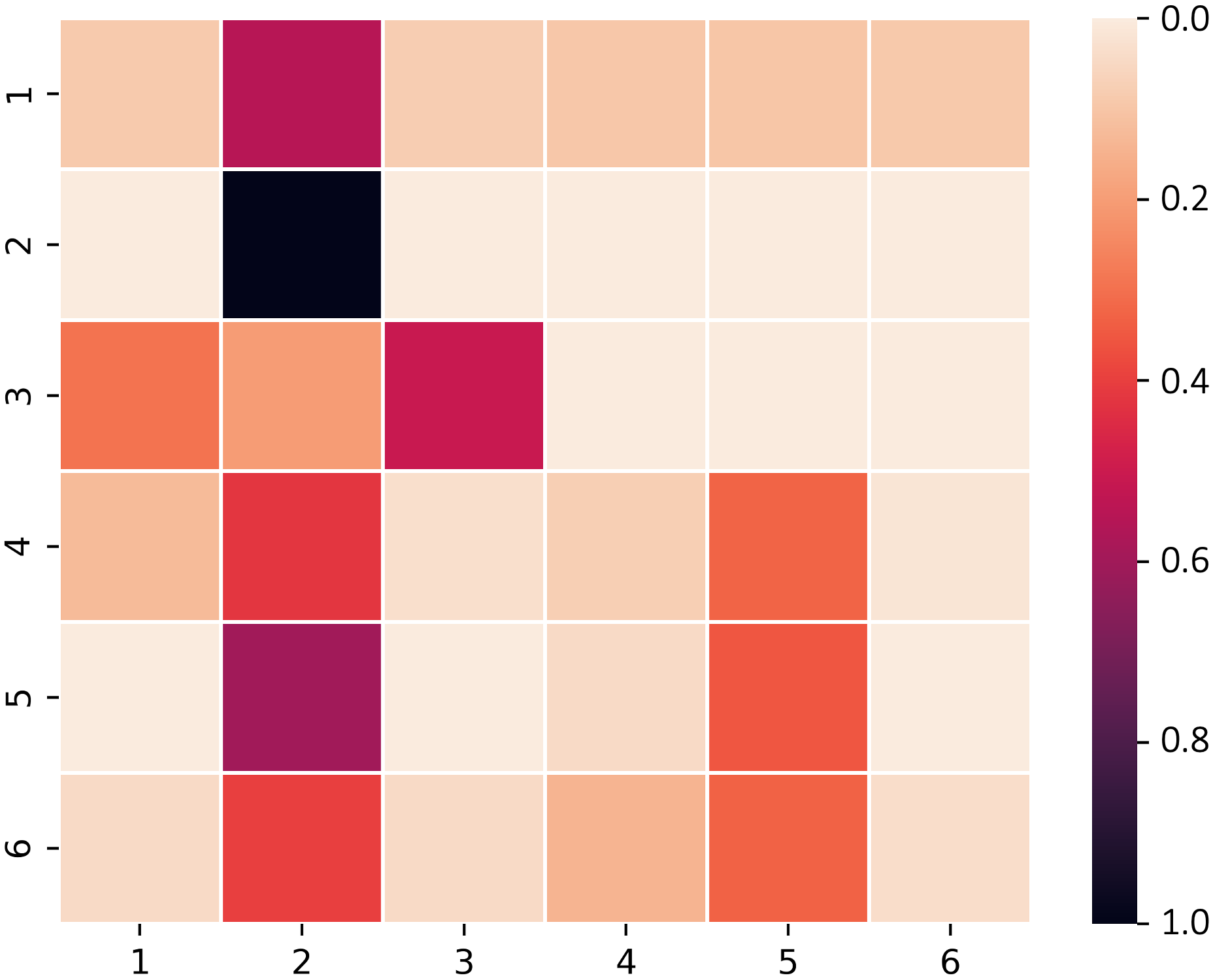}}
	
	\subfigure[$m_1=10,N=11$.]{
		\label{fig:simulations:3-3} 
		\includegraphics[width=80mm]{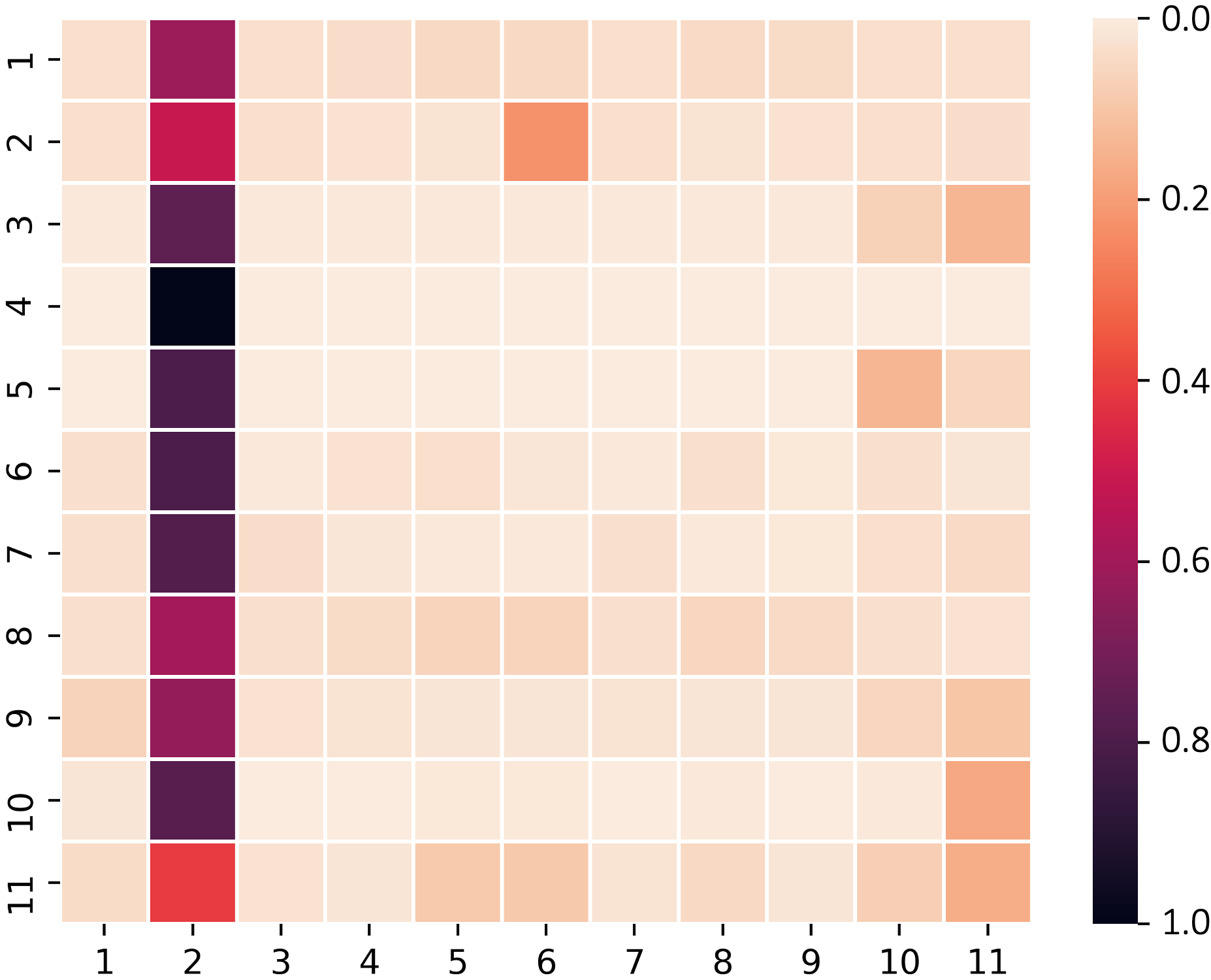}}
	\subfigure[$m_1=25,N=26$.]{
		\label{fig:simulations:3-4} 
		\includegraphics[width=80mm]{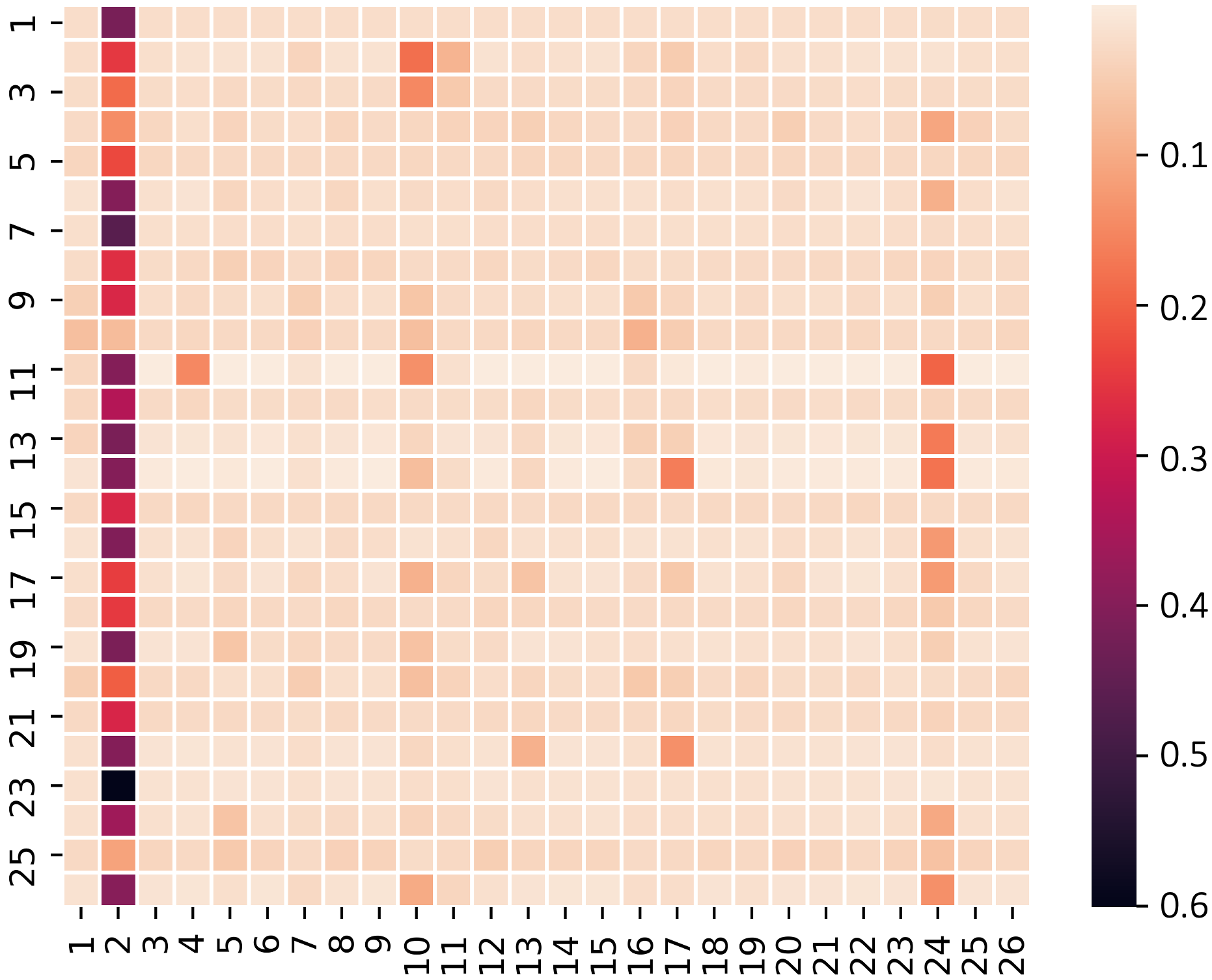}}
	
	\subfigure[$m_1=40,N=41$.]{
		\label{fig:simulations:3-5} 
		\includegraphics[width=80mm]{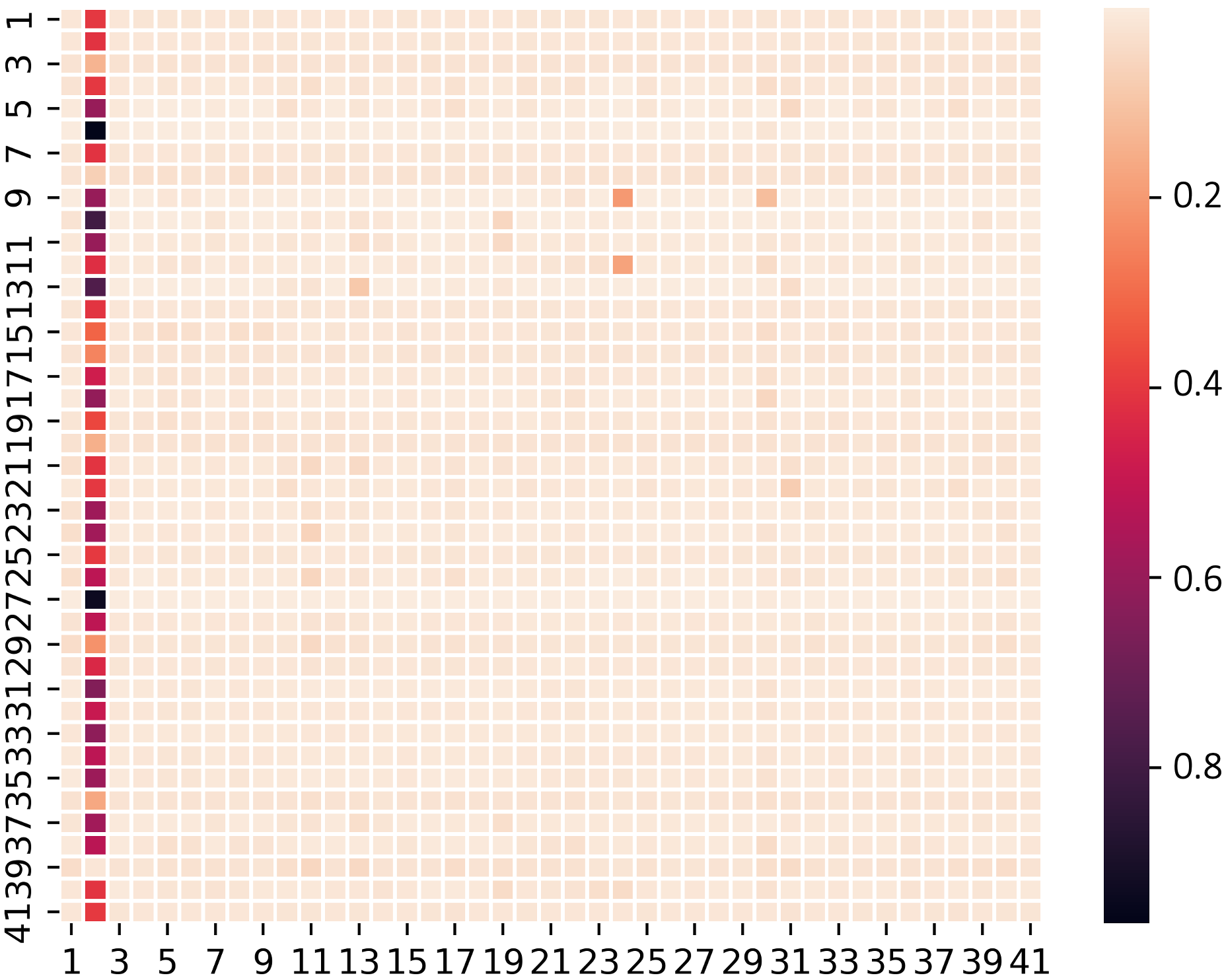}}
	\subfigure[$m_1=50,N=51$.]{
		\label{fig:simulations:3-6} 
		\includegraphics[width=80mm]{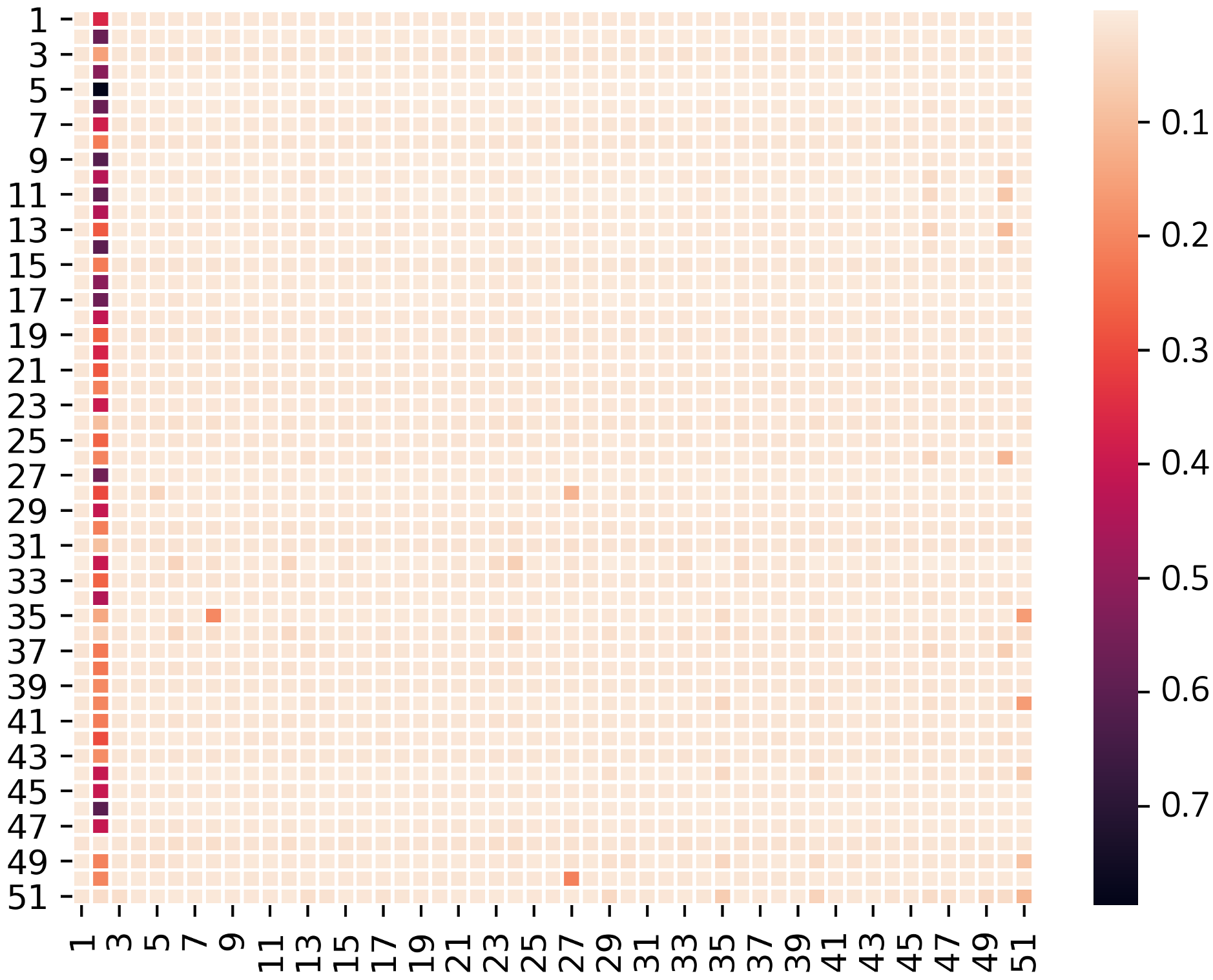}}
	
	\caption{Heat map of attentions weights, where $f=f_7$ and $i_L=2$.}
	\label{fig:simulation:attention-2}
\end{figure*}

\section{Design of $\mathbf{L}^s(\mathbf{W})$}
\label{appendix:loss_design}
The design process of $\mathbf{L}^s(\mathbf{W})$ is inspired from the design of the triplet loss function \cite{triplet-loss}. Specifically, the distance between the extracted features of anchor UAV and positive UAV should be smaller than the distance between the extracted features of anchor UAV and negative UAV, i.e.,
\begin{align}
	\label{equ:small}
	\left\|\mathbf{e}_{i_a,t}-\mathbf{e}_{i_p,t}\right\|_2\le \left\|\mathbf{e}_{i_a,t}-\mathbf{e}_{i_n,t}\right\|_2.
\end{align}
To avoid the trivial solution $\mathbf{e}_{i_a,t}=\mathbf{e}_{i_p,t}=\mathbf{e}_{i_n,t}$, we augment \eqref{equ:small} by
\begin{align}
	\label{equ:small_2}
	\left\|\mathbf{e}_{i_a,t}-\mathbf{e}_{i_p,t}\right\|_2- \left\|\mathbf{e}_{i_a,t}-\mathbf{e}_{i_n,t}\right\|_2\le -\gamma,
\end{align}
where $\gamma>0$. Note that we only require the extracted features to satisfy \eqref{equ:small_2}, which means $\mathcal{L}^s(\mathbf{W})$ is $0$ when $\left\|\mathbf{e}_{i_a,t}-\mathbf{e}_{i_p,t}\right\|_2- \left\|\mathbf{e}_{i_a,t}-\mathbf{e}_{i_n,t}\right\|_2+\gamma\le 0$. Hence, the loss function $\mathcal{L}^s(\mathbf{W})$ is designed as:
\begin{align}
	\label{equ:loss}
	\mathcal{L}^s(\mathbf{W})=\bigg[\left\|\mathbf{e}_{i_p,t}-\mathbf{e}_{i_a,t}\right\|_2-\left\|\mathbf{e}_{i_n,t}-\mathbf{e}_{i_a,t}\right\|_2+\gamma\bigg]_+.
\end{align}

\section{Examples of the Heat Map of the  GASSL}
\label{appendix:heat-map}

\reffig{fig:simulation:attention-2} shows the heat map of the attention weight in different number of FUAVs $m_1$, where the IFS $f(\cdot)$ is an MLP and the index of HUAV $i_L$ is $2$, i.e.,  $f=f_7$, and $i_L=2$. We can see that the GASSL is able to find the HUAV $i_L$ in all cases.



\ifCLASSOPTIONcaptionsoff
  \newpage
\fi


\begin{thebibliography}{1}

\bibitem{uav_swarm}
H. Wang, H. Zhao, J. Zhang, D. Ma, J. Li and J. Wei, ``Survey on unmanned aerial vehicle networks: a cyber physical system perspective," \emph{IEEE Commun. Surv.  Tutor.,} vol. 22, no. 2, pp. 1027--1070, Dec. 2019.
\bibitem{data_collection_1}
Y. Zhang, Z. Mou, F. Gao, L. Xing, J. Jiang and Z. Han, ``Hierarchical deep reinforcement learning for backscattering data collection with multiple UAVs," \emph{IEEE Internet of Things J.,} vol. 8, no. 5, pp. 3786--3800, Mar. 2021.
\bibitem{data_collection_2}
X. Xu, H. Zhao, H. Yao and S. Wang, "A blockchain-enabled energy-efficient data collection system for UAV-assisted IoT," \emph{IEEE Internet Things J.,} vol. 8, no. 4, pp. 2431--2443, Feb. 2021.
\bibitem{area_coverage}
Z. Mou, Y. Zhang, F. Gao, H. Wang, T. Zhang and Z. Han, ``Deep reinforcement learning based three-dimensional area coverage with UAV swarm," \emph{IEEE J. Sel. Areas Commun.,} vol. 39, no. 10, pp. 3160--3176, Oct. 2021.
\bibitem{security}
H. Wang, H. Fang and X. Wang, 		``Safeguarding cluster heads in UAV swarm using edge intelligence: linear discriminant analysis-based cross-layer authentication," \emph{IEEE Open J. Commun. Soc.,} vol. 2, pp. 1298--1309, May. 2021.
\bibitem{tra_1}
Á. Madridano, A. Al-Kaff, and D.  Martín, ``3d trajectory planning method for uavs swarm in building emergencies,"  \emph{Sens.,} vol. 20, no. 3: 642, pp. 1--20, Jan. 2020.
\bibitem{tra_2}
H. Teng, I. Ahmad, A. Msm, and K. Chang, ``3D optimal surveillance trajectory planning for multiple UAVs by using particle swarm optimization with surveillance area priority," \emph{IEEE Access,}  vol. 8, pp. 86316--86327. May. 2020.
\bibitem{tra_3}
Z. Mou, F. Gao, J. Liu and Q. Wu, ``Resilient UAV swarm communications with graph convolutional neural network," \emph{IEEE J. Sel. Areas  Commun.,} vol. 40, no. 1, pp. 393--411, Jan. 2022.
\bibitem{data_1}
F. Xiong, A.  Li, H. Wang, and L. Tang, ``An SDN-MQTT based communication system for battlefield UAV swarms," \emph{IEEE Commun. Mag.,} vol. 57, no. 8, pp. 41--47, Aug. 2019.
\bibitem{data_2}
Y. Zhang, Z. Mou, F. Gao, J. Jiang, R. Ding, and Z. Han, ``UAV-enabled secure communications by multi-agent deep reinforcement learning," \emph{IEEE Trans. Veh. Tech.,} vol. 69, no. 10, pp. 11599--11611. Oct. 2020.
\bibitem{data_3}
C. Zhan, Y. Zeng and R. Zhang, ``Energy-efficient data collection in UAV enabled wireless sensor network," \emph{IEEE Wirel. Commun. Lett.,} vol. 7, no. 3, pp. 328--331, Jun. 2018.
\bibitem{charging_1}
V. Hassija, V.  Chamola, D. N. G. Krishna, and M. Guizani, ``A distributed framework for energy trading between UAVs and charging stations for critical applications," \emph{IEEE Trans. Veh. Tech.,} vol. 69, no. 5, pp. 5391--5402, May. 2020.
\bibitem{charging_2}
K. Wang, X. Zhang, L. Duan and J. Tie, ``Multi-UAV cooperative trajectory for servicing dynamic demands and charging battery," \emph{IEEE Trans. Mob. Comput.,} Early Access, Sep. 2021.
\bibitem{defenders}
B. Liu and J. Sun, ``Stackelberg game under asymmetric information in unmanned aerial vehicle swarm active deception defense: from a multi-layer network perspective," in \emph{Int. Conf. Big Data Intell. Decis. Mak.,} Guilin, China, Jul. 2021. pp. 75--79.
\bibitem{hie}
C. Xu, K. Zhang, Y. Jiang, S.  Niu, T. Yang, and H.  Song, ``Communication aware UAV swarm surveillance based on hierarchical architecture," \emph{Drones,} vol. 5, no. 2, pp. 1--26, Apr. 2021.
\bibitem{confrontation}
D. Xing, Z. Zhen, H. Gong, ``Offense–defense confrontation decision making for dynamic UAV swarm versus UAV swarm," \emph{Proc. Inst.  Mech. Eng., Part G: J. Aerosp. Eng.,} vol. 233, no. 15, pp. 5689--5702. Jun. 2019.
\bibitem{survey_cd}
F. Liu, S. Xue, J. Wu, C. Zhou, W. Hu, C. Paris,  S. Nepal, J. Yang, and P. S. Yu, ``Deep learning for community detection: progress, challenges and opportunities," in \emph{Proc. 29th Int. Joint Conf. Artif. Intell.,(IJCAI),} Yokohama, Japan, Jan. 2021, pp. 4981--4987.
\bibitem{appearance}
A. Tahir, J.  Böling, M. H.  Haghbayan, H. T. Toivonen, and J.  Plosila,  ``Swarms of unmanned aerial vehicles—a survey," \emph{J. Ind. Inf. Integr.,} vol. 16, no. 100106. pp. 1--7, Dec. 2019.
\bibitem{hetero}
J. Chen, Q. Wu, Y. Xu, N. Qi, X. Guan, Y. Zhang, and Z.  Xue, ``Joint task assignment and spectrum allocation in heterogeneous UAV communication networks: A coalition formation game-theoretic approach,"  \emph{IEEE Trans. Wirel. Comm.,} vol. 20, no.1, pp. 440--452, Jan. 2021.
\bibitem{hete}
C. Zhang, D. Song, C. Huang, A. Swami, and N. V. Chawla, ``Heterogeneous graph neural network," in \emph{Proc. 25th ACM SIGKDD Int. Conf. Knowl. Discov. Data Min.} Anchorage, AK, USA, Jul. 2019, pp. 793--803.
\bibitem{gat}
P. Veličković, G. Cucurull, A. Casanova, A.  Romero, P. Liò, and Y. Bengio, ``Graph attention networks," in \emph{Proc. 6th Int. Conf. Learn. Represent. (ICLR),} Vancouver Convention Center, Vancouver Canada, Feb. 2018, pp. 1--12.
\bibitem{gru}
Cho, K., van Merriënboer, B., Gulcehre, C., Bahdanau, D., Bougares, F., Schwenk, H., and Bengio, Y. ``Learning phrase representations using RNN encoder–decoder for statistical machine translation, in \emph{Proc. Conf. Empirical Methods Nat. Lang. Process. (EMNLP),}  Doha, Qatar, Oct. 2014, pp. 1724--1734.
\bibitem{k-means}
K. Krishna and M. Narasimha Murty, ``Genetic K-means algorithm," \emph{ IEEE Trans. Syst., Man, Cybern. B, Cybern.,} vol. 29, no. 3, pp. 433--439, Jun. 1999.
\bibitem{metric_learning}
J. Wang, F.  Zhou, S. Wen, X. Liu, and Y. Lin, ``Deep metric learning with angular loss," in \emph{Proc. IEEE Int. Conf. Comput. Vis.  Pattern Recognit. (CVPR),} Honolulu, Hawaii, USA, Jul. 2017, pp. 2593--2601.
\bibitem{hard-sample}
F. Schroff, D. Kalenichenko, and J. Philbin, ``Facenet: A unified embedding for face recognition and clustering," in \emph{Proc. IEEE Conf. Comput. Vis. Pattern Recognit. (CVPR),} Boston, Massachusetts, USA, Jun. 2015, pp. 815--823.
\bibitem{triplet-loss}
W. Ge, ``Deep metric learning with hierarchical triplet loss," in \emph{Proc. Eur. Conf. Comput. Vis. (ECCV),} Munich, Germany, Sep. 2018, pp. 269--285.
\bibitem{maml}
C. Finn, R.  Abbeel, and S. Levine, ``Model-agnostic meta-learning for fast adaptation of deep networks," in \emph{Proc. 34th Int. Conf. Mach. Learn. (ICML),} International Convention Centre, Sydney, Australia, Jul. 2017, pp. 1126--1135.
\bibitem{gap-statistic}
R. Tibshirani, G. Walther, and T. Hastie, ``Estimating the number of clusters in a data set via the gap statistic,” \emph{J. R. Stat. Soc.,} vol. 63, no. 2, pp. 411--423, Jan. 2002.
\bibitem{pre-trained}
A. Raghu, J. Lorraine, S.  Kornblith, M.  McDermott, and D. K. Duvenaud, ``Meta-learning to improve pre-training," in \emph{Proc. 35th Conf. Neural Inf. Process. Sys.,} virtual only, May. 2021, pp. 1--31.
\bibitem{nmf}
Y. Pei, N. Chakrabort, and K. Sycara, ``Nonnegative matrix tri-factorization with graph regularization for community detection in social networks," in \emph{Proc. 24th int. jt. conf. artif. intell.,} Buenos Aires, Argentina, Jun. 2015. pp. 2083--2089.
\bibitem{spectral_clustering}
U. Von Luxburg, ``A tutorial on spectral clustering," \emph{Stat. Comput.,} vol. 17, no. 4, pp. 395--416. Aug. 2007.

\end{thebibliography}
\end{document}